\documentclass{article}

\usepackage[preprint]{neurips_2025}

\usepackage{amsmath,amsfonts,bm, amssymb, amsthm}

\usepackage[utf8]{inputenc}
\usepackage[T1]{fontenc}
\usepackage{graphicx}
\usepackage{booktabs}
\usepackage{enumitem}
\usepackage{float}
\usepackage{subcaption}
\usepackage{tikz}
\usepackage{pdflscape}
\usepackage{bm, bbm}
\usepackage{comment}
\usepackage{natbib}
\usepackage{placeins}
\usepackage{setspace}
\usepackage{parskip}
\usepackage{multirow}
\usepackage{colortbl}
\usepackage{array, makecell}
\usepackage{rotating}

\newcommand{\best}[1]{\textbf{\cellcolor{green!25}#1}}
\newcommand{\greencell}[1]{\cellcolor{green!25}#1}
\newcolumntype{L}{>{\raggedright\arraybackslash}p{3.8cm}}
\newcolumntype{C}{>{\centering\arraybackslash}p{2.2cm}}
\newcommand{\ms}[2]{#1\,{$\pm$\,#2}}

\usepackage{algorithm2e}
\RestyleAlgo{ruled}

\usepackage{hyperref}
\hypersetup{pdfborder = {0 0 0.5 [3 3]}, colorlinks = true, linkcolor = BrightRed, citecolor = SkyBlue}
\usepackage[capitalize]{cleveref}

\newcommand{\marginalplain}{\pi}

\newcommand{\marginal}[1]{\marginalplain_{#1}}
\newcommand{\timet}{t}
\newcommand{\timeidx}{i}
\newcommand{\sampleidx}{n}
\newcommand{\totsteps}{I}
\newcommand{\timestep}[1]{\timet_{#1}}

\newcommand{\de}{\text{d}}
\newcommand{\responseplain}{X}
\newcommand{\obsplain}{Y}
\newcommand{\altrespplain}{Z}
\newcommand{\statevar}{\bm{y}}

\newcommand{\responseat}[1]{\bm{\responseplain}_{#1}}

\newcommand{\altresponseat}[1]{\bm{\altrespplain}_{#1}}

\newcommand{\obsat}[1]{\bm{\obsplain}_{#1}}

\newcommand{\obsatall}[1]{\bm{\obsplain}_{#1}^{\text{all}}}

\newcommand{\tottrajec}{N}

\newcommand{\drift}{\bm b}
\newcommand{\volatility}{\bm g}

\newcommand{\truevolatility}{\volatility_{0}}

\newcommand{\brownianm}{\bm{W}_{\timet}}
\newcommand{\truedrift}{\drift_{0}}

\newcommand{\allparam}{\bm \theta}

\newcommand{\stateconditionalplain}{f}
\newcommand{\empconditional}[2]
{\hat{\stateconditionalplain}(#1;#2)}
\newcommand{\empjoint}[2]{\hat{\stateconditionalplain}(#1,#2)}
\newcommand{\empjointnoarg}{\hat{\stateconditionalplain}}
\newcommand{\predconditional}[2]
{\stateconditionalplain_{\allparam}(#1|#2)}
\newcommand{\predjoint}[2]{\stateconditionalplain_{\allparam}(#1,#2)}
\newcommand{\predjointnoarg}
{\stateconditionalplain_{\allparam}}

\newcommand{\marginaldenplain}{f}

\newcommand{\timemeasure}{h}
\newcommand{\emptimemeasure}{\hat{h}}

\newcommand{\barycentershort}{\marginaldenplain_{\text{bary}}}

\DeclareMathOperator{\MMD}{MMD}
\DeclareMathOperator{\MMDu}{MMD^2_{U,U}}
\DeclareMathOperator{\diag}{diag}
\newcommand{\resprange}{\R^d}

\newcommand{\timerange}{\mathcal{T}}

\newcommand{\totsimummd}{M}
\newcommand{\simuidx}{m}

\def\eqref#1{equation~\ref{#1}}

\def\1{\bm{1}}

\DeclareMathAlphabet{\mathsfit}{\encodingdefault}{\sfdefault}{m}{sl}
\SetMathAlphabet{\mathsfit}{bold}{\encodingdefault}{\sfdefault}{bx}{n}

\newcommand{\E}{\mathbb{E}}

\newcommand{\R}{\mathbb{R}}

\newcommand{\KL}{D_{\mathrm{KL}}}

\DeclareMathOperator*{\argmin}{arg\,min}

\usepackage{thmtools}
\usepackage{thm-restate}

\newtheorem{proposition}{Proposition}[section]
\newtheorem{assumption}{Assumption}[section]
\newtheorem{definition}{Definition}[section]

\definecolor{SkyBlue}{RGB}{14, 118, 188}
\definecolor{BrightRed}{RGB}{223,82, 78}

\def\keywordname{{\bfseries \emph Keywords}}%
\def\keywords#1{\par\addvspace\medskipamount{\rightskip=0pt plus1cm
\def\and{\ifhmode\unskip\nobreak\fi\ $\cdot$
}\noindent\keywordname\enspace\ignorespaces#1\par}}

\def\R{\mathbb{R}}

\def\E{\mathbb{E}}

\title{Oh SnapMMD! Forecasting Stochastic Dynamics Beyond the Schrödinger Bridge’s End
}
\author{Renato Berlinghieri$^*$ \\ MIT \And Yunyi Shen$^*$ \\ MIT \And Jialong Jiang \\ The Rockefeller University \And Tamara Broderick \\ MIT
}

\begin{document}

\maketitle

\begin{abstract}
Scientists often want to make predictions beyond the observed time horizon of “snapshot” data following latent stochastic dynamics. For example, in time course single-cell mRNA profiling, scientists have access to cellular transcriptional state measurements (snapshots) from different biological replicates at different time points, but they cannot access the trajectory of any one cell because  measurement destroys the cell. Researchers want to forecast (e.g.)\ differentiation outcomes from early state measurements of stem cells. Recent Schrödinger-bridge (SB) methods are natural for interpolating between snapshots. But past SB papers have not addressed forecasting --- likely since existing methods either (1) reduce to following pre-set reference dynamics (chosen before seeing data) or (2) require the user to choose a fixed, state-independent volatility since they minimize a Kullback–Leibler divergence. Either case can lead to poor forecasting quality. In the present work, we propose a new framework, SnapMMD, that learns dynamics by directly fitting the joint distribution of both state measurements and observation time with a maximum mean discrepancy (MMD) loss. Unlike past work, our method allows us to infer unknown and state-dependent volatilities from the observed data. We show in a variety of real and synthetic experiments that our method delivers accurate forecasts. Moreover, our approach allows us to learn in the presence of incomplete state measurements and yields an $R^2$-style statistic that diagnoses fit. We also find that our method's performance at interpolation (and general velocity-field reconstruction) is at least as good as (and often better than) state-of-the-art in almost all of our experiments. 
\end{abstract}

\section{Introduction}
\label{sec:intro}
Many scientific modeling problems require forecasting stochastic dynamics from snapshot data. Here, snapshot data represents observations taken at different time points, without access to individual trajectories. And forecasting represents predicting future states beyond the observed times. For example, single-cell RNA sequencing (scRNA-seq) is widely used to study dynamic processes such as development, immune activation, and cancer progression. Each measurement in scRNA-seq destroys the cell, so scientists observe independent biological replicates at discrete times rather than a single replicate across multiple times. Despite the absence of individual cell trajectories, researchers often aim to forecast future cellular states; for instance, researchers are interested in forecasting differentiation outcomes of stem cells, immune cell activity after initial signal stimulation, or long-term cancer cell response to drugs with transcriptomic snapshots taken shortly after treatments. A common additional challenge is incomplete state measurement. For instance, although protein expression level mediates the dynamics of gene expression, protein level cannot be measured by scRNA-seq. 

Recent work has addressed the problem of interpolating between snapshots through Schrödinger bridge (SB) methods \citep{pavon2021data,de2021diffusion,Vargas2021,koshizuka2022neural,Wang2023} and their multi-marginal extensions \citep{shen2024multi,zhang2024joint,guan2024identifying,chen2024deep,Lavenant2021}. These methods reconstruct likely trajectories between snapshots at consecutive times. SB methods interpolate via regularized couplings, often chosen to be entropy-regularized. In settings where the goal is to fill in missing timepoints between observed snapshots, these techniques perform well. However, SB methods have not yet been systematically evaluated at forecasting, and we expect them to struggle for two reasons. (1) In the present work, we show that forecasting with many SB methods requires following reference dynamics defined before observing any data.
(2) All SB-based methods optimize a Kullback--Leibler divergence between data and a nominal diffusion model, and this optimization formulation requires a known, fixed volatility. In practice, volatility is often unknown or state-dependent, as in scRNA-seq, where noise arises from a mix of biological and technical sources. We see in our experiments that missing the state dependence or choosing an inappropriate (though standard) volatility value can lead to poor forecasts.
Finally, we note that SB methods often rely on iterative Sinkhorn solvers that offer limited interpretability and few tools for diagnosing model fit.

In this work, we propose a new framework, SnapMMD, that shifts the modeling focus from interpolation to accurate, interpretable forecasting. Our approach begins with the observation that in typical trajectory-inference settings \citep{Lavenant2021}, each sample can be viewed as an i.i.d.\ draw from the \emph{joint} distribution of a system's state (e.g., mRNA expression level) and the time of measurement. Then, we characterize a parametric family of stochastic differential equations (SDEs) and seek the member of this family whose implied joint distribution over state and time best matches the empirical joint distribution observed in the data. We perform this matching using maximum mean discrepancy (MMD), a kernel-based measure of distance between distributions that allows us to compare the model-predicted state-time distribution to the observed one without requiring access to individual trajectories or likelihoods. With a specific kernel choice, we show that this matching problem reduces to a distributional least-squares objective. Our formulation yields three key benefits: (1) accurate forecasts beyond the observed time horizon, (2) robust interpolation and model learning even with incomplete state measurements, and (3) interpretable model outputs, including an explicit velocity field and an $R^2$-like metric that quantifies model fit.

We evaluate SnapMMD across a range of synthetic and real-world systems, including a synthetic gene regulatory network and real time-course single-cell RNA sequencing datasets. Our primary focus is on forecasting, where SnapMMD consistently outperforms Schrödinger bridge baselines, providing more accurate one-step ahead predictions beyond the observed time window. We also assess SnapMMD on interpolation tasks --- inferring intermediate dynamics between observed snapshots. In these experiments, our method matches or exceeds the performance of both SB-based and flow-matching baselines in almost all cases. Crucially, SnapMMD handles incomplete state measurements with ease, reconstructs an interpretable velocity field that supports downstream scientific analysis, and provides an $R^2$-style metric to evaluate model fit. Code to reproduce the experiments is available at \href{https://github.com/YunyiShen/snapMMD}{https://github.com/YunyiShen/snapMMD}.

\section{Setup and Background}
\label{sec:setup}
Though our work has application beyond scRNA-seq, we next describe our data and goals using scRNA-seq terminology to clarify and concretize our notation.

\textbf{Data Setup.}  
We consider single-cell mRNA measurements collected at $\totsteps$ distinct time points, labeled $\timestep{1} < \timestep{2} < \cdots < \timestep{\totsteps}$. For convenience, we set $\timestep{1} = 0$. We do not require these time points to be equally spaced. At each time $\timestep{\timeidx}$, the observed data consist of $\tottrajec_{\timeidx}$ cells, each providing a single mRNA expression level measurement (representing a cell state) in $\mathbb{R}^d$, denoted $\obsat{\timestep{\timeidx}}^{\sampleidx}$. After a cell's mRNA level is measured, the cell is destroyed. So each cell appears exactly once in the dataset. We therefore collect $\tottrajec = \sum_{\timeidx=1}^{\totsteps} \tottrajec_{\timeidx}$ total observations across all time points. We write $\obsatall{\timestep{\timeidx}}$ for the full set of measurements taken at time $\timestep{\timeidx}$.

\textbf{Goal.} If a cell's mRNA expression level were not measured, the cell would remain alive, and its mRNA expression would evolve continuously over time along a (latent) trajectory. 
Formally, we denote the latent trajectory of the $\sampleidx$th cell observed at the $\timeidx$th time step by $\responseat{\timet}^{(\timeidx, \sampleidx)}$. The observed state of the cell at time $\timestep{\timeidx}$ is a single point on this trajectory $\obsat{\timestep{\timeidx}}^{\sampleidx} = \responseat{\timet = \timestep{\timeidx}}^{ (\timeidx, \sampleidx) }$. We assume these latent trajectories are independent realizations from an underlying latent distribution. 
Consequently, the observed snapshot measurements are also independent.
Our objective is
to infer a probabilistic model, chosen from a specified parametric family, that best captures the distribution of these unobserved trajectories. Our model should provide a distribution over forecasted trajectories beyond the measured time and also over interpolated trajectories between observed times.

\textbf{Dataset and independence assumptions.} By the assumptions above, the dataset $\{\obsat{\timestep{\timeidx}}^{\sampleidx}\}_{\timeidx,\sampleidx}$ represents independent observations. But cells measured at different times can be expected to reflect different state distributions. So, if we treat the observation times of data points as fixed and known, it's not reasonable to assume the observations are both independent and identically distributed (i.i.d.). However, a key insight is that we can treat the observation times 
themselves as random draws from an underlying distribution $\timemeasure(\timestep{})$. Then a dataset of pairs $\{(\obsat{\timestep{\timeidx}}^{\sampleidx}, \timestep{\timeidx})\}_{\timeidx,\sampleidx}$ can be treated as i.i.d.\ samples from a joint state-time distribution. While it follows that $\{\obsat{\timestep{\timeidx}}^{\sampleidx}\}_{\timeidx,\sampleidx}$  are now i.i.d.\ as well, using the pairs $\{(\obsat{\timestep{\timeidx}}^{\sampleidx}, \timestep{\timeidx})\}_{\timeidx,\sampleidx}$ as our data aligns better with real biological experiments. Instead of sequencing each sample immediately after collection, experimenters often tag each sample with a unique identifier that encodes the time it was collected. Then, all cells are sequenced together in a single batch, with the time information retrieved from the tags. 

\textbf{Model.} In addition to the time distribution $\timemeasure(\timestep{})$, we model each latent trajectory of the $(\timeidx, \sampleidx)$ cell with a stochastic differential equation (SDE) driven by a $d$-dimensional Brownian motion $\brownianm^{ (\timeidx,\sampleidx) }$, independent across particles:
\begin{equation}
    \de\responseat{\timet}^{ (\timeidx,\sampleidx) } = \truedrift(\responseat{\timet}^{ (\timeidx,\sampleidx) },t) \de\timet + \truevolatility(\responseat{\timet}^{ (\timeidx,\sampleidx) },\timet) \de\brownianm,~~\responseat{\timet = 0}^{ (\timeidx,\sampleidx) }\sim \marginal{0}.
    \label{eq:mainsde}
\end{equation}
We assume that the drift $\truedrift(\cdot,\cdot): \R^{d}\times [0,\timet_{\totsteps}] \to \R^{d}$ and initial marginal distribution $\marginal{0}$ are unknown. 
Previous Schrödinger bridge methods typically assume fixed, known volatility \citep[e.g.][]{de2021diffusion,Vargas2021,koshizuka2022neural,Wang2023,shen2024multi,zhang2024joint,guan2024identifying,chen2024deep}. By contrast, we allow the common case where the volatility function $\truevolatility$ can be unknown and also state- and time-dependent.
For example, in scRNA-seq, transcriptional noise varies with gene identity, cell state, and developmental stage --- and is further confounded by technical artifacts such as amplification bias and stochastic capture. Volatility in these systems reflects true biological uncertainty and is rarely known in advance.

Finally, we assume standard SDE regularity conditions. The first assumption below ensures a strong solution to the SDE exists; see \citet[][Chapter 3, Theorem 3.1]{pavliotis2016stochastic}. The second ensures that the process does not exhibit unbounded variability.
\begin{assumption}
\label{assumption-lipschitz}
The drifts and volatility are $L$ and $L'$-Lipschitz respectively; i.e., for all $t\in [0,\timet_{\totsteps}]$, $\bm{x_1}, \bm{x_2} \in \R^d$, $\|\truedrift(\bm{x_1}, t)-\truedrift(\bm{x_2}, t)\| \le L \| \bm{x_1}-\bm{x_2} \|$ and $|\truevolatility(\bm{x_1}, t)-\truevolatility(\bm{x_2}, t)|\le L' \| \bm{x_1}-\bm{x_2} \|$, where $\|\cdot \|$ denotes the usual $L^2$ norm of a vector. And we have at most linear growth; i.e., there exist $K, K'<\infty$ and constant $c$ such that $\| \truedrift(\bm{x_1}, t)\| <K \| \bm{x_1} \|+c$ and $\| \truevolatility(\bm{x_1}, t)\| <K' \| \bm{x_1} \|+c'$. 
\end{assumption} 
\begin{assumption}
\label{assumption-bddsecondmoments}
At each time step $\timestep{\timeidx}$, the distribution
of the $\tottrajec_{\timeidx}$ particles has bounded second moments. Moreover, the initial distribution $\marginal{0}$ also has bounded second moments.
\end{assumption}

\textbf{Forecasting with Schrödinger bridges often reduces to the pre-set reference.} In principle, SB methods reconstruct distributions of trajectories by solving a constrained optimization problem that matches observed marginals using a Kullback–Leibler divergence to a predefined reference process, typically Brownian motion; see \cref{eq:sb-problem} in \cref{app:sb-forecasting} for a detailed formulation. However, as we formalize in \cref{prop:sb-forecasting}, forecasting with many SB methods reduces to propagating the final observed marginal distribution forward according to the reference dynamics. As a result, SB forecasts inherently depend only on the reference SDE and the last observed snapshot. This lack of non-trivial data dependence limits flexibility and predictive power.

\section{Our method}
\label{sec:method}

To address the forecasting limitations of SBs described above, we propose an alternative approach.
SBs compare the observed data and candidate model directly via Kullback–Leibler divergence.
We instead formulate an optimization problem that directly matches the joint distribution of state–time pairs predicted by a candidate model to the empirical distribution observed in the data. Below, we describe precisely how we frame and solve this optimization problem using MMD, introduce an interpretable diagnostic metric for assessing model fit, and detail how this framework naturally extends to scenarios involving incomplete state measurements.

\subsection{A least squares approach}
We formalize our approach by (1) decomposing the joint state--time empirical distribution into marginal and conditional components and (2) detailing how the resulting matching problem  reduces to a least-squares formulation when using Maximum Mean Discrepancy (MMD) for distance.

\textbf{Empirical distributions.}
We let $\empconditional{\statevar}{\timestep{}}$ denote the empirical measure over state $\statevar$ at any observed time $\timestep{\timeidx}$. For unobserved times, we give it a placeholder distribution. Likewise, we let $\emptimemeasure(\timestep{})$ denote the (marginal) empirical measure over observed times. Precisely, we have
\begin{equation}
    \label{eq:empirical_dists}
\empconditional{\statevar}{\timestep{}} = \left\{ \begin{array}{ll} (\tottrajec_{\timeidx})^{-1} \sum_{\sampleidx=1}^{\tottrajec_{\timeidx}} \delta_{\obsat{\timet_{\timeidx}}^{\sampleidx}}(\statevar) & \timestep{} = \timestep{\timeidx} \\
    \delta_{\bm{0}}(\statevar) & \textrm{else}
    \end{array}
    \right.,
    \quad \textrm{ and } \quad
        \emptimemeasure(\timestep{}) = \sum_{\timeidx=1}^{\totsteps} \left(\frac{\tottrajec_{\timeidx}}{\sum_{j=1}^{\totsteps} \tottrajec_{j}}\right) \delta_{\timestep{\timeidx}}(\timestep{}). 
\end{equation}
Let $\empjoint{\statevar}{\timestep{}}$ denote the empirical joint distribution over observed state-time pairs. This joint decomposes into the empirical marginal and conditional described above: $\empjoint{\statevar}{\timestep{}} = \emptimemeasure(\timestep{}) \empconditional{\statevar}{\timestep{}}$.

\textbf{Directly matching the joint state--time distributions.} We plan to estimate the parameters \(\allparam\) of a candidate SDE model by aligning its predicted joint distribution over state--time pairs with the empirical distribution. 
Let $\predconditional{\statevar}{\timestep{}}$ denote the predicted state distribution at time $t$. We will minimize a discrepancy between (1) the empirical joint $\empjoint{\statevar}{\timestep{}}$ and (2) the joint implied by the predictive conditional $\predconditional{\statevar}{\timestep{}}$ together with the empirical marginal $\emptimemeasure(\timestep{})$: namely, $\predjoint{\statevar}{\timestep{}} := \emptimemeasure(\timestep{}) \predconditional{\statevar}{\timestep{}}$.

For the discrepancy, we choose the MMD \citep{gretton2012kernel}. MMD computes the squared distance between the kernel mean embeddings of two distributions in a reproducing kernel Hilbert space (RKHS). 
We choose a kernel that factors across state and time; this choice lets us break the joint MMD into a weighted sum of marginal MMDs at each time point. The resulting optimization objective is reminiscent of least-squares regression, with time acting as a discrete index. We formalize this idea in the following result.
\begin{proposition}
\label{prop:MMDdecomposition}
Let \(f(\statevar,t) = f(\statevar\mid t) h(t)\) and \(g(\statevar,t) = g(\statevar\mid t) h(t)\)
be joint distributions over \(\statevar\in \resprange\) (for dimension $d$) and discrete time \(t \in \timerange\), where \(h(t)\) is a probability mass function and $f(\statevar\mid t),  g(\statevar\mid t)$ are conditional distributions. Use the kernel \(K((\statevar,t), (\statevar',t')) = K_{\statevar}(\statevar,\statevar') \delta(t - t')\), where \(K_{\statevar}\) is positive definite on the state space and, for all $t\in \timerange$, $\E_{\statevar\sim f(\statevar\mid t), \statevar'\sim f(\statevar\mid t)}K_{\statevar}(\statevar,\statevar')<\infty$, $\E_{\statevar\sim f(\statevar\mid t), \statevar'\sim g(\statevar\mid t)}K_{\statevar}(\statevar,\statevar')<\infty$, and $\E_{\statevar\sim g(\statevar\mid t), \statevar'\sim g(\statevar\mid t)}K_{\statevar}(\statevar,\statevar')<\infty$.\footnote{Many practical kernels satisfy these assumptions, e.g.,  radial basis function, Mat\'ern, and Laplace.} Then:
\[
\MMD_K^2(f, g) = \sum_{t \in \timerange} h^2(t) \, \MMD_{K_{\statevar}}^2\left(f(\cdot\mid t), g(\cdot\mid t)\right).
\]
\end{proposition}

This result (proof in \cref{app:proof-main-prop}) shows that aligning joint distributions here boils down to matching conditional state distributions across time. 
We apply \cref{prop:MMDdecomposition} with the two joint distributions from above, $\predjoint{y}{\timestep{}}$ and $\empjoint{y}{\timestep{}}$, to obtain the following optimization objective:
\begin{align}
   \MMD_K^2(\predjointnoarg, \empjointnoarg)
=\sum_{\timeidx=1}^{\totsteps} \left(\frac{\tottrajec_{\timeidx}}{\sum_{j=1}^{\totsteps}\tottrajec_{j}}\right)^2\,\MMD_{K_{\statevar}}^2 ( \predconditional{\cdot}{\timestep{\timeidx}}, \empconditional{\cdot}{\timestep{\timeidx}} ).
\end{align}
It remains to estimate the righthand MMDs and also to choose the MMD state-space kernel, the model $\predconditional{\cdot}{\timestep{}}$, and the optimization algorithm.

\textbf{Estimating MMD at each time point.} In practice, we approximate $\MMD_{K_{\statevar}}^2 ( \predconditional{\cdot}{\timestep{}}, \empconditional{\cdot}{\timestep{}} )$ using the MMD's U-statistic estimator (Lemma 6, \citep{gretton2012kernel}). To that end, we simulate $\totsimummd$ trajectories from the candidate model. For the $\simuidx$th trajectory, we record the state snapshot $\altresponseat{\timestep{\timeidx}}^{\simuidx}$ at time $\timestep{\timeidx}$. 
The U-statistic estimator, which is unbiased and consistent \citep{hall2004generalized}, is then given by
\[
\begin{aligned}
\MMDu( \predconditional{\cdot}{\timestep{\timeidx}}, \empconditional{\cdot}{\timestep{\timeidx}} )
=&\,\frac{1}{\tottrajec_{\timeidx}(\tottrajec_{\timeidx}-1)}\sum_{\sampleidx\ne \sampleidx'} K_{\statevar}({\obsat{\timestep{\timeidx}}^{\sampleidx}},{\obsat{\timestep{\timeidx}}^{\sampleidx'}}) -\frac{2}{\tottrajec_{\timeidx}\totsimummd}\sum_{\sampleidx,\simuidx}K_{\statevar}({\obsat{\timestep{\timeidx}}^{\sampleidx}},{\altresponseat{\timestep{\timeidx}}^{\simuidx}}) \\
&\quad +\frac{1}{\totsimummd(\totsimummd-1)}\sum_{\simuidx\ne \simuidx'}K_{\statevar}({\altresponseat{\timestep{\timeidx}}^{\simuidx}},{\altresponseat{\timestep{\timeidx}}^{\simuidx'}}).
\end{aligned}
\]

The overall optimization problem for parameter fitting then becomes:
\begin{align}
    \label{eq:leastsquare}
\hat{\allparam} = \argmin_{\allparam}\; \sum_{\timeidx=1}^{\totsteps} w_{\timeidx} \MMDu( \predconditional{\cdot}{\timestep{\timeidx}}, \empconditional{\cdot}{\timestep{\timeidx}} )
\quad
\textrm{ with weights }
    \quad
    w_{\timeidx} := \left(\frac{\tottrajec_{\timeidx}}{\sum_{j=1}^{\totsteps}\tottrajec_{j}}\right)^2.
\end{align}
The least squares objective from \cref{eq:leastsquare} naturally extends classical regression to distributional settings by measuring discrepancy in the RKHS defined by kernel mean embeddings. Specifically, this least squares framework reduces to classical Euclidean regression if each predicted distribution is a Dirac measure and the kernel is linear. 

\textbf{Choosing kernel, optimizer, and SDE model.}
In practice, we use the radial basis function (RBF) kernel for the state space; we determine the length scale by the median heuristic \citep{garreau2017large} applied to pairwise distances in the data. Additionally, we scale time to lie in  $[0,1]$. For optimization, we compute  gradients with respect to the parameters using the stochastic adjoint method \citep{li2020scalable}. We use the Adam optimizer to perform the parameter updates. 

\begin{table}[t]
\centering
\caption{MMD for forecast and interpolation tasks. For forecast, we report mean and standard deviation over 10 random seeds. For interpolation, we aggregate over 10 seeds and validation time points. The best method (lowest MMD) is shown in bold in a green cell; we also highlight in green any other methods whose mean is contained in the one-standard deviation interval for the best method. We show the top three interpolation methods here; full results can be found in \protect\cref{app:complete-summary-interpolation}.}
\small
\begin{tabular}{lcccccc}
\toprule
      & \multicolumn{3}{c}{\textbf{Forecast}} & \multicolumn{3}{c}{\textbf{Interpolation}} \\ 
\cmidrule(lr){2-4}\cmidrule(lr){5-7}
\textbf{Task}
        & \textbf{Ours} & \texttt{SBIRR-ref} & \texttt{SB-forward}
        & \textbf{Ours} & \texttt{SBIRR} & \texttt{DMSB} \\ \midrule
LV
      & \best{\ms{0.01}{0.01}} & \ms{0.14}{0.02} & \ms{0.71}{0.50}
      & \best{\ms{0.02}{0.01}} & \ms{0.04}{0.03} & \ms{1.15}{0.40} \\ \midrule
ReprParam
      & \best{\ms{0.02}{0.02}} & \ms{0.47}{0.30} & \ms{0.42}{0.20}
      & \best{\ms{0.04}{0.04}} & \ms{0.16}{0.10} & \ms{1.57}{0.40} \\
ReprSemiparam
      & \best{\ms{0.08}{0.03}} & \ms{0.29}{0.11} & \ms{1.15}{0.30}
      & \best{\ms{0.31}{0.40}} & \greencell{\ms{0.48}{0.30}} & \ms{1.57}{0.40} \\
ReprProtein
      & \best{\ms{0.01}{0.01}} & \ms{1.26}{0.06} & \ms{1.22}{0.09}
      & \best{\ms{0.02}{0.02}} & \ms{0.34}{0.30} & \ms{1.28}{0.40} \\ \midrule
GoM
      & \ms{0.66}{0.03} & \best{\ms{0.35}{0.03}} & \ms{0.62}{0.05}
      & \ms{0.29}{0.20} & \best{\ms{0.07}{0.06}} & \ms{0.15}{0.09} \\ \midrule
PBMC
      & \best{\ms{0.00}{0.00}} & \ms{0.10}{0.06} & \ms{0.51}{0.20}
      & \best{\ms{0.01}{0.01}} & \greencell{\ms{0.01}{0.01}} & \ms{0.56}{0.10} \\
\bottomrule
\end{tabular}
\label{tab:mmd_main_text}
\end{table}

We pick the candidate conditional $\predconditional{\cdot}{\timestep{}}$ based on domain knowledge about the underlying process. This SDE can either follow a parametric form, in which drift and volatility are governed by functions with finitely many parameters (as in the Lotka--Volterra experiment; see \cref{sec:lv-main}), or a more flexible design that incorporates neural network architectures for the drift or volatility terms. For example, in the repressilator experiment (\cref{sec:repr-main}), we compare a purely parametric model to a semiparametric approach in which we used a multilayer perceptron to approximate part of the drift in the system; see \cref{eq:mlpactive} for more details. This neural network component allows the model to capture complex, nonlinear effects that would be difficult to represent in a simple parametric setting.

\textbf{Handling incomplete state measurements.} In many practical applications --- such as our mRNA sequencing example --- it is common for only a subset of relevant state variables to be observed. For instance, while mRNA concentrations are routinely measured, corresponding protein levels (which are also important for modeling the underlying dynamical system) are often unavailable. Our framework can handle these missing-data settings because it relies on matching the joint distribution of time and the observed dimensions, rather than requiring all dimensions to be measured. More precisely, since our loss (\cref{eq:leastsquare}) is defined over the observed state variables (together with time), the model is trained to match the marginal distribution of the observed variables along with time, without making any additional assumptions or imputations for the missing dimensions. We illustrate with an experiment in \cref{sec:experiments} (\cref{fig:repr-main}, lower row).

\subsection{Evaluating model fit with an $R^2$-style metric}
\label{sec:r2-main}

In traditional regression, the standard \(R^2\) metric offers a  straightforward diagnostic and basis for model comparison; we propose a similar metric, but for distributional data.
In particular, recall that the \(R^2\) metric quantifies how well a model explains the data by comparing residual variability against total variability around a baseline constant prediction, namely the mean of the data responses. Given the least-squares objective of \cref{eq:leastsquare}, we might consider a similar metric in the present distributional case. But first we need to choose an appropriate baseline.

\textbf{A distributional baseline.}
Analogous to the response-mean baseline in traditional regression, we want to find the best constant (time-independent) model within our RKHS-based least squares framework. The barycenter of distributional data is the distribution minimizing the sum of squared RKHS distances to all observed distributions. For distributions embedded in an RKHS, \citet{cohen2020estimating} showed that the barycenter is  the weighted mixture of the empirical distributions:
\begin{equation}
    \label{eq:barycenter}
    \barycentershort(\statevar) := \argmin_{\marginaldenplain} \sum_{\timeidx=1}^{\totsteps} w_{\timeidx}\, \MMD_{K_{\statevar}}^2 ( f(\statevar\mid \timestep{\timeidx}), \empconditional{\statevar}{\timestep{\timeidx}} ) = \frac{1}{\sum_{\timeidx=1}^{\totsteps} w_{\timeidx}\tottrajec_{\timeidx}}\sum_{\timeidx=1}^{\totsteps}\sum_{\sampleidx=1}^{\tottrajec_{\timeidx}}w_{\timeidx}\delta_{\obsat{\timet_{\timeidx}}^{\sampleidx}}(\statevar),
\end{equation}
with weights \(w_{\timeidx}\) as in \cref{eq:leastsquare}.

\textbf{Our RKHS-based \(R^2\) metric.}
With this baseline in hand, we define our metric.

\begin{definition}[RKHS-based \(R^2\) metric]
\label{def:R2_metric}

Let $\predconditional{\cdot}{\timestep{}}$ denote the model-predicted state distribution at time $\predconditional{\cdot}{\timestep{}}$. Let $ \empconditional{\cdot}{\timestep{\timeidx}}$ be defined as in \cref{eq:empirical_dists}, 
\(\barycentershort\) as in \cref{eq:barycenter}, and \(w_{\timeidx}\) as in \cref{eq:leastsquare}.
The RKHS-based coefficient of determination \(R^2\) is
\begin{equation}
    \label{eq:R2_metric}
    R^2 = 1 - \frac{\sum_{\timeidx=1}^{\totsteps} w_{\timeidx}\,\MMD_{K_{\statevar}}^2(\predconditional{\cdot}{\timestep{\timeidx}}, \empconditional{\cdot}{\timestep{\timeidx}})}{\sum_{\timeidx=1}^{\totsteps} w_{\timeidx}\,\MMD_{K_{\statevar}}^2(\barycentershort, \empconditional{\cdot}{\timestep{\timeidx}})}.
\end{equation}
\end{definition}

The numerator captures how well the candidate model explains the observed data by measuring the discrepancy between predicted and empirical distributions at each time. The denominator quantifies the total variability around the barycenter, akin to total variance in standard regression. Thus, \(R^2\) represents the fraction of total variability explained by the candidate model. 

The \(R^2\) metric defined in \cref{eq:R2_metric} naturally provides a standardized and interpretable criterion to compare candidate SDE models, guiding model selection and diagnosing fit quality. Values of \(R^2\) close to 1 indicate a high-quality fit, while values near or below 0 signal poor model performance relative to the simple barycenter model. Finally, note that \(R^2\) is always upper bounded by 1 due to the non-negativity of the MMD, and it may become negative if the candidate model performs worse than the barycenter, analogous to regression models without an intercept. We discuss how we use this metric in our experiments in \cref{app:use-r2-experiments}.

\section{Experiments}
\label{sec:experiments}
\begin{figure}[t!]
    \centering
    \includegraphics[width=\linewidth]{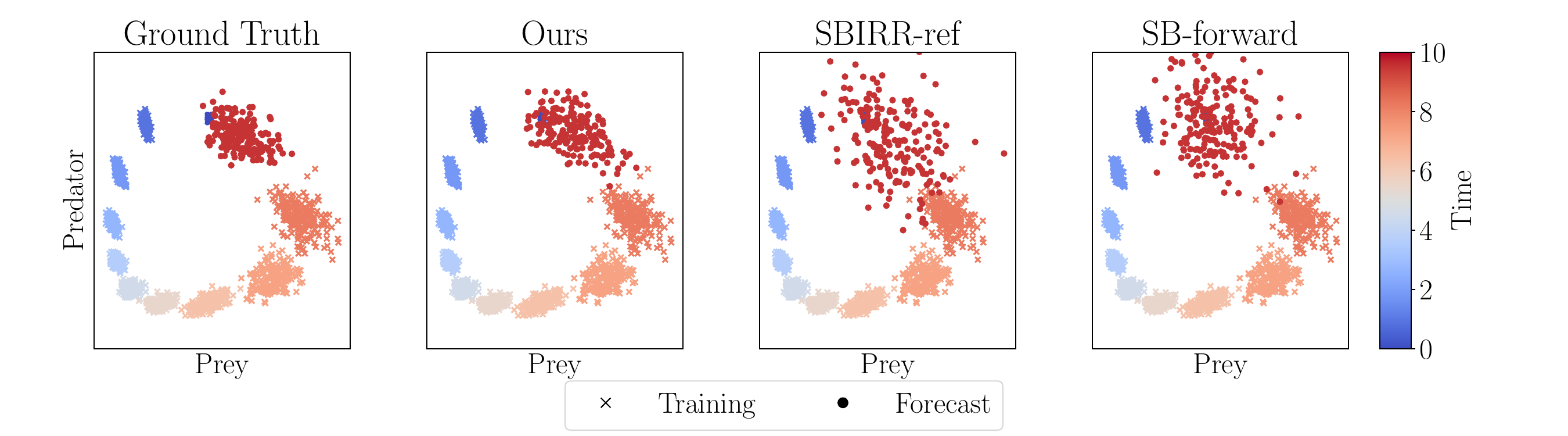}
    \caption{Lotka-Volterra results (\protect\cref{sec:lv-main}). We show 200 samples at each of 10 training times and 1 forecast time (red). Forecast points overlap with the training points at time 0 (blue).}
    \label{fig:LV}
\end{figure}

In synthetic and real-data experiments,
we find that our SnapMMD method consistently provides better forecasts than competitors. We also find that, in almost all experiments, SnapMMD provides better or matching interpolation performance relative to competitors. In \cref{app:vector-field-reconstruction}, we further demonstrate that our method outperforms competitors on vector field reconstruction.

\textbf{Beyond cell states.} Though we use cell-state terminology above for concreteness, our experiments also include applications beyond cellular dynamics. E.g., states can instead represent predator and prey counts (\cref{sec:lv-main}) or spatial locations of particles following ocean currents (\cref{sec:gom-main}).

\textbf{Metrics of success.} To evaluate forecasting performance, we reserve a validation snapshot at a future time point beyond the training horizon.
For interpolation, we retain intermediate validation snapshots between training time points.
In both forecasting and interpolation tasks, we measure discrepancy between the validation data and predictions with two metrics: (1) the MMD and (2) the earth mover's distance (EMD)\footnote{For EMD, we use the implementation provided by \citet{trajnet}.} between the forecast distribution and the held-out empirical distribution at the validation time. We also provide visual comparisons.

\textbf{Forecasting (and vector field reconstruction) baselines.} We compare SnapMMD against two baselines: (1) Schrödinger bridge with iterative reference refinement (\texttt{SBIRR}) \citep{shen2024multi, zhang2024joint, guan2024identifying}, and (2) multimarginal Schrödinger bridge with shared forward drift (\texttt{SB-forward}) \citep{shen2024learning}. 
While these methods were designed and tested for interpolation (\texttt{SBIRR}) and vector field reconstruction (\texttt{SB-forward}), we can adapt them for forecasting. To forecast beyond the observed time horizon, we use (1) the best fitted reference of \texttt{SBIRR} (henceforth \texttt{SBIRR-ref}), and (2) the best-fitted forward drift of \texttt{SB-forward}.
Reference fitting is also used by \citet{zhang2024joint} and \citet{guan2024identifying}, but they focused exclusively on linear models, while \texttt{SBIRR-ref} allows general model families.
\texttt{SBIRR-ref} and \texttt{SB-forward} require a fixed, known volatility; we set it to 0.1 as in prior work \citep{Vargas2021, wang2021deep, shen2024multi}.

\begin{figure}[t!]
    \centering
    \includegraphics[width=\linewidth]{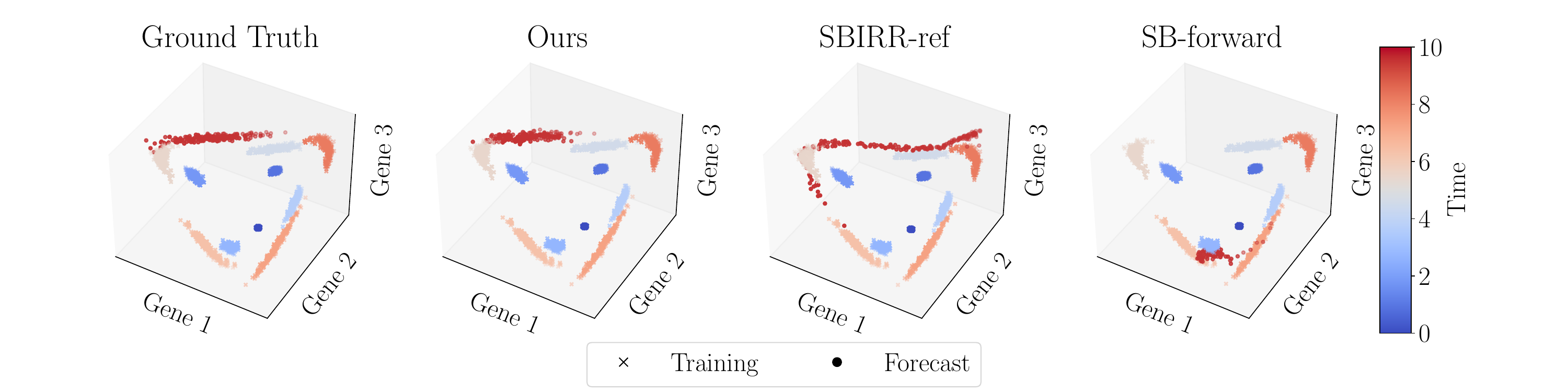}
    \includegraphics[width=\linewidth]{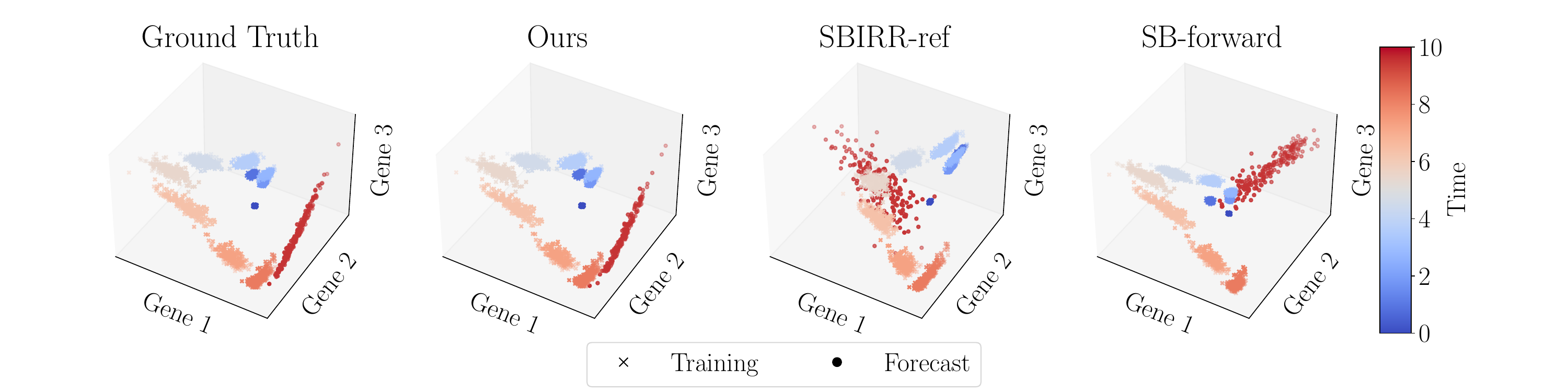}
    \protect\caption{Repressilator results: mRNA-only (upper, \protect\cref{sec:repr-main}) and mRNA and protein (lower, \protect\cref{sec:repr-protein}). We show 200 samples at each of 10 training times and 1 forecast time (red).}
    \label{fig:repr-main}
\end{figure}

\textbf{Interpolation baselines.} For interpolation, we compare SnapMMD against five methods: 
optimal transport-conditional flow matching (\texttt{OT-CFM}) and Schrödinger bridge-conditional flow matching (\texttt{SB-CFM}) \citep{tong2023improving}, simulation-free Schrödinger bridge (\texttt{SF2M}) \citep{tong2024simulation}, deep momentum multimarginal Schrödinger bridge (\texttt{DMSB}) \citep{chen2024deep}, and again to the Schrödinger bridge with iterative reference refinement (\texttt{SBIRR}) \citep{shen2024multi}, where now we use the output of the main algorithm (rather than the learned reference as in \texttt{SBIRR-ref}).
For each method, we use default code settings. We describe additional related work in \cref{app:related-work}.

\subsection{Lotka--Volterra system}
\label{sec:lv-main}

\textbf{Setup.} We simulated data from a two-dimensional Lotka–Volterra predator–prey system, where each coordinate’s volatility scales proportionally with its state variable. E.g., we set the volatility for the prey population $X$ to be $\sigma X$, with the same constant $\sigma$ across predator and prey. We train on 10 time points, each with 200 samples. For methods that take a model choice (ours and \texttt{SBIRR} variants), we use a parametric Lotka--Volterra model. See \cref{app:lv-app} for full details.

\textbf{Results.} In \cref{fig:LV}, we see that our method's forecast (red dots) is closer to ground truth than the baselines are. MMD (LV, Forecast in \cref{tab:mmd_main_text}) and EMD (LV, Forecast in \cref{tab:emd_combined}) agree that our method performs best. Our method is also best in the interpolation task (LV, Interpolation in \cref{tab:mmd_main_text}). See \cref{app:lv-app} for further results.

\subsection{Repressilator: mRNA only}
\label{sec:repr-main}
\textbf{Setup.} We simulated mRNA concentration data from a repressilator system, a biological clock composed of three genes that inhibit each other in a cyclic manner. As for Lotka--Volterra, we let each coordinate’s volatility scale proportionally with its state variable. We train on 10 time points, each with 200 samples. For methods that take a model choice, we consider two options. (1) We use the same parametric model as the data-generating process; see \cref{app:repr-app-parametric} for full results. (2) We use a semiparametric model with a multilayer perceptron; see \cref{app:repr_implementation-semiparametric} for details.

\textbf{Results.}
In \cref{fig:repr-main} (upper row), we see that, when using the semiparametric model, our method's forecast (red dots) is closer to ground truth than the baselines are. MMD (ReprSemiparam, Forecast in \cref{tab:mmd_main_text}) and EMD (ReprSemiparam, Forecast in \cref{tab:emd_combined}) agree that our method performs best. Our method is also tied for best in the interpolation task (ReprSemiparam, Interpolation in \cref{tab:mmd_main_text}). We find similar results when using the parametric model (\cref{tab:mmd_main_text,tab:emd_combined}, \cref{fig:repr-parametric-forecast}). See \cref{app:repr-app-parametric} and \cref{app:repr-app-semiparametric} for further results. 

\subsection{Repressilator: mRNA and protein}
\label{sec:repr-protein}
\textbf{Setup.}
A more-complete biochemical model of the repressilator includes both mRNA and protein (\cref{eq:repr-family-protein}) even though only mRNA concentration is actually observed in practice. We next generate simulated data using the more-complete model and keep only the mRNA concentrations in our observations. We again train on 10 time points, each with 200 samples. Since our method has the capacity to handle incomplete state observations, we can use the full mRNA-protein model with our method.
Since \texttt{SBIRR} methods do not have the capacity to handle models with latent variables, we use the mRNA-only model in these methods. 

\begin{figure}[!t]
    \centering
\includegraphics[width=\linewidth]{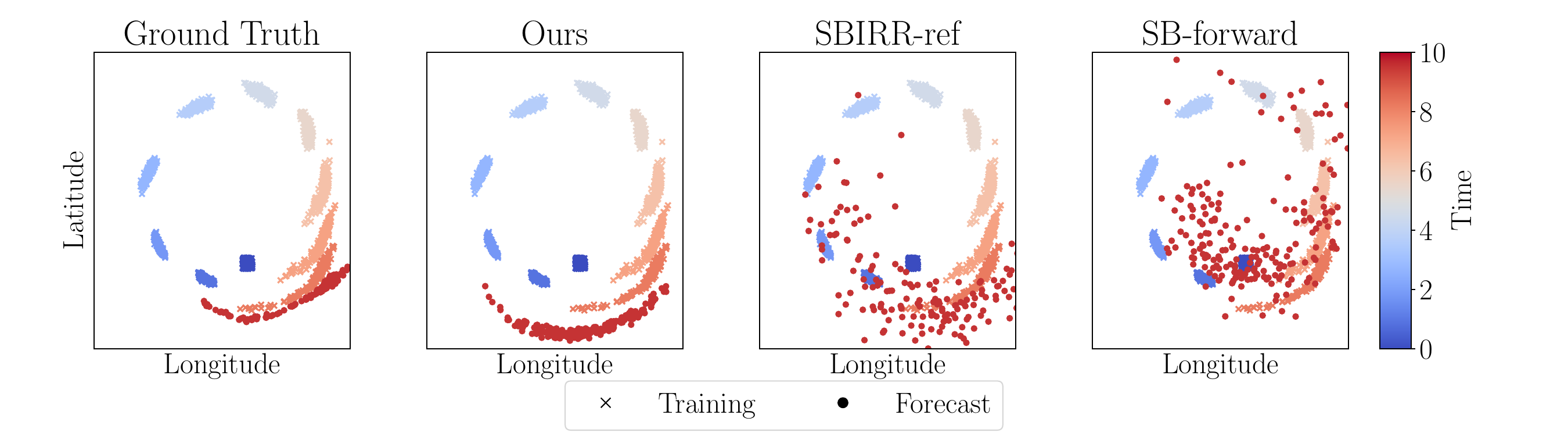}
    \caption{Gulf of Mexico results (\protect\cref{sec:gom-main}). We show 200 samples at each of 10 training times and 1 forecast time (red). }
    \label{fig:GoMforecast}
\end{figure}

\textbf{Results.} In \cref{fig:repr-main} (lower row), we see that our method's forecast (red dots) is closer to ground truth than the baselines are. MMD (ReprProtein, Forecast in \cref{tab:mmd_main_text}) and EMD (ReprProtein, Forecast in \cref{tab:emd_combined}) agree that our method performs best. We emphasize that none of the methods directly observe protein levels. But since our method is aware that protein levels are also driving the underlying dynamics, it is able to better forecast mRNA concentration. Our method is also best in the interpolation task (ReprProtein, Interpolation in \cref{tab:mmd_main_text}). See \cref{app:repr-app-missing} for further results.

\subsection{Ocean currents in the Gulf of Mexico}
\label{sec:gom-main}
\textbf{Setup.}
We use real ocean-current data from the Gulf of Mexico: namely, high-resolution (1 km) bathymetry data from the HYbrid Coordinate Ocean Model (HYCOM) reanalysis.\footnote{Dataset available at \url{https://www.hycom.org/data/gomb0pt01/gom-reanalysis}.} 
We extract a velocity field centered on a region that appears to exhibit a vortex. We then simulate the motion of particles --- representing buoys or ocean debris --- evolving under this field. The training data consist of 10 time points with 400 particles each. Since the data is real in this experiment, the models used by any method must be misspecified. For methods that take a model choice, we use a physically motivated model for the vortex, where the velocity field is the sum of a Lamb-Oseen vortex and a constant divergence field. The first term accounts for swirling, rotational dynamics typical of a vortex in low viscosity fluid like water, while the divergence field accounts for vertical motion or non-conservative forces that may cause a net expansion or contraction of the flow. See \cref{app:gom_implementation} for more details.

\textbf{Results.} In \cref{fig:GoMforecast}, we see that our method's forecast (red dots) more closely aligns with ground truth than the baselines do. EMD (GoM, Forecast in \cref{tab:emd_combined}) agrees that our method performs best. However, MMD (GoM, Forecast in \cref{tab:mmd_main_text}) prefers \texttt{SBIRR-ref} to our method (SnapMMD); 
we suspect we see this behavior because MMD using an RBF kernel can prefer a diffuse, but less accurate, cloud over a concentrated, geometrically correct one.
Recall that the MMD with RBF mixes two ingredients: (i) how tightly the forecast particles cluster among themselves and (ii) how far the forecast and ground truth particles are. Our forecast points sit on a lower-dimensional curve than the baselines' points, so the ``self-similarity'' part of the MMD score is higher. 
All methods perform well visually at the interpolation task (\cref{fig:GoM-interpol}), but \texttt{SBIRR} yields the best MMD (GoM, Interpolation in \cref{tab:mmd_main_text}) and EMD (GoM, Interpolation in \cref{tab:emd_combined}). \texttt{SBIRR} is built to interpolate every observed snapshot and then smooth between them, so with densely sampled times, it almost inevitably lands near the held-out validation points. Our method's aim to recover a smooth velocity field (rather than enforce exact interpolation) can be an advantage for forecasting, but less so for interpolation. See \cref{app:gom-interpol} for further results.

\subsection{T cell-mediated immune activation}
\label{sec:pbmc}

\textbf{Setup.}
We use a real single-cell RNA-sequencing dataset that tracks T cell–mediated immune activation in peripheral blood mononuclear cells (PBMCs) \citep{jiang2024d}. Scientists recorded gene-expression profiles every 30 minutes for 30 hours. 
We use the 41 snapshots collected between 0 h and 20 h --- prior to the onset of steady state; we take 20 alternating snapshots (at integer hours) for training and the remaining 21 for validation.
We use the 30-dimensional projection (``gene program'') of the original measurements released by \citet{jiang2024d} as our data. In \cref{fig:pbmc-forecasting} (leftmost four), we show the training snapshots at four time steps. \cref{fig:pbmc-training-progression} shows the full progression of the training points over  20 time steps. Since the data is real in this experiment, the model used by any method must be misspecified. For methods that use a model, we use the same model as in the semiparametric repressilator experiment (\cref{app:repr-setup-semiparametric}). Full details are in \cref{app:pbmc-biology}.

\textbf{Results.}
In \cref{fig:pbmc-forecasting} (rightmost four), we see that our method's forecast is closer to ground truth than the baselines are. MMD (PBMC, Forecast in \cref{tab:mmd_main_text}) agrees. We do not use EMD in this experiment as it suffers from the curse of dimensionality and is thus unreliable in our 30-dimensional setting; see the discussion at the end of Section 2.5.2 in \citet{chewi2024statistical}. Our method ties with \texttt{SBIRR} in the interpolation task (PMBC, Interpolation in \cref{tab:mmd_main_text}). See \cref{app:pbmc-interpol} for more results.

\begin{figure}[!t] \centering \includegraphics[width=\linewidth]{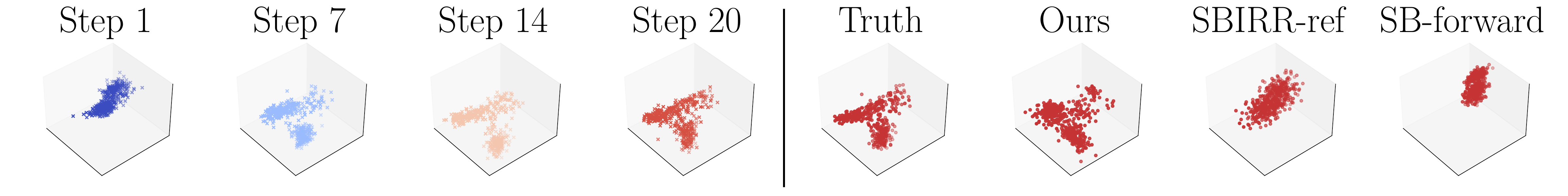} \caption{PBMC results (\protect\cref{sec:pbmc}). The axes in every plot are the same three principal components, computed over the full data: i.e., 41 time steps of the 30-dimensional gene programs.
Leftmost four panels: evolution of the training data at time steps 1, 7, 14, and 20. "Truth" panel: ground truth snapshot at time step 21. Rightmost three panels: model forecasts at time step 21.
} 
\label{fig:pbmc-forecasting} 
\end{figure}

\section{Discussion}
\label{sec:discussion}
In this work, we introduced a new method for learning SDEs from population-level snapshot data. Our approach is based on matching state--time distributions using a least squares scheme in a distributional space. Our proposed method handles real-life challenges such as unknown and state-dependent volatility, missing dimensions, and diagnostics for performance. Overall, our experiments indicate that our proposed framework outperforms existing methods in a wide range of applications. 

\textbf{Limitations.}
While many scientific applications (including those above) feature useful models, it can nonetheless be a limitation that our method requires specification of an appropriate model family. A poorly chosen model family can lead to bad performance. An overly wide family can face identifiability challenges. In fact, even the complete time series of marginal distributions need not always uniquely determine the drift and volatility functions of the SDE; see \cref{app:identifiability} for further discussion. Finally, our method is not simulation-free, so it requires compute that scales with sample size and state dimension; development of a simulation-free method is an interesting future direction.

\textbf{Social impact.} Our method might prove useful in areas of positive social impact, e.g.\ predicting a cell's reaction to a drug treatment or predicting the trajectory of oil particles in the ocean.

\section{Acknowledgements}
We would like to thank David Burt for helpful discussions during early stages of the work. We also thank Lazar Atanackovic and Katarina Petrovic for directing us to resources to implement the flow-matching baselines. This work was supported in part by
the Office of Naval Research under grant N00014-20-1-
2023 (MURI ML-SCOPE) and by the NSF TRIPODS
program (award DMS-2022448).

\bibliography{references.bib}

\section*{Supplementary Material}
\label{sec:appendix}
\appendix
\renewcommand{\theequation}{A\arabic{equation}}
\renewcommand{\thefigure}{A\arabic{figure}}  

\setcounter{equation}{0}
\setcounter{section}{0}
\setcounter{subsection}{0}
\setcounter{subsubsection}{0}

\section{Additional Related Work}
\label{app:related-work}
In this section, we discuss related work covering generative modeling, Schrödinger bridges, flow matching, and other approaches to trajectory inference.

\textbf{Generative modeling.} The machine learning community has made substantial progress in sampling from complex, unknown distributions by transporting particles using diffusion models \citep{ho2020denoising, song2019generative, song2020score}, Schrödinger bridges \citep{de2021diffusion, pavon2021data, Vargas2021, Wang2023}, continuous normalizing flows \citep{chen2018neural, grathwohl2018ffjord}, and flow matching \citep{lipman2022flow}. These methods are primarily designed for generative modeling, where the goal is to transform a simple distribution (e.g., a Gaussian) into a data distribution. While they involve a notion of “time,” it typically serves as an auxiliary dimension rather than representing physical or real-world temporal dynamics. As a result, these methods focus on two-marginal transport problems and are not naturally suited to forecasting or modeling systems with multiple observed time points.

\textbf{Trajectory inference with Schrödinger bridges.} Several recent works have explored the use of Schrödinger bridges (SBs) and optimal transport for modeling trajectories across time \citep{schiebinger2019optimal, yang2018scalable}. \citet{hashimoto2016learning} proposed a regularized recurrent neural network trained with a Wasserstein gradient flow loss for trajectory reconstruction. \citet{bunne2022proximal} extended this direction with \texttt{JKOnet}, a neural implementation of the Jordan–Kinderlehrer–Otto (JKO) scheme, enabling learning of the underlying energy landscape. \citet{tong2024simulation} introduced \texttt{SF2M}, a simulation-free SB method based on score and flow matching, designed for efficiency but limited to Brownian motion as the reference and two-marginal settings.

Most SB-based methods are restricted to pairwise interpolation between consecutive marginals and cannot handle multiple marginals in a principled way. Several extensions have been proposed to address this. \texttt{DMSB} \citep{chen2024deep} incorporates momentum into the particles to exploit local information and model multi-marginal dynamics, while \citet{hong2025trajectory} proposed generalizations using higher-order derivatives, although these approaches are computationally feasible only in low-dimensional settings. \citet{shen2024multi} introduced \texttt{SBIRR}, a multi-marginal Schrödinger bridge method with iterative reference refinement. By integrating information from all marginals into the reference dynamics, \texttt{SBIRR} enables not only interpolation but also extrapolation via learned dynamics, and serves as one of the main baselines in our experiments.

\textbf{Trajectory inference with flow matching and normalizing flows.} Other works use deterministic flows to infer trajectories. \texttt{TrajectoryNet} \citep{tong2020trajectorynet} combines dynamic optimal transport with continuous normalizing flows (CNFs) to generate continuous-time, nonlinear trajectories from snapshot data. These flows are governed by ODEs rather than SDEs, and incorporate regularization that encourages short, energy-efficient paths. \citet{tong2023improving} extended flow matching by modeling marginals as mixtures indexed by a latent variable. They proposed \texttt{OT-CFM} and \texttt{SB-CFM}, where each component evolves under a simple vector field and the overall dynamics are determined by mixing strategies based on optimal transport or SB principles. Like many flow-matching methods, these approaches are restricted to two marginals and are applied piecewise for longer time series. To address issues of geometry in the data space, \citet{huguet2022manifold} proposed \texttt{MIOFlow}, which first learns the data manifold and then solves the flow problem within that learned structure. They use neural ODEs to transport points along this manifold. \citet{atanackovic2024meta} introduced methods to learn multiple flows over a Wasserstein manifold, though these too are aimed at interpolation rather than forecasting.

Overall, these methods move particles using learned vector fields and various forms of regularization, but they are generally not designed for forecasting. Like many SB-based methods, they suffer when applied to extrapolation tasks, particularly in the absence of a good prior or reference dynamic.

\section{Proofs and additional results}
\label{app:proofs}
\subsection{On forecasting with Schr\"odinger bridges}
\label{app:sb-forecasting}
In this section, we discuss the limitations of the Schrödinger bridges methods with a fixed reference process set a priori.

\textbf{The Schrödinger bridge problem.} Schr\"odinger bridges provide a framework for modeling stochastic processes conditioned on observed marginals. Formally, they solve the following constrained optimization problem: given a reference process $p$ (typically Brownian motion), we seek a trajectory distribution $q$ that is closest to $p$ in Kullback--Leibler divergence while matching observed marginals $\{\pi_1, \dots, \pi_T\}$:
\begin{equation}
    \argmin_{q \in \mathcal{Q}} \KL(q \mid\mid p), \quad \mathcal{Q} = \{q : q_{t_i} = \pi_i, \; i = 1, \dots, T\}.
    \label{eq:sb-problem}
\end{equation}
Typically, both $p$ and $q$ are defined via SDEs sharing the same volatility, and are supported over a finite time horizon.

\textbf{Forecasting limitations of Schrödinger bridges with fixed reference process.} To study forecasting, we extend this setup to consider an unknown future marginal $\pi_{T+1}$. In this setting, we solve:
\begin{equation}
    \pi_{T+1} = \argmin_{\pi} \min_{q \in \mathcal{Q}} \KL(q \mid\mid p), \quad \mathcal{Q} = \{q : q_{t_i} = \pi_i, \; i = 1, \dots, T, \; q_{t_{T+1}} = \pi\}.
    \label{eq:sbforecastinggoal}
\end{equation}

When both $p$ and $q$ are defined via SDEs, forecasting amounts to evolving particles forward from the last known marginal $\pi_T$ under the reference process $p$. To make this precise, we adopt notation from \citet{Lavenant2021}. Let $p_{t_i, t_{i+1}}$ and $q_{t_i, t_{i+1}}$ denote the transition densities of $p$ and $q$ over $[t_i, t_{i+1}]$. Let $q_{t_i} p_{t_i, t_{i+1}}$ denote the joint trajectory distribution that begins with marginal $q_{t_i}$ and evolves forward using the dynamics of $p$ until time $t_{i+1}$. \\

\begin{proposition}[Forecasting limitation of fixed-reference SB methods]
\label{prop:sb-forecasting}
Assume that $p$ and $q$ are SDEs with the same volatility, finite time horizon and satisfying \cref{assumption-lipschitz} and \cref{assumption-bddsecondmoments}. The solution to \eqref{eq:sbforecastinggoal} is the marginal at $t_{T+1}$ of $\pi_T p_{t_T, t_{T+1}}$; that is, forecasting at $t_{T+1}$ reduces to evolving $\pi_T$ forward using the reference dynamics $p$ until $t_{T+1}$.
\end{proposition}

\begin{proof}
When $p$ and $q$ share volatility and finite horizon, the processes are Markovian. Using Proposition D.1 and Remark D.2 of \citet{Lavenant2021}, we decompose the KL objective:
\begin{align}
    \KL(q \mid\mid p) &= \KL(q_{t_1, t_2} \mid\mid p_{t_1, t_2})
    + \sum_{i=2}^{T} \KL(q_{t_i, t_{i+1}} \mid\mid q_{t_i} p_{t_i, t_{i+1}})
    + \KL(q_{t_T, t_{T+1}} \mid\mid q_{t_T} p_{t_T, t_{T+1}}).
\end{align}
The first $T$ terms are fixed by the constraints on $\{\pi_1, \dots, \pi_T\}$ and does not depends on $\pi_{T+1}$, so the final term governs the choice of $\pi_{T+1}$. This term is minimized by setting:
\begin{equation}
    q_{t_T, t_{T+1}} = \pi_T p_{t_T, t_{T+1}},
\end{equation}
which yields $q_{t_{T+1}} = \pi_{T+1}$ as the marginal of $\pi_T$ evolved forward under $p$. Thus, the optimal forecast is the marginal of $\pi_T p_{t_T, t_{T+1}}$.
\end{proof}

This result shows that in standard SB setups with fixed reference dynamics $p$, forecasting reduces to propagating the last observed marginal forward under $p$. Consequently, the quality of SB-based forecasts is limited by the choice of reference and does not adapt to information in the earlier snapshots, unless additional refinement is performed \citep[e.g.,][]{shen2024multi}.

\subsection{Proof of \cref{prop:MMDdecomposition}}
\label{app:proof-main-prop}
In this section, we prove our main proposition from the main text, \cref{prop:MMDdecomposition}.
\begin{proof}[Proof of \cref{prop:MMDdecomposition}]
We start with the definition of the MMD squared between the joint distributions:
\begin{equation}
\begin{aligned}
    \MMD^2_{K}(f(\statevar,t),g(\statevar,t))&=\E_{(\statevar,t)\sim f, (\statevar',t')\sim f}\left[K((\statevar,t), (\statevar',t'))\right]\\
    &~~-2\E_{(\statevar,t)\sim f, (\statevar',t')\sim g}\left[K((\statevar,t), (\statevar',t'))\right]\\
    &~~+\E_{(\statevar,t)\sim g, (\statevar',t')\sim g}\left[K((\statevar,t), (\statevar',t'))\right]
\end{aligned}
\end{equation}

Then we can rewrite the first term in right-hand side as follows:

\begin{equation}
\begin{aligned}
\E_{(\statevar,t)\sim f, (\statevar',t')\sim f}\left[K((\statevar,t), (\statevar',t'))\right]&=\mathbb{E}_{(\statevar,t)\sim f,\; (\statevar',t')\sim f}\!\left[K_{\statevar}(\statevar,\statevar')\,\delta(t-t')\right]\\
&=\mathbb{E}_{t\sim h(t),\; t'\sim h(t')}\!\left[\delta(t-t')\,\mathbb{E}_{\substack{\statevar\sim f(\statevar\mid t) \\ \statevar'\sim f(\statevar\mid t')}}\!\left[K_{\statevar}(\statevar,\statevar')\right]\right]\\
&=\sum_{t,t'\in \timerange}\left[\delta(t-t')\,\mathbb{E}_{\substack{\statevar\sim f(\statevar\mid t) \\ \statevar'\sim f(\statevar\mid t')}}\!\left[K_{\statevar}(\statevar,\statevar')\right]\right]h(t)h(t')\\
&=\sum_{t\in \timerange} \mathbb{E}_{\substack{\statevar, \statevar'\sim f(\statevar\mid t)}}\!\left[K_{\statevar}(\statevar,\statevar')\right]h^2(t)
\end{aligned}
\end{equation}
where the first equality uses the factorized form of the kernel, the second equality is by the the law of iterated expectation conditioning on the time components.

Similarly the second term:
\begin{equation}
\begin{aligned}
    -2\E_{(\statevar,t)\sim f, (\statevar',t')\sim g}\left[K((\statevar,t), (\statevar',t'))\right] &= \mathbb{E}_{(\statevar,t)\sim f,\; (\statevar',t')\sim g}\!\left[K_{\statevar}(\statevar,\statevar')\,\delta(t-t')\right]\\
    &=\mathbb{E}_{t\sim h(t),\; t'\sim h(t')}\!\left[\delta(t-t')\,(-2\mathbb{E}_{\substack{\statevar\sim f(\statevar\mid t) \\ \statevar'\sim g(\statevar\mid t')}}\!\left[K_{\statevar}(\statevar,\statevar')\right])\right]\\
    &=\sum_{t,t'\in \timerange}\left[\delta(t-t')\,(-2\mathbb{E}_{\substack{\statevar\sim f(\statevar\mid t) \\ \statevar'\sim g(\statevar\mid t')}}\!\left[K_{\statevar}(\statevar,\statevar')\right])\right]h(t)h(t')\\
    &=\sum_{t\in \timerange} -2\mathbb{E}_{\substack{\statevar\sim f(\statevar\mid t) \\ \statevar'\sim g(\statevar\mid t)}}\!\left[K_{\statevar}(\statevar,\statevar')\right]h^2(t)
\end{aligned}
\end{equation}

The third term
\begin{equation}
\begin{aligned}
\E_{(\statevar,t)\sim g, (\statevar',t')\sim f}\left[K((\statevar,t), (\statevar',t'))\right]&=\mathbb{E}_{(\statevar,t)\sim g,\; (\statevar',t')\sim g}\!\left[K_{\statevar}(\statevar,\statevar')\,\delta(t-t')\right]\\
&=\mathbb{E}_{t\sim h(t),\; t'\sim h(t')}\!\left[\delta(t-t')\,\mathbb{E}_{\substack{\statevar\sim g(\statevar\mid t) \\ \statevar'\sim g(\statevar\mid t')}}\!\left[K_{\statevar}(\statevar,\statevar')\right]\right]\\
&=\sum_{t,t'\in \timerange}\left[\delta(t-t')\,\mathbb{E}_{\substack{\statevar\sim g(\statevar\mid t) \\ \statevar'\sim g(\statevar\mid t')}}\!\left[K_{\statevar}(\statevar,\statevar')\right]\right]h(t)h(t')\\
&=\sum_{t\in \timerange} \mathbb{E}_{\substack{\statevar, \statevar'\sim g(\statevar\mid t)}}\!\left[K_{\statevar}(\statevar,\statevar')\right]h^2(t)
\end{aligned}
\end{equation}
Collecting terms, we have 
\begin{align*}
     &~~\MMD^2_{K}(f(\statevar,t),g(\statevar,t))\\
     &=\sum_{t\in \timerange} h^2(t)\left[ \mathbb{E}_{\substack{\statevar, \statevar'\sim f(\statevar\mid t)}}\!\left[K_{\statevar}(\statevar,\statevar')\right] -2\mathbb{E}_{\substack{\statevar\sim f(\statevar\mid t) \\ \statevar'\sim g(\statevar\mid t)}}\!\left[K_{\statevar}(\statevar,\statevar')\right]+\mathbb{E}_{\substack{\statevar, \statevar'\sim g(\statevar\mid t)}}\!\left[K_{\statevar}(\statevar,\statevar')\right]\right]\\
     &=\sum_{t\in \timerange} h^2(t)\MMD_{K_{\statevar}}^2(f(\cdot\mid t),g(\cdot\mid t))
\end{align*}

\end{proof}

In our application, we used the empirical distribution for time, i.e., $\emptimemeasure(\timestep{})$ from \cref{eq:empirical_dists}. By doing so, the two distributions of interest to apply \cref{prop:MMDdecomposition} are $\predjoint{y}{\timestep{}}$ and $\empjoint{y}{\timestep{}}$, and we can rewrite the squared MMD as
\begin{align*}
   \MMD_K^2(\predjointnoarg, \empjointnoarg)
=\sum_{\timeidx=1}^{\totsteps} \left(\frac{\tottrajec_{\timeidx}}{\sum_{j=1}^{\totsteps}\tottrajec_{j}}\right)^2\,\MMD_{K_{\statevar}}^2 ( \predconditional{\cdot}{\timestep{\timeidx}}, \empconditional{\cdot}{\timestep{\timeidx}} ).
\end{align*}
as noted in the main text.

\section{Alternative view of $R^2$}
\label{app:alternativeR2}
An alternative way to view this metric is through the lens of comparing joint distributions to the product of their marginals. In information theory, mutual information quantifies the dependence between two random variables by measuring the divergence (typically via the Kullback–Leibler divergence) between the joint distribution and the product of the marginal distributions. Analogously, by reapplying \cref{prop:MMDdecomposition}, we can interpret our $R^2$ metric as comparing the joint distribution of state and time as predicted by the model with the distribution obtained by taking the product of the marginal (state and time) distributions. In this view, the denominator in \cref{eq:R2_metric} (which uses the barycenter) reflects the total variation or “spread” in the observed data. And the numerator captures the remaining error when the model-predicted joint distribution is compared to the empirical joint distribution. Thus, a higher $R^2$ indicates that the model captures more of the dependence structure between state and time—just as in regression a higher $R^2$ means the model explains a larger fraction of the variability in the data. This analogy to mutual information provides an intuitive understanding of how our metric not only assesses goodness-of-fit but also the degree to which the model captures the temporal structure of the data.
\section{Further experimental details}
\label{app:further-exp-details}

In this section, we provide additional details and results for our experiments. First we provide more details about the vector field reconstruction task. Then, we provide the summary table for the forecasting and interpolation task for all experiments using EMD. After that, for each experiment introduced in the main text, we describe: (1) the experimental setup, (2) the choice of model family used with our method, (3) forecasting results, (4) vector field reconstruction results, and (5) interpolation results. Our experiments are carried out using four cores of Intel Xeon Gold 6248 CPU and one Nvidia Volta V100 GPU with 32 GB RAM.

\subsection{Vector Field Reconstruction}
\label{app:vector-field-reconstruction}
In cases where the true underlying drift function is known (e.g., synthetic experiments), we also evaluate how accurately methods reconstruct the vector field driving the system dynamics. We use the same baselines as in the forecasting task, since this task also requires recovering a coherent forward-time dynamic, rather than just interpolating between marginals. We measure reconstruction accuracy visually and by computing mean squared error (MSE) between the learned drift and ground truth drift on a dense grid covering the observed data range. Overall, our method provides better or matching vector field reconstruction performance relative to competitors. Detailed results for vector field reconstruction are presented in \cref{app:lv-vector} (for Lotka-Volterra), \cref{app:repr-vector-parametric} (for repressilator with parametric model family), \cref{app:repr-vector-semiparametric} (for repressilator with semiparametric model family), and \cref{app:gom-vector} (for Gulf of Mexico). 

\subsection{How we use $R^2$ in our experiments}
\label{app:use-r2-experiments}
In this section, we briefly describe how the proposed $R^2$ metric is used in our experiments. We apply it in two main ways. First, for early stopping: we select the number of training epochs such that $R^2$ increases by less than 0.01 over the last 20 epochs. Second, as a model selection criterion when choosing among neural network architectures. In particular, for real-data experiments using multilayer perceptrons, we perform a small grid search over the number of layers and hidden units, selecting the model with the highest $R^2$. Specific details for each experiment are provided in the “Model family choice” subsections in \cref{app:further-exp-details}.

\subsection{EMD table with summary results across all methods and experiments}
\label{app:emd-big-table}
In this section, we provide the summary table (\cref{tab:emd_combined}) for comparing forecasting and interpolation performance in terms of EMD over all the experiments of interest. For the interpolation results, we report the mean and standard deviation for each method, aggregated over different random seeds and interpolation held-out time points. Detailed per-time-point results are provided for each experiment in the corresponding “Interpolation results” subsubsection later in \cref{app:further-exp-details}. 

\begin{table}[!ht]
\caption{EMD for forecast and interpolation tasks.  Bold green = best method; plain green = method whose mean lies within one standard deviation of the best method. }
\centering
\makebox[0pt][c]{
\small
\begin{tabular}{lcccccc}
\toprule
      & \multicolumn{3}{c}{\textbf{Forecast}} & \multicolumn{3}{c}{\textbf{Interpolation}} \\
\cmidrule(lr){2-4}\cmidrule(lr){5-7}
\textbf{Task}
      & \textbf{Ours} & \texttt{SBIRR-ref} & \texttt{SB-forward}
      & \textbf{Ours} & \texttt{SBIRR} & \texttt{DMSB} \\ \midrule
LV
      & \best{\ms{0.21}{0.04}} & \ms{0.79}{0.05} & \ms{1.82}{0.9}
      & \best{\ms{0.06}{0.03}} & \ms{0.10}{0.08} & \ms{0.82}{0.6} \\ \midrule
ReprParam
      & \best{\ms{0.19}{0.08}} & \ms{1.55}{0.8} & \ms{1.39}{0.6}
      & \best{\ms{0.10}{0.04}} & \ms{0.17}{0.06} & \ms{1.89}{0.9} \\
ReprSemiparam
      & \best{\ms{0.35}{0.09}} & \ms{1.18}{0.4} & \ms{1.16}{0.3}
      & \best{\ms{0.21}{0.11}} & \ms{0.39}{0.1} & \ms{1.89}{0.9} \\
ReprProtein
      & \best{\ms{0.26}{0.04}} & \ms{6.36}{0.5} & \ms{7.24}{0.5}
      & \best{\ms{0.08}{0.03}} & \ms{0.59}{0.8} & \ms{1.48}{1.0} \\ \midrule
GoM
      & \best{\ms{0.71}{0.01}} & \ms{0.89}{0.03} & \ms{0.94}{0.08}
      & \ms{0.10}{0.04} & \best{\ms{0.04}{0.01}} & \ms{0.08}{0.04} \\ \midrule
PBMC
      & -- & -- & --  & -- & -- & -- \\
\bottomrule
\end{tabular}}
\label{tab:emd_combined}
\end{table}

\subsection{Full results for interpolation task}
\label{app:complete-summary-interpolation}
In this section, we provide summary tables for all methods for the interpolation task. In \cref{tab:interp-summary-complete-mmd} we provide the summary table for MMD. In \cref{tab:interp-summary-complete-emd} we provide the summary table for EMD. As in \cref{app:emd-big-table}, we report the mean and standard deviation for each method, aggregated over different random seeds and interpolation held-out time points. We find that in almost all experiments, SnapMMD provides better or matching interpolation performance relative to competitors.

\begin{table}[!hb]
\caption{Global summary of MMD across all seeds and validation points}
\centering
\makebox[0pt][c]{
\small
\begin{tabular}{lrrrrrr}
\toprule
Task & Ours & \texttt{SBIRR} & \texttt{DMSB} & \texttt{OT-CFM} & \texttt{SB-CFM} & \texttt{SF2M} \\\midrule
LV & \cellcolor{green!25}\bm{$0.017\pm{ 0.013}$} & $0.042\pm{ 0.031}$ & $1.148\pm{ 0.385}$ & $1.020\pm{ 0.358}$ & $0.856\pm{ 0.286}$ & $0.640\pm{ 0.256}$ \\
ReprParam & \cellcolor{green!25}\bm{$0.035\pm{ 0.035}$} & $0.163\pm{ 0.124}$ & $1.575\pm{ 0.395}$ & $1.230\pm{ 0.369}$ & $0.963\pm{ 0.410}$ & $0.847\pm{ 0.270}$ \\
ReprSemiparam & \cellcolor{green!25}\bm{$0.016\pm{ 0.020}$} & $0.342\pm{ 0.312}$ & $1.280\pm{ 0.426}$ & $1.030\pm{ 0.447}$ & $0.649\pm{ 0.426}$ & $0.727\pm{ 0.435}$ \\
ReprProtein & \cellcolor{green!25}\bm{$0.311\pm{ 0.379}$} & \cellcolor{green!25}$0.477\pm{ 0.304}$ & $1.575\pm{ 0.395}$ & $1.230\pm{ 0.369}$ & $0.963\pm{ 0.410}$ & $0.847\pm{ 0.270}$ \\
Gulf of Mexico & $0.288\pm{ 0.197}$ & \cellcolor{green!25}\bm{$0.072\pm{ 0.058}$} & $0.147\pm{ 0.093}$ & $1.081\pm{ 0.376}$ & $0.771\pm{ 0.338}$ & $0.788\pm{ 0.334}$ \\
PBMC & \cellcolor{green!25}\bm{$0.011\pm{ 0.008}$} & \cellcolor{green!25}$0.012\pm{ 0.013}$ & $0.561\pm{ 0.102}$ & $0.039\pm{ 0.018}$ & $0.047\pm{ 0.024}$ & $0.032\pm{ 0.010}$ \\
\bottomrule
\end{tabular}}
\label{tab:interp-summary-complete-mmd}
\end{table}

\begin{table}[!hb]
\caption{Global summary of EMD across all seeds and validation points. For the repressilator with incomplete state measurements, for some random seeds the trajectories generated by SB-CFM and SF2M diverged significantly, leading to extremely large average EMD values ($1138.222 \pm 8653.586$ for SB-CFM and $1954.050 \pm 14737.310$ for SF2M). To maintain visualization clarity, we omit these entries from the summary table.}
\centering
\makebox[0pt][c]{
\small
\begin{tabular}{lrrrrrr}
\toprule
Task & Ours & \texttt{SBIRR} & \texttt{DMSB} & \texttt{OT-CFM} & \texttt{SB-CFM} & \texttt{SF2M} \\\midrule
LV & \cellcolor{green!25}\bm{$0.064\pm{ 0.031}$} & $0.096\pm{ 0.075}$ & $0.819\pm{ 0.559}$ & $0.742\pm{ 0.698}$ & $0.669\pm{ 0.674}$ & $0.770\pm{ 0.706}$ \\
ReprParam & \cellcolor{green!25}\bm{$0.096\pm{ 0.041}$} & $0.168\pm{ 0.057}$ & $1.887\pm{ 0.865}$ & $1.367\pm{ 0.793}$ & $1.412\pm{ 1.377}$ & $23.201\pm{ 122.101}$ \\
ReprProtein & \cellcolor{green!25}\bm{$0.080\pm{ 0.032}$} & $0.591\pm{ 0.783}$ & $1.480\pm{ 1.011}$ & $1.328\pm{ 1.026}$ & -- & -- \\
ReprSemiparam & \cellcolor{green!25}\bm{$0.210\pm{ 0.106}$} & $0.394\pm{ 0.136}$ & $1.887\pm{ 0.865}$ & $1.367\pm{ 0.793}$ & $1.412\pm{ 1.377}$ & $23.201\pm{ 122.101}$ \\
Gulf of Mexico & $0.103\pm{ 0.040}$ & \cellcolor{green!25}\bm{$0.043\pm{ 0.012}$} & $0.078\pm{ 0.037}$ & $0.252\pm{ 0.109}$ & $0.265\pm{ 0.136}$ & $0.249\pm{ 0.102}$ \\
PBMC & -- & -- & -- & -- & -- & -- \\
\bottomrule
\end{tabular}}
\label{tab:interp-summary-complete-emd}
\end{table}

\clearpage

\subsection{Lotka-Volterra}
\label{app:lv-app}
\subsubsection{Experiment setup}
In this experiment, we study the stochastic Lotka-Volterra model, which describes predator-prey interactions over time. The population dynamics are governed by the following system of SDEs:
\begin{align}
\begin{split}
\label{eq:lv-family}
    dX &= \alpha X - \beta XY + \sigma X dW_x, \\
    dY &= \gamma XY - \delta Y + \sigma Y dW_y,
\end{split}
\end{align}
where $[dW_x, dW_y]$ denotes a two-dimensional Brownian motion. The true parameter values are set to $\alpha = 1.0$, $\beta = 0.4$, $\gamma = 0.4$, $\delta = 0.1$, and $\sigma = 0.02$. Initial population sizes are sampled from uniform distributions: $X_0 \sim U(5, 5.1)$ and $Y_0 \sim U(4, 4.1)$. We simulate the system over 19 discrete time points using the Euler–Maruyama method (via the \texttt{torchsde} Python package), with a time step of 0.5 and 200 samples per time point. We use the 10 odd-numbered time steps as training data. The 9 even-numbered time steps are held out for evaluating interpolation performance. To assess forecasting, we simulate one additional time step beyond the final snapshot, using a larger time increment of 1.0, and hold it out as the test point.

\subsubsection{Model family choice}
\label{app:LV_implementation}
For this experiment, we have access to the data-generating process, as described in \cref{eq:lv-family}. Therefore, we select the model family to be the set of SDEs that satisfy this system of equations, \cref{eq:lv-family} with free parameters $\alpha, \beta, \gamma, \delta, \sigma>0$. We learn the parameters by minimizing the proposed MMD loss using gradient descent, with a learning rate of 0.05 over 300 epochs. We choose the number of epochs such that in the last 20 epochs $R^2$ increases by less than 0.01. We implemented the model family as a Python class using in the \texttt{torchsde} \citep{li2020scalable} module.

\subsubsection{Forecasting results.}
\label{app:lv-results-forecasting}
In the first row of \cref{tab:lv} we show the MMD results for the forecast task. The MMD is computed using a RBF kernel with length scale 1. In each cell, the first number represent the MMD averaged across 10 different seeds, and the second number (in parenthesis) is the standard deviation over the same 10 seeds. We color in green the cell corresponding to the method with lowest MMD. We also highlight in green any other methods whose mean is contained in the one-standard deviation confidence interval for the best method. From the first row, we can see how our method is (by far) the best method at the forecasting task. In the second row, we have the same set of results, using EMD instead of MMD. We can see that our method is, also using EMD, by far the best method at the forecasting task. 

\subsubsection{Vector field reconstruction results.}
\label{app:lv-vector}
If we look at the middle row of \cref{fig:LV-appendix}, we see that from a visual perspective the reconstructed vector fields are very similar to the ground truth for all the three methods. Also from the bottom row of the same figure, we can see that the difference between the reconstructed fields and ground truth for our method and \texttt{SBIRR-ref} is very similar, whereas for \texttt{SB-forward} it is a bit worse. The same intuition is confirmed by looking at the bottom row in \cref{tab:lv} where we compare the MSEs for the vector reconstruction task, and our method achieves the lowest value. 

\begin{figure}[!ht]
    \centering
    \includegraphics[width=\linewidth]{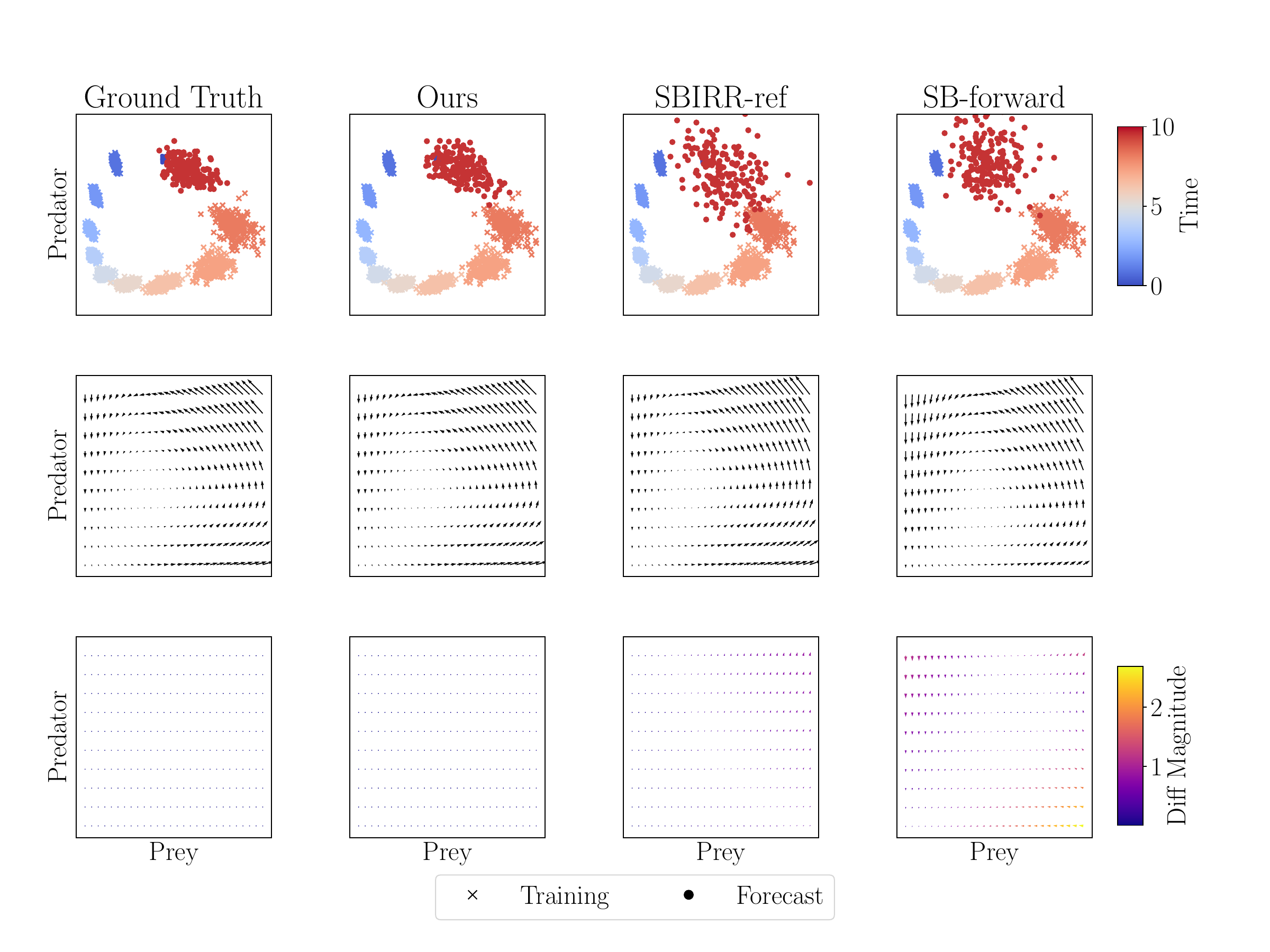}
    \caption{Experimental results for the Lotka-Volterra system. \textit{Top row}: forecast prediction task. A method is successful if the forecast predicted points (in red) match the red points in the ground truth figure. \textit{Middle row:} ground truth vector field (left) and reconstructed vector fields with the three methods. \textit{Bottom row:} Difference between reconstructed vector fields and ground truth. For each point of interest on the grid, we represent the difference between the two vectors with an arrow and color it according to the magnitude of the difference (colorbar to the right).}
    \label{fig:LV-appendix}
\end{figure}

\begin{table}[ht!]
    \centering
    \caption{Evaluation metric for Lotka-Volterra (mean (sd)). Drift was evaluated using MSE on a grid.}
    \begin{tabular}{lccc}
    \hline
           &          &   LV  &   \\
    Metric     &  Ours & \texttt{SBIRR-ref} & \texttt{SB-forward} \\
    \hline
    Forecast-MMD & \cellcolor{green!25}\textbf{0.012 (0.0067)} & 0.14 (0.023)& 0.71 (0.49)  \\
    Forecast-EMD & \cellcolor{green!25}\textbf{0.21 (0.035)} & 0.79 (0.053)& 1.82 (0.94)  \\
    Drift & \cellcolor{green!25}\textbf{0.00071 (0.000027)} & 0.079 (0.0080)& 0.59 (0.13)   \\
    \hline
    \end{tabular}
    \label{tab:lv}
\end{table}

\subsubsection{Interpolation results}
\label{app:lv-interpol}
We evaluate model performance on the interpolation task in the classic LV experiment by comparing both qualitative and quantitative results. Specifically, we assess the quality of inferred trajectories against held-out validation snapshots using MMD and EMD. In \cref{fig:lv-appendix-parametric-interpol}, the held-out validation snapshots are indicated by x-markers. An interpolation method is successful when its learned trajectories intersect these markers. For visual clarity, we omit the training snapshots: they fall between consecutive validation times and would excessively clutter the figure without adding interpretive value. As shown in \cref{fig:lv-appendix-parametric-interpol}, our method produces trajectories that interpolate closely the held-out data. Among all baselines, \texttt{SBIRR} and the two \texttt{CFM} achieve comparable visual match. This visual impression is also supported by the quantitative metrics. In \cref{fig:lv-metrics-interpol}, we plot MMD and EMD values over all validation time points. Our method consistently achieves the lowest values across time in both metrics. \texttt{SBIRR} performs comparably well at some validation points. In contrast, the other baselines show significantly higher errors, particularly in later time steps where the trajectory distribution becomes more complex. The corresponding tables (\cref{tab:lv-mmd} and \cref{tab:lv-emd}) confirm these trends. For every validation point, our method is always as good as the best baseline (\texttt{SBIRR}) and in most of the cases it achieves the lowest MMD and EMD values.

\begin{figure}[!ht]
    \centering
    \includegraphics[width=\linewidth]{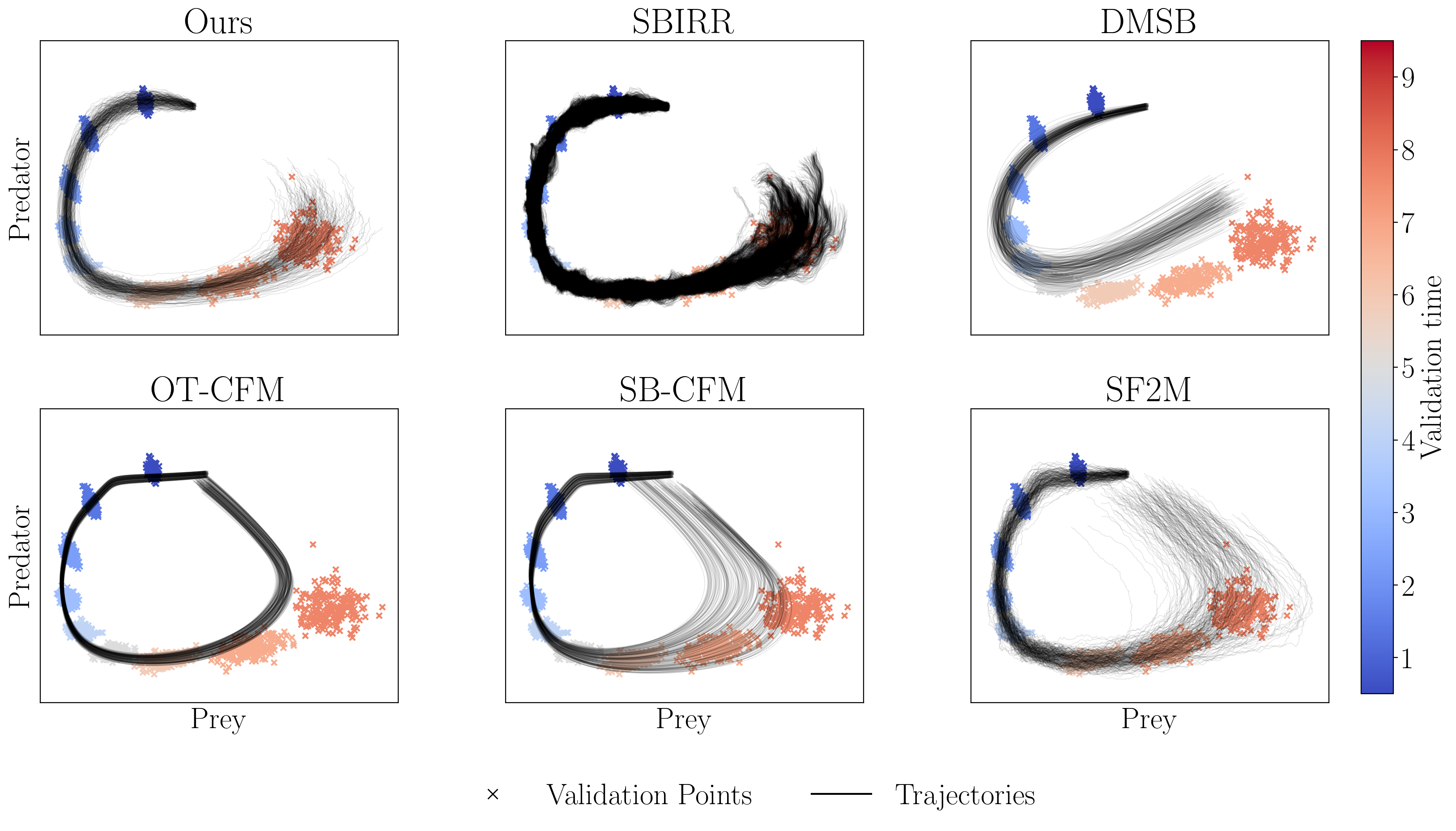}
    \caption{LV interpolation}
    \label{fig:lv-appendix-parametric-interpol}
\end{figure}

\begin{figure}[!ht]
    \centering
    \includegraphics[width=\linewidth]{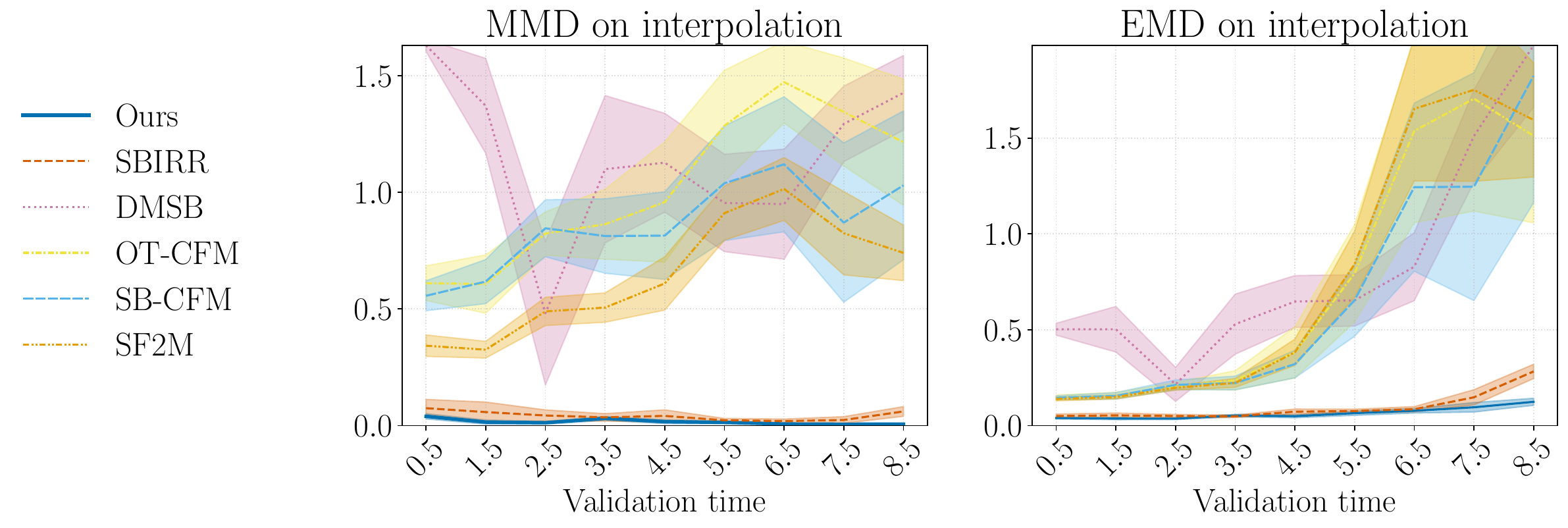}
    \caption{LV interpolation metrics}
    \label{fig:lv-metrics-interpol}
\end{figure}

\begin{table}[ht]
\caption{MMD at each validation point for Lotka-Volterra.}
\centering
\makebox[0pt][c]{
\small
\begin{tabular}{lrrrrrr}
\toprule
Time & Ours & \texttt{SBIRR} & \texttt{DMSB} & \texttt{OT-CFM} & \texttt{SB-CFM} & \texttt{SF2M} \\\midrule
0.5 & \cellcolor{green!25}\bm{$0.040\pm{ 0.011}$} & $0.075\pm{ 0.038}$ & $1.629\pm{ 0.026}$ & $0.610\pm{ 0.074}$ & $0.556\pm{ 0.064}$ & $0.343\pm{ 0.046}$ \\
1.5 & \cellcolor{green!25}\bm{$0.015\pm{ 0.009}$} & $0.058\pm{ 0.043}$ & $1.370\pm{ 0.204}$ & $0.607\pm{ 0.125}$ & $0.617\pm{ 0.095}$ & $0.325\pm{ 0.036}$ \\
2.5 & \cellcolor{green!25}\bm{$0.012\pm{ 0.006}$} & $0.043\pm{ 0.024}$ & $0.482\pm{ 0.307}$ & $0.823\pm{ 0.093}$ & $0.846\pm{ 0.122}$ & $0.489\pm{ 0.061}$ \\
3.5 & \cellcolor{green!25}\bm{$0.031\pm{ 0.009}$} & \cellcolor{green!25}$0.035\pm{ 0.016}$ & $1.099\pm{ 0.316}$ & $0.863\pm{ 0.150}$ & $0.813\pm{ 0.160}$ & $0.506\pm{ 0.063}$ \\
4.5 & \cellcolor{green!25}\bm{$0.017\pm{ 0.008}$} & $0.042\pm{ 0.026}$ & $1.127\pm{ 0.212}$ & $0.958\pm{ 0.257}$ & $0.814\pm{ 0.188}$ & $0.609\pm{ 0.114}$ \\
5.5 & \cellcolor{green!25}\bm{$0.014\pm{ 0.006}$} & $0.023\pm{ 0.008}$ & $0.954\pm{ 0.209}$ & $1.285\pm{ 0.238}$ & $1.039\pm{ 0.246}$ & $0.911\pm{ 0.116}$ \\
6.5 & \cellcolor{green!25}\bm{$0.007\pm{ 0.004}$} & $0.020\pm{ 0.009}$ & $0.949\pm{ 0.236}$ & $1.472\pm{ 0.175}$ & $1.120\pm{ 0.289}$ & $1.014\pm{ 0.135}$ \\
7.5 & \cellcolor{green!25}\bm{$0.006\pm{ 0.004}$} & $0.024\pm{ 0.014}$ & $1.293\pm{ 0.162}$ & $1.344\pm{ 0.231}$ & $0.870\pm{ 0.341}$ & $0.825\pm{ 0.179}$ \\
8.5 & \cellcolor{green!25}\bm{$0.007\pm{ 0.004}$} & $0.061\pm{ 0.021}$ & $1.426\pm{ 0.160}$ & $1.215\pm{ 0.270}$ & $1.030\pm{ 0.319}$ & $0.740\pm{ 0.119}$ \\
\bottomrule
\end{tabular}}
\label{tab:lv-mmd}
\end{table}

\begin{table}[ht]
\caption{EMD at each validation point for Lotka-Volterra.}
\centering
\makebox[0pt][c]{
\small
\begin{tabular}{lrrrrrr}
\toprule
Time & Ours & \texttt{SBIRR} & \texttt{DMSB} & \texttt{OT-CFM} & \texttt{SB-CFM} & \texttt{SF2M} \\\midrule
0.5 & \cellcolor{green!25}\bm{$0.040\pm{ 0.005}$} & $0.050\pm{ 0.011}$ & $0.503\pm{ 0.031}$ & $0.147\pm{ 0.014}$ & $0.145\pm{ 0.010}$ & $0.139\pm{ 0.010}$ \\
1.5 & \cellcolor{green!25}\bm{$0.037\pm{ 0.006}$} & $0.053\pm{ 0.012}$ & $0.503\pm{ 0.119}$ & $0.157\pm{ 0.017}$ & $0.155\pm{ 0.017}$ & $0.149\pm{ 0.010}$ \\
2.5 & \cellcolor{green!25}\bm{$0.036\pm{ 0.005}$} & $0.050\pm{ 0.009}$ & $0.215\pm{ 0.088}$ & $0.207\pm{ 0.018}$ & $0.213\pm{ 0.026}$ & $0.198\pm{ 0.012}$ \\
3.5 & \cellcolor{green!25}$0.052\pm{ 0.006}$ & \cellcolor{green!25}\bm{$0.049\pm{ 0.007}$} & $0.529\pm{ 0.157}$ & $0.236\pm{ 0.050}$ & $0.222\pm{ 0.036}$ & $0.223\pm{ 0.023}$ \\
4.5 & \cellcolor{green!25}\bm{$0.049\pm{ 0.008}$} & $0.073\pm{ 0.014}$ & $0.647\pm{ 0.135}$ & $0.382\pm{ 0.133}$ & $0.321\pm{ 0.072}$ & $0.383\pm{ 0.069}$ \\
5.5 & \cellcolor{green!25}\bm{$0.065\pm{ 0.012}$} & \cellcolor{green!25}$0.076\pm{ 0.010}$ & $0.653\pm{ 0.134}$ & $0.795\pm{ 0.254}$ & $0.654\pm{ 0.188}$ & $0.841\pm{ 0.173}$ \\
6.5 & \cellcolor{green!25}\bm{$0.077\pm{ 0.012}$} & \cellcolor{green!25}$0.086\pm{ 0.014}$ & $0.829\pm{ 0.177}$ & $1.537\pm{ 0.479}$ & $1.244\pm{ 0.439}$ & $1.652\pm{ 0.376}$ \\
7.5 & \cellcolor{green!25}\bm{$0.096\pm{ 0.025}$} & $0.148\pm{ 0.040}$ & $1.510\pm{ 0.241}$ & $1.705\pm{ 0.586}$ & $1.246\pm{ 0.594}$ & $1.751\pm{ 0.477}$ \\
8.5 & \cellcolor{green!25}\bm{$0.124\pm{ 0.019}$} & $0.283\pm{ 0.037}$ & $1.983\pm{ 0.323}$ & $1.511\pm{ 0.454}$ & $1.825\pm{ 0.662}$ & $1.594\pm{ 0.298}$ \\
\bottomrule
\end{tabular}}
\label{tab:lv-emd}
\end{table}

\clearpage

\subsection{mRNA-only repressilator with parametric family}
\label{app:repr-app-parametric}

\subsubsection{Experiment setup}
\label{app:repr-setup-parametric}
The repressilator is a synthetic genetic circuit designed to function as a biological oscillator, producing sustained periodic fluctuations in the concentrations of its components. It consists of a network of three genes arranged in a cyclic inhibitory loop: each gene encodes a protein that suppresses the expression of the next, with the last gene repressing the first, completing the feedback cycle.

The system’s dynamics can be described by the following stochastic differential equations (SDEs):
\begin{align}
\label{eq:repr-family}
    dX_1 &= \frac{\beta}{1+(X_3/k)^n} - \gamma X_1 + \sigma X_1 dW_1, \nonumber \\
    dX_2 &= \frac{\beta}{1+(X_1/k)^n} - \gamma X_2 + \sigma X_2 dW_2, \\
    dX_3 &= \frac{\beta}{1+(X_2/k)^n} - \gamma X_3 + \sigma X_3 dW_3, \nonumber
\end{align}
where $[dW_1, dW_2, dW_3]$ represents a three-dimensional Brownian motion. The inhibitory structure of the system is evident from the drift terms, which describe how each gene’s expression is repressed by another in the cycle. For our simulations, we set the parameters to $\beta = 10$, $n = 3$, $k = 1$, $\gamma = 1$, and $\sigma = 0.02$. The initial conditions are sampled from uniform distributions: $X_1, X_2 \sim U(1,1.1)$ and $X_3 \sim U(2,2.1)$. To simulate the system, we numerically integrate the SDEs over 19 discrete time points, with sampling rate 0.5 with the Euler-Maruyama scheme(implemented via the \texttt{torchsde} Python package) with 200 samples at each step. Out of these 19 time steps, we use the 10 odd-numbered time steps as training data. The 9 even-numbered time steps are held out for evaluating interpolation performance. To assess forecasting, we simulate one additional time step beyond the final snapshot, using a larger time increment of 1.0, and hold it out as the test point.

\subsubsection{Model family choice}
\label{app:repr_implementation-parametric}
For this experiment, we have access to the data-generating process, as described in \cref{eq:repr-family}. Therefore, we pick as a model family the set of SDEs that satisfy this system of equations, \cref{eq:repr-family}. The learning process involves optimizing the parameters using gradient descent, with a learning rate of 0.05 over 500 epochs. We choose this number of epochs such that in the last 20 epochs $R^2$ increases by less than 0.01. We implemented the model family as a Python class using in the \texttt{torchsde} \citep{li2020scalable} module.

\subsubsection{Forecasting results.}
\label{app:repr-forecasting-parametric}
In this section we further discuss results for the repressilator experiment with parametric model family. In particular, we analyze the EMD and MMD in \cref{tab:repr-param}.

\begin{figure}[t!]
    \centering
    \includegraphics[width=\linewidth]{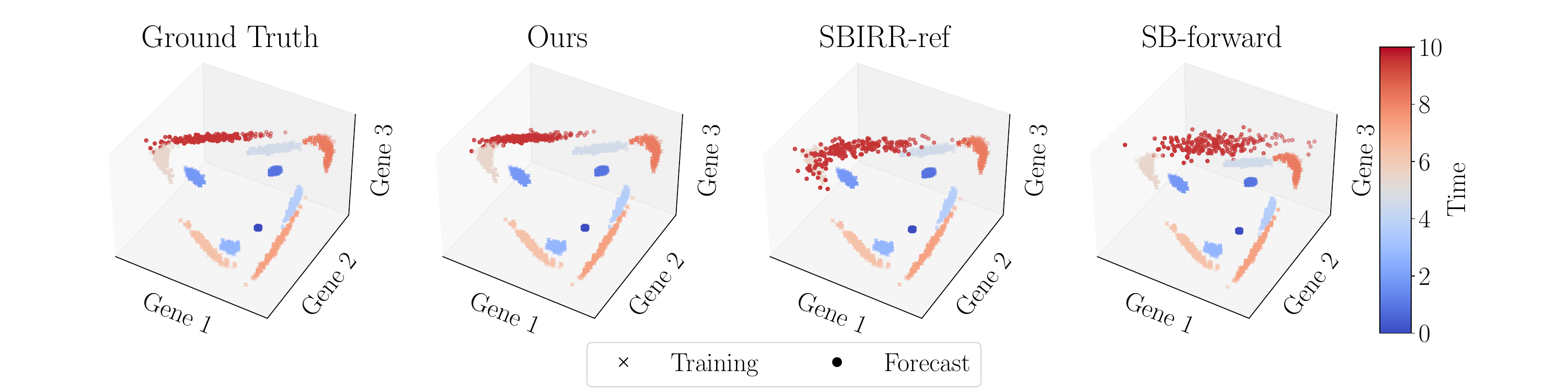}
    \caption{Forecasting task for the repressilator system with parametric model family. }
    \label{fig:repr-parametric-forecast}
\end{figure}

\begin{figure}[t!]
    \centering
    \includegraphics[width=\linewidth]{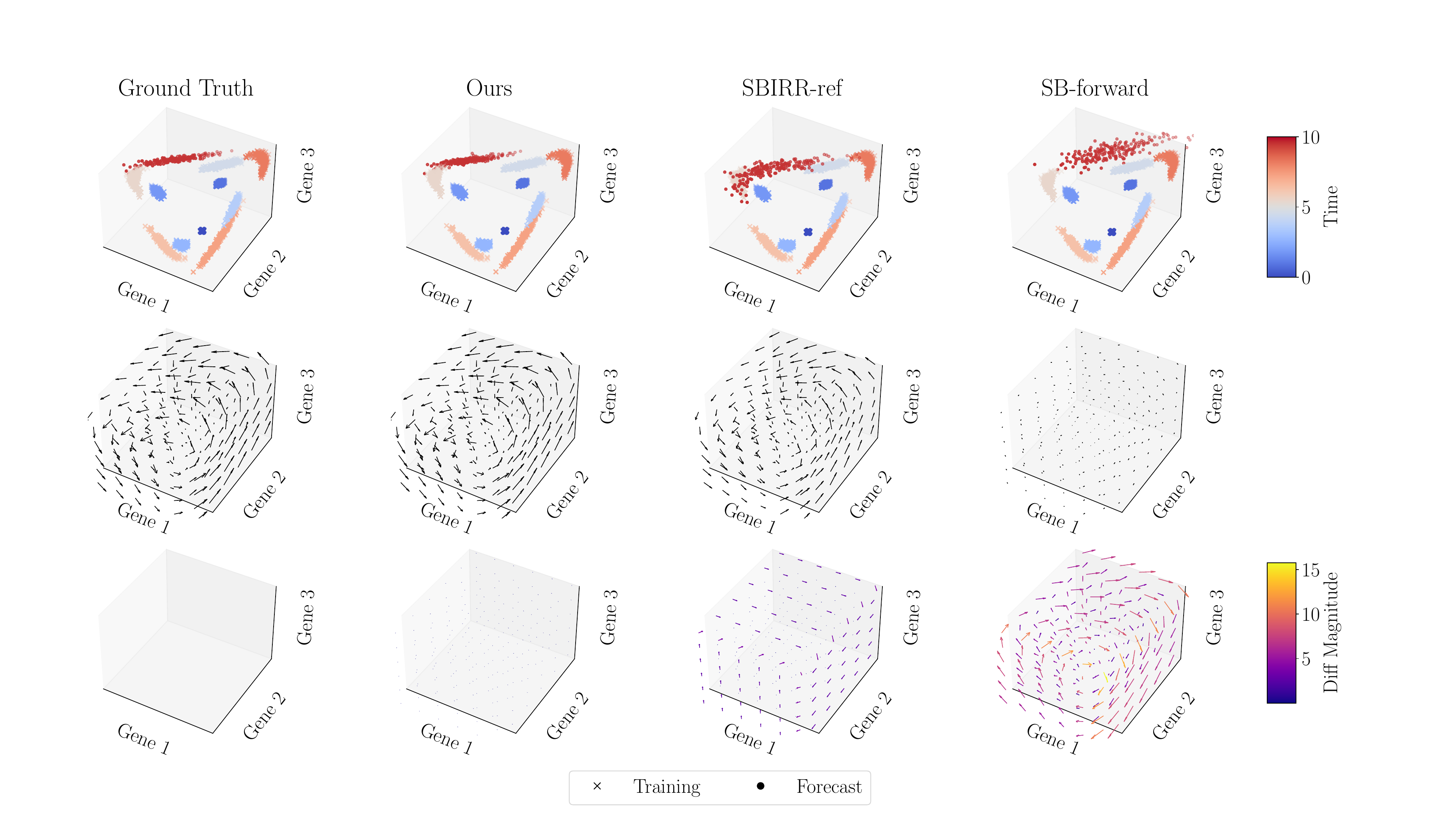}
    \caption{Experimental results for the repressilator system using parametric model as model family.\textit{Top row}: forecast prediction task. A method is successful if the forecast predicted points (in red) match the red points in the ground truth figure. \textit{Middle row:} ground truth vector field (left) and reconstructed vector fields with the three methods. \textit{Bottom row:} Difference between reconstructed vector fields and ground truth. For each point of interest on the grid, we represent the difference between the two vectors with an arrow and color it according to the magnitude of the difference (colorbar to the right).}
    \label{fig:repr-appendix-parametric}
\end{figure}

\begin{table}[ht!]
    \centering
    \caption{Evaluation metric for repressilator when using the parametric model (mean(sd)). Drift was evaluated using MSE on a grid.}
    \begin{tabular}{lccc}
    \hline
    &     & Repressilator (parametric) & \\
    Metric     &  Ours & \texttt{SBIRR-ref} & \texttt{SB-forward} \\
    \hline
    Forecast-MMD &  \cellcolor{green!25}\textbf{0.016 (0.016)} &0.47 (0.34)& 0.42 (0.16) \\
    Forecast-EMD &  \cellcolor{green!25}\textbf{0.19 (0.075)} &1.55 (0.79)& 1.39 (0.62) \\
    Drift & \cellcolor{green!25}\textbf{0.027 (0.063)} & 1.71 (0.20)& 12.9 (0.21) \\
    \hline
    \end{tabular}
    
    \label{tab:repr-param}
\end{table}

In the first row of \cref{tab:repr-param}, we see that for the forecasting task, our method achieves a much lower MMD compared to the two baselines. This quantitatively supports the visual intuition from \cref{fig:repr-parametric-forecast}, where our approach more accurately captures the underlying distribution of the data. In the second row, we see that also using EMD our method significantly outperforms all the baselines. 

\subsubsection{Vector field reconstruction results.}
\label{app:repr-vector-parametric}

In the third row of \cref{tab:repr-param}, we observe that the MSE for the vector field reconstruction task is significantly lower for our method, indicating superior performance in recovering the true dynamics. This is further corroborated by the visualizations in \cref{fig:repr-appendix-parametric}: in the middle row, our reconstructed vector field closely resembles the ground truth, whereas \texttt{SBIRR-ref} exhibits small but notable deviations, and \texttt{SB-forward} fails both in magnitude and direction. The bottom row further reinforces this conclusion, showing that the magnitude of the differences between the reconstructed and true vector fields is substantially larger for the two baselines compared to our method (for which is very close to 0 everywhere on the grid).

\subsubsection{Interpolation results}
\label{app:repr-interpol-parametric}

We next assess interpolation performance for the parametric model, again comparing inferred trajectories to held-out snapshots with MMD and EMD. Visual inspection of \cref{fig:repr-appendix-parametric-interpol} shows that our method interpolates all the validation snapshots very closely, understanding the periodic behavior of this system. \texttt{SBIRR} yields a qualitatively similar plot with more noisy trajectories, while \texttt{DMSB}, \texttt{OT-CFM}, \texttt{SB-CFM}, and \texttt{SF2M} provide poor interpolations --- \texttt{OT-CFM} fails because they are just pairwise interpolation methods and so they “connect” training points, missing the long-term behavior of the system; \texttt{DMSB}, \texttt{SB-CFM} , and \texttt{SF2M} fail because they start drifting outward and end up very far from the actual validation points.  

The metric curves in \cref{fig:repr-metric-interpol} and tables \cref{tab:repr-classic-mmd}–\cref{tab:repr-classic-emd} corroborate these impressions. Across all validation times, our method always achieves the lowest MMD and EMD, whereas the second best method (\texttt{SBIRR} ) is comparable to ours only on one validation time. All the other baselines achieve worse performance. We note that in  \cref{fig:repr-metric-interpol} we set the y-axis limit to 5 to show meaningful comparisons. We did this since \texttt{SF2M} ’s EMD explodes after time 5.5 --- peaking at $\approx 138$ at time 8.5 --- signalling complete geometric mismatch with the target distribution. We also refrain from highlighting the \texttt{SF2M}  cell in green, despite it formally satisfying the coloring criterion. This is because the overlap with the best-performing method arises primarily from \texttt{SF2M} 's extremely large mean and standard deviation, rather than from a meaningful proximity in performance.

\begin{figure}[!ht]
    \centering
    \includegraphics[width=\linewidth]{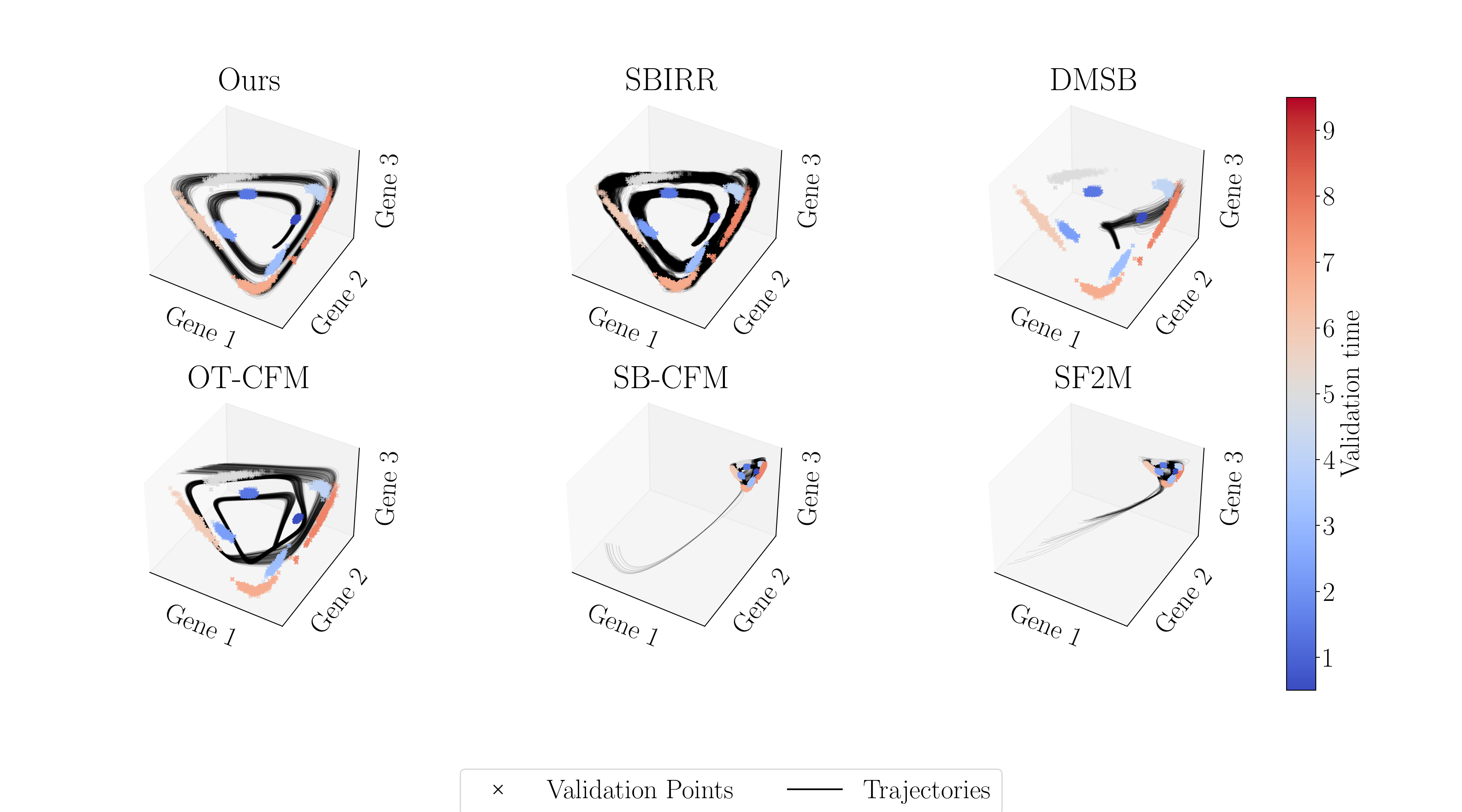}
    
    \caption{Parametric interpolation}
    \label{fig:repr-appendix-parametric-interpol}
\end{figure}

\begin{figure}[!ht]
    \centering
    \includegraphics[width=\linewidth]{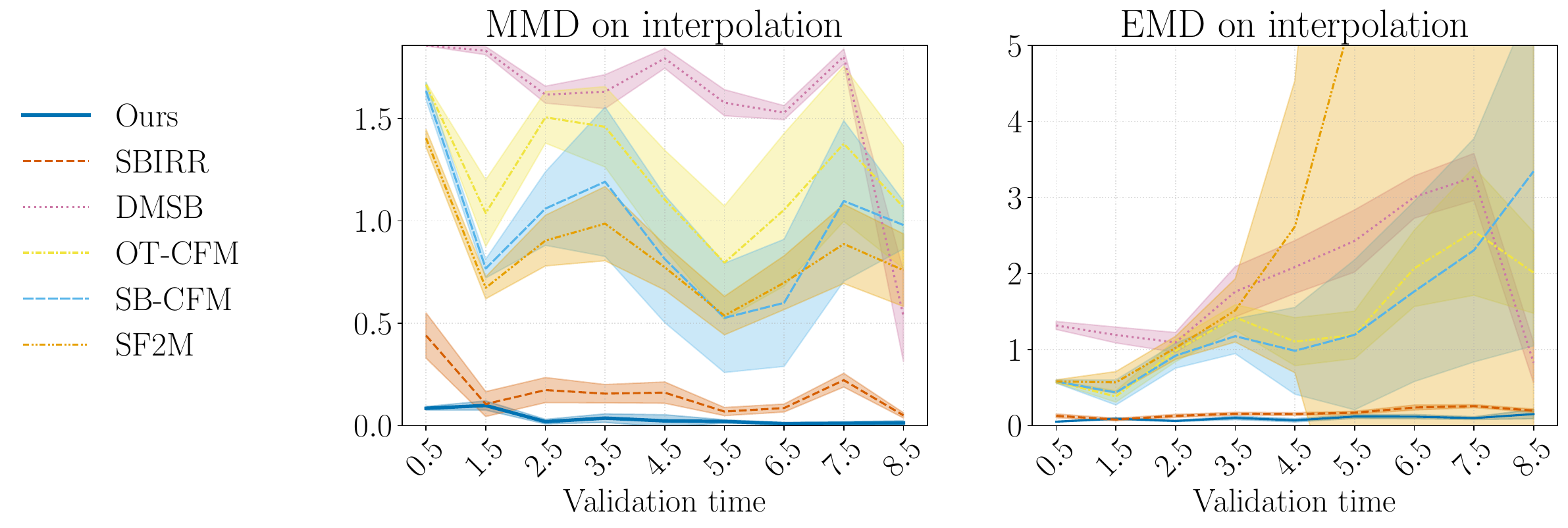}
    \caption{Parametric interpolation metric}
    \label{fig:repr-metric-interpol}
\end{figure}

\begin{table}[ht]
\caption{MMD at each validation point  for Repressilator parametric.}
\centering
\makebox[0pt][c]{
\small
\begin{tabular}{lrrrrrr}
\toprule
Time & Ours & \texttt{SBIRR} & \texttt{DMSB} & \texttt{OT-CFM} & \texttt{SB-CFM} & \texttt{SF2M} \\\midrule
0.5 & \cellcolor{green!25}\bm{$0.085\pm{ 0.009}$} & $0.440\pm{ 0.110}$ & $1.857\pm{ 0.002}$ & $1.666\pm{ 0.010}$ & $1.635\pm{ 0.043}$ & $1.404\pm{ 0.043}$ \\
1.5 & \cellcolor{green!25}\bm{$0.098\pm{ 0.023}$} & \cellcolor{green!25}$0.105\pm{ 0.061}$ & $1.832\pm{ 0.020}$ & $1.038\pm{ 0.163}$ & $0.767\pm{ 0.046}$ & $0.673\pm{ 0.054}$ \\
2.5 & \cellcolor{green!25}\bm{$0.019\pm{ 0.011}$} & $0.174\pm{ 0.061}$ & $1.617\pm{ 0.042}$ & $1.506\pm{ 0.126}$ & $1.060\pm{ 0.179}$ & $0.904\pm{ 0.124}$ \\
3.5 & \cellcolor{green!25}\bm{$0.037\pm{ 0.021}$} & $0.156\pm{ 0.045}$ & $1.631\pm{ 0.083}$ & $1.460\pm{ 0.196}$ & $1.191\pm{ 0.365}$ & $0.987\pm{ 0.181}$ \\
4.5 & \cellcolor{green!25}\bm{$0.023\pm{ 0.030}$} & $0.161\pm{ 0.052}$ & $1.794\pm{ 0.048}$ & $1.104\pm{ 0.242}$ & $0.814\pm{ 0.311}$ & $0.774\pm{ 0.111}$ \\
5.5 & \cellcolor{green!25}\bm{$0.020\pm{ 0.008}$} & $0.069\pm{ 0.020}$ & $1.578\pm{ 0.064}$ & $0.796\pm{ 0.277}$ & $0.526\pm{ 0.267}$ & $0.537\pm{ 0.094}$ \\
6.5 & \cellcolor{green!25}\bm{$0.010\pm{ 0.006}$} & $0.086\pm{ 0.020}$ & $1.529\pm{ 0.034}$ & $1.054\pm{ 0.373}$ & $0.600\pm{ 0.311}$ & $0.698\pm{ 0.131}$ \\
7.5 & \cellcolor{green!25}\bm{$0.011\pm{ 0.008}$} & $0.222\pm{ 0.034}$ & $1.803\pm{ 0.037}$ & $1.377\pm{ 0.379}$ & $1.098\pm{ 0.392}$ & $0.888\pm{ 0.193}$ \\
8.5 & \cellcolor{green!25}\bm{$0.013\pm{ 0.011}$} & $0.050\pm{ 0.013}$ & $0.535\pm{ 0.222}$ & $1.067\pm{ 0.300}$ & $0.980\pm{ 0.117}$ & $0.760\pm{ 0.177}$ \\
\bottomrule
\end{tabular}}
\label{tab:repr-classic-mmd}
\end{table}

\begin{table}[ht]
\caption{EMD at each validation point for Repressilator parametric.}
\centering
\makebox[0pt][c]{
\small
\begin{tabular}{lrrrrrr}
\toprule
Time & Ours & \texttt{SBIRR} & \texttt{DMSB} & \texttt{OT-CFM} & \texttt{SB-CFM} & \texttt{SF2M} \\\midrule
0.5 & \cellcolor{green!25}\bm{$0.052\pm{ 0.004}$} & $0.129\pm{ 0.025}$ & $1.317\pm{ 0.053}$ & $0.578\pm{ 0.017}$ & $0.581\pm{ 0.022}$ & $0.581\pm{ 0.021}$ \\
1.5 & \cellcolor{green!25}$0.091\pm{ 0.013}$ & \cellcolor{green!25}\bm{$0.080\pm{ 0.017}$} & $1.192\pm{ 0.104}$ & $0.384\pm{ 0.075}$ & $0.436\pm{ 0.165}$ & $0.569\pm{ 0.143}$ \\
2.5 & \cellcolor{green!25}\bm{$0.061\pm{ 0.009}$} & $0.128\pm{ 0.021}$ & $1.095\pm{ 0.131}$ & $0.985\pm{ 0.145}$ & $0.920\pm{ 0.164}$ & $1.023\pm{ 0.152}$ \\
3.5 & \cellcolor{green!25}\bm{$0.102\pm{ 0.023}$} & $0.157\pm{ 0.019}$ & $1.762\pm{ 0.330}$ & $1.424\pm{ 0.166}$ & $1.176\pm{ 0.229}$ & $1.515\pm{ 0.416}$ \\
4.5 & \cellcolor{green!25}\bm{$0.070\pm{ 0.021}$} & $0.153\pm{ 0.018}$ & $2.086\pm{ 0.348}$ & $1.104\pm{ 0.316}$ & $0.984\pm{ 0.571}$ & $2.618\pm{ 1.923}$ \\
5.5 & \cellcolor{green!25}\bm{$0.120\pm{ 0.025}$} & $0.170\pm{ 0.020}$ & $2.428\pm{ 0.409}$ & $1.193\pm{ 0.312}$ & $1.194\pm{ 0.986}$ & $5.505\pm{ 7.561}$ \\
6.5 & \cellcolor{green!25}\bm{$0.118\pm{ 0.028}$} & $0.238\pm{ 0.032}$ & $3.006\pm{ 0.279}$ & $2.068\pm{ 0.502}$ & $1.765\pm{ 1.185}$ & $14.974\pm{ 27.778}$ \\
7.5 & \cellcolor{green!25}\bm{$0.099\pm{ 0.018}$} & $0.257\pm{ 0.020}$ & $3.270\pm{ 0.310}$ & $2.556\pm{ 0.842}$ & $2.306\pm{ 1.469}$ & $44.051\pm{ 96.243}$ \\
8.5 & \cellcolor{green!25}\bm{$0.152\pm{ 0.061}$} & \cellcolor{green!25}$0.198\pm{ 0.021}$ & $0.828\pm{ 0.266}$ & $2.011\pm{ 0.534}$ & $3.345\pm{ 2.295}$ & $137.973\pm{ 328.155}$ \\
\bottomrule
\end{tabular}}
\label{tab:repr-classic-emd}
\end{table}

\clearpage

\subsection{mRNA-only repressilator with semiparametric family}
\label{app:repr-app-semiparametric}

\subsubsection{Experiment setup}
\label{app:repr-setup-semiparametric}
The experimental setup is the same as the one for the repressilator with the parametric model choice, as discussed in \cref{app:repr-setup-parametric}.

\subsubsection{Model family choice}
\label{app:repr_implementation-semiparametric}
In this experiment, we do not assume that we know the full functional form as in \cref{eq:repr-family}, but only up to an unknown activation function $f_{\bm \theta} : \R_{+}^{3}\to [0,1]^3$, that encodes the regulation among the three genes. In particular, we consider the following model:
\begin{equation}
    d\responseat{\timet}=\bm{M}f_{\bm \theta}(\responseat{\timet})-\bm{L} \responseat{\timet}+\bm{G} \diag(\responseat{\timet})d\bm{W}_{\timet}
    \label{eq:mlpactive}
\end{equation}
where $\bm{M}$ is a diagonal matrix of (positive) maximum production rate, $\bm{L}$ is a diagonal matrix of (positive) degradation rate, $\bm{G}$ is a diagonal matrix of (positive) volatilities, all unknown (parameterized by their logarithm). We also parameterize the activation function using an MLP with three hidden layers of [32, 64, 32] hidden neurons each, ReLU activation, and one final sigmoid layer. 

\subsubsection{Forecasting results.}
\label{app:repr-forecasting-semiparametric}
For what concerns the experiment with the semiparametric model family, we can see in the first row of \cref{tab:repr-mlp} that also with this model our method achieves a substantially lower MMD compared to the two baselines. This aligns with the visual evidence from \cref{fig:repr-main} in the main text (top row), where our method’s predicted points (in red) more closely match the ground truth. In the second row of \cref{tab:repr-mlp}, we see that EMD results are aligned with the MMD ones. 

\subsubsection{Vector field reconstruction results.}
\label{app:repr-vector-semiparametric}

In the third row of \cref{tab:repr-mlp}, we see that for this model choice our method and \texttt{SBIRR-ref} achieve similar results, whereas \texttt{SB-forward} exhibits much higher MSE. \Cref{fig:repr-appendix-semiparametric} confirms this intuition: our reconstructed vector field and the one for \texttt{SBIRR-ref} are quite similar and not too different from the ground truth, whereas \texttt{SB-forward} performs particularly poorly, failing to recover both the direction and magnitude of the vector field. 

\begin{table}[ht!]
    \centering
    \caption{Evaluation metric for Repressilator using MLP activation function (mean(sd)). Drift was evaluated using MSE on a grid, while forecast was evaluated using MMD with RBF kernel and length scale 1 as well as EMD.}
    \begin{tabular}{lccc}
    \hline
    &     & Repressilator (semiparametric) & \\
    Metric     &  Ours & \texttt{SBIRR-ref} & \texttt{SB-forward} \\
    \hline
    Forecast-MMD & \cellcolor{green!25}\textbf{0.077 (0.031)}& 0.29 (0.11) & 1.15 (0.33) \\
    Forecast-EMD & \cellcolor{green!25}\textbf{0.35 (0.091)}& 1.18 (0.44) & 1.16 (0.33) \\
    Drift & \cellcolor{green!25}\textbf{6.25 (0.37)} & \cellcolor{green!25}\textbf{7.85 (1.85)} & 12.00 (0.74)  \\
    \hline
    \end{tabular}
    \label{tab:repr-mlp}
\end{table}

\begin{figure}[h!]
    \centering
    \includegraphics[width=\linewidth]{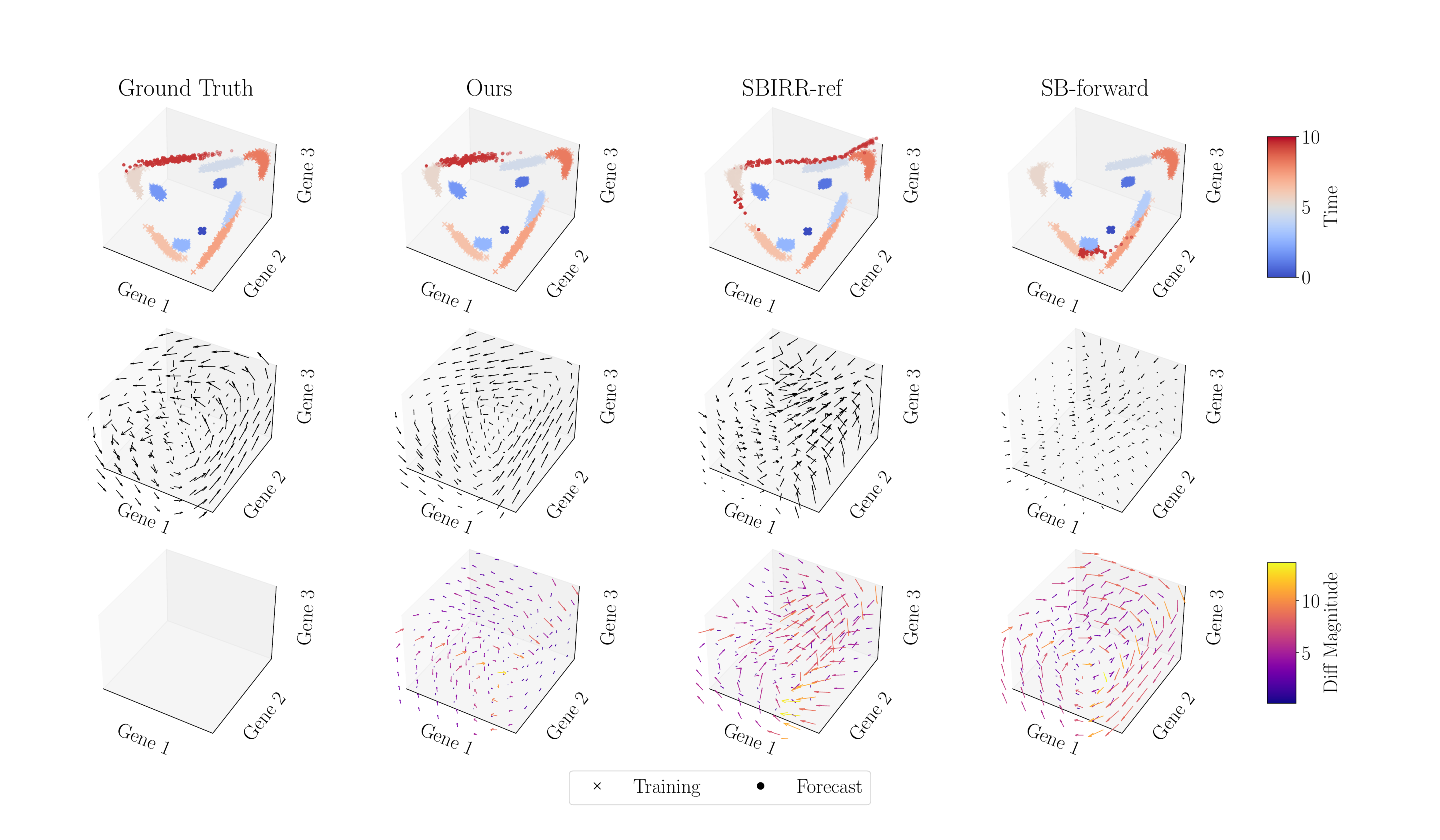}
    \caption{Experimental results for the repressilator system using semiparametric model as model family. \textit{Top row}: forecast prediction task. A method is successful if the forecast predicted points (in red) match the red points in the ground truth figure. \textit{Middle row:} ground truth vector field (left) and reconstructed vector fields with the three methods. \textit{Bottom row:} Difference between reconstructed vector fields and ground truth. For each point of interest on the grid, we represent the difference between the two vectors with an arrow and color it according to the magnitude of the difference (colorbar to the right).}
    \label{fig:repr-appendix-semiparametric}
\end{figure}

\subsubsection{Interpolation results}
\label{app:repr-interpol-semiparametric}

We now evaluate interpolation performance in the more realistic semiparametric setting, where neither our method nor \texttt{SBIRR} have access to the true data-generating process. Instead, both methods rely on the same semiparametric reference family from \cref{eq:mlpactive}, introducing a meaningful model mismatch that more closely reflects real-world conditions. As shown in \cref{fig:repr-appendix-semiparametric-interpol}, interpolation quality for our method and \texttt{SBIRR} is very similar to the parametric case, as they both are still very good at interpolating all the validation snapshots. The remaining baselines are not affected by this modeling choice, so the trajectories are exactly as in \cref{fig:repr-appendix-parametric-interpol}.
The quantitative results in \cref{fig:repr-semiparametric-metric-interpol} and tables \cref{tab:repr-mlp-mmd}–\cref{tab:repr-mlp-emd} reinforce these trends. Although all methods exhibit increased error relative to the parametric setting, our method continues to outperform all baselines across nearly all validation times. In terms of MMD, we are the best method at six of nine time points; \texttt{SBIRR} performs similarly at three time points, but is consistently worse on the rest. EMD results are even more decisive: our method achieves the lowest error at every time point.

\begin{figure}
    \centering
    \includegraphics[width=\linewidth]{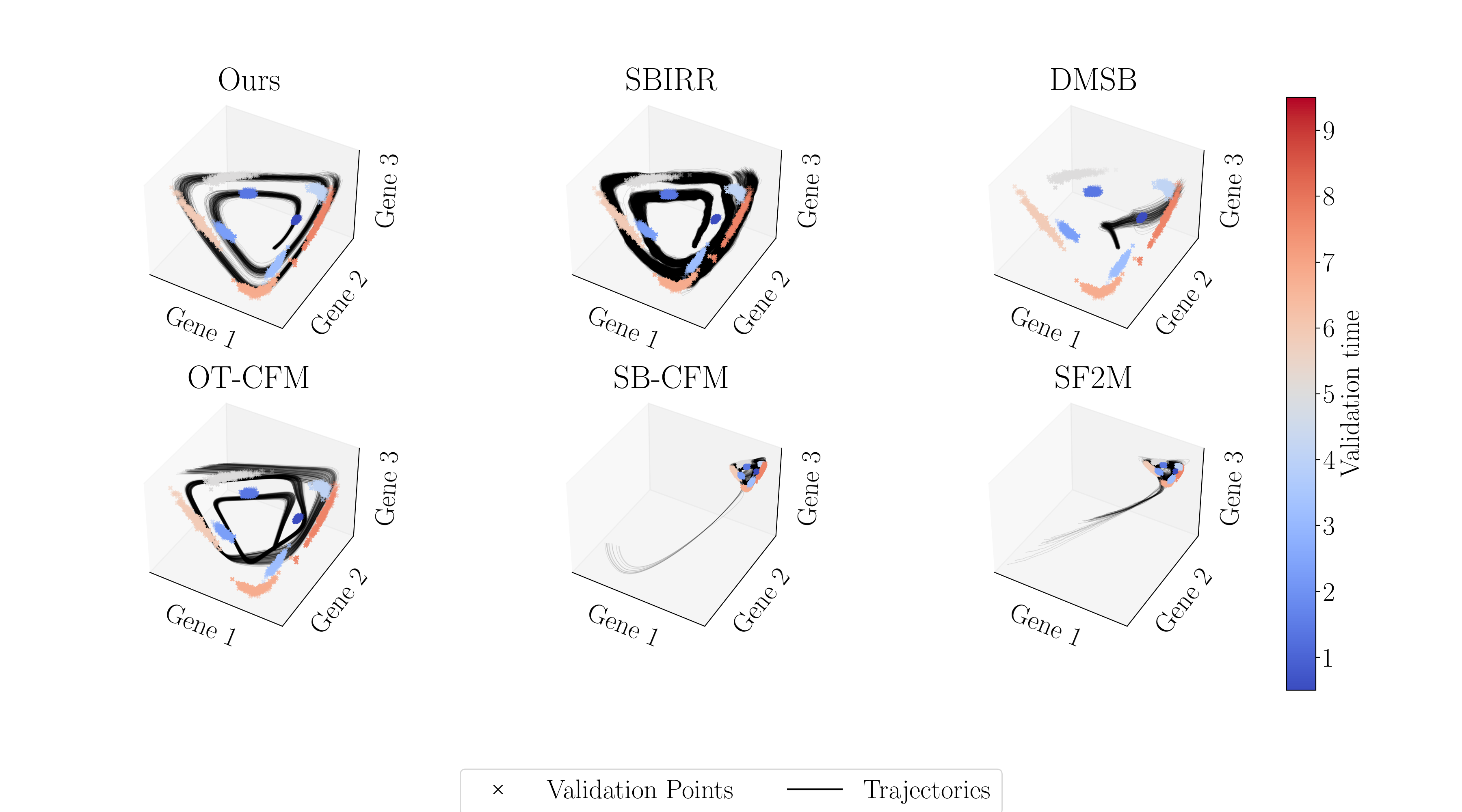}
    
    \caption{Semiparametric interpolation of repressilator system.}
    \label{fig:repr-appendix-semiparametric-interpol}
\end{figure}

\begin{figure}
    \centering

    \includegraphics[width=\linewidth]{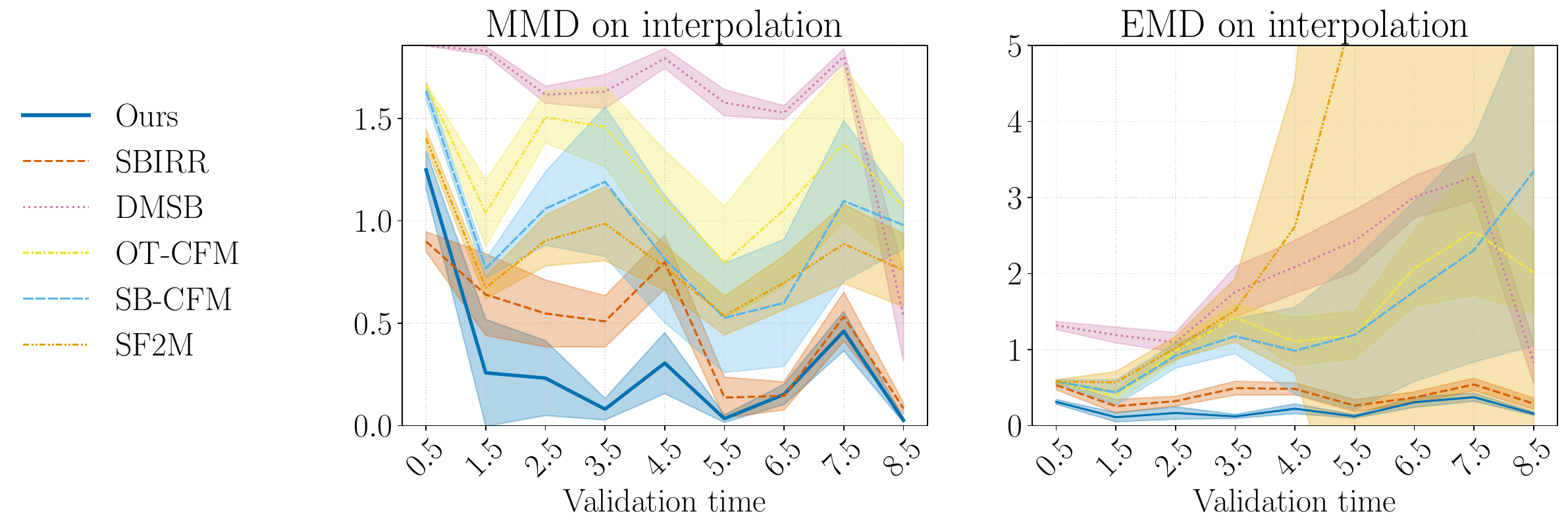}
    \caption{Metrics for Semiparametric interpolation of repressilator system.}
    \label{fig:repr-semiparametric-metric-interpol}
\end{figure}

\begin{table}[ht]
\caption{MMD at each validation point for Repressilator with parametric model family.}
\centering
\makebox[0pt][c]{
\small
\begin{tabular}{lrrrrrr}
\toprule
Time & Ours & \texttt{SBIRR} & \texttt{DMSB} & \texttt{OT-CFM} & \texttt{SB-CFM} & \texttt{SF2M} \\\midrule
0.5 & $1.249\pm{ 0.094}$ & \cellcolor{green!25}\bm{$0.899\pm{ 0.048}$} & $1.857\pm{ 0.002}$ & $1.666\pm{ 0.010}$ & $1.635\pm{ 0.043}$ & $1.404\pm{ 0.043}$ \\
1.5 & \cellcolor{green!25}\bm{$0.258\pm{ 0.261}$} & $0.640\pm{ 0.200}$ & $1.832\pm{ 0.020}$ & $1.038\pm{ 0.163}$ & $0.767\pm{ 0.046}$ & $0.673\pm{ 0.054}$ \\
2.5 & \cellcolor{green!25}\bm{$0.232\pm{ 0.183}$} & $0.548\pm{ 0.164}$ & $1.617\pm{ 0.042}$ & $1.506\pm{ 0.126}$ & $1.060\pm{ 0.179}$ & $0.904\pm{ 0.124}$ \\
3.5 & \cellcolor{green!25}\bm{$0.080\pm{ 0.053}$} & $0.509\pm{ 0.125}$ & $1.631\pm{ 0.083}$ & $1.460\pm{ 0.196}$ & $1.191\pm{ 0.365}$ & $0.987\pm{ 0.181}$ \\
4.5 & \cellcolor{green!25}\bm{$0.305\pm{ 0.149}$} & $0.799\pm{ 0.132}$ & $1.794\pm{ 0.048}$ & $1.104\pm{ 0.242}$ & $0.814\pm{ 0.311}$ & $0.774\pm{ 0.111}$ \\
5.5 & \cellcolor{green!25}\bm{$0.034\pm{ 0.018}$} & $0.137\pm{ 0.100}$ & $1.578\pm{ 0.064}$ & $0.796\pm{ 0.277}$ & $0.526\pm{ 0.267}$ & $0.537\pm{ 0.094}$ \\
6.5 & \cellcolor{green!25}$0.153\pm{ 0.050}$ & \cellcolor{green!25}\bm{$0.145\pm{ 0.069}$} & $1.529\pm{ 0.034}$ & $1.054\pm{ 0.373}$ & $0.600\pm{ 0.311}$ & $0.698\pm{ 0.131}$ \\
7.5 & \cellcolor{green!25}\bm{$0.462\pm{ 0.096}$} & \cellcolor{green!25}$0.532\pm{ 0.120}$ & $1.803\pm{ 0.037}$ & $1.377\pm{ 0.379}$ & $1.098\pm{ 0.392}$ & $0.888\pm{ 0.193}$ \\
8.5 & \cellcolor{green!25}\bm{$0.026\pm{ 0.011}$} & $0.084\pm{ 0.028}$ & $0.535\pm{ 0.222}$ & $1.067\pm{ 0.300}$ & $0.980\pm{ 0.117}$ & $0.760\pm{ 0.177}$ \\
\bottomrule
\end{tabular}}
\label{tab:repr-mlp-mmd}

\end{table}

\begin{table}[ht]
\caption{EMD at each validation point for Repressilator with parametric model family.}
\centering
\makebox[0pt][c]{
\small
\begin{tabular}{lrrrrrr}
\toprule
Time & Ours & \texttt{SBIRR} & \texttt{DMSB} & \texttt{OT-CFM} & \texttt{SB-CFM} & \texttt{SF2M} \\\midrule
0.5 & \cellcolor{green!25}\bm{$0.312\pm{ 0.032}$} & $0.533\pm{ 0.061}$ & $1.317\pm{ 0.053}$ & $0.578\pm{ 0.017}$ & $0.581\pm{ 0.022}$ & $0.581\pm{ 0.021}$ \\
1.5 & \cellcolor{green!25}\bm{$0.111\pm{ 0.060}$} & $0.255\pm{ 0.087}$ & $1.192\pm{ 0.104}$ & $0.384\pm{ 0.075}$ & $0.436\pm{ 0.165}$ & $0.569\pm{ 0.143}$ \\
2.5 & \cellcolor{green!25}\bm{$0.167\pm{ 0.082}$} & $0.322\pm{ 0.065}$ & $1.095\pm{ 0.131}$ & $0.985\pm{ 0.145}$ & $0.920\pm{ 0.164}$ & $1.023\pm{ 0.152}$ \\
3.5 & \cellcolor{green!25}\bm{$0.121\pm{ 0.027}$} & $0.493\pm{ 0.091}$ & $1.762\pm{ 0.330}$ & $1.424\pm{ 0.166}$ & $1.176\pm{ 0.229}$ & $1.515\pm{ 0.416}$ \\
4.5 & \cellcolor{green!25}\bm{$0.222\pm{ 0.064}$} & $0.483\pm{ 0.078}$ & $2.086\pm{ 0.348}$ & $1.104\pm{ 0.316}$ & $0.984\pm{ 0.571}$ & $2.618\pm{ 1.923}$ \\
5.5 & \cellcolor{green!25}\bm{$0.122\pm{ 0.023}$} & $0.263\pm{ 0.079}$ & $2.428\pm{ 0.409}$ & $1.193\pm{ 0.312}$ & $1.194\pm{ 0.986}$ & $5.505\pm{ 7.561}$ \\
6.5 & \cellcolor{green!25}\bm{$0.307\pm{ 0.067}$} & \cellcolor{green!25}$0.369\pm{ 0.075}$ & $3.006\pm{ 0.279}$ & $2.068\pm{ 0.502}$ & $1.765\pm{ 1.185}$ & $14.974\pm{ 27.778}$ \\
7.5 & \cellcolor{green!25}\bm{$0.374\pm{ 0.053}$} & $0.542\pm{ 0.081}$ & $3.270\pm{ 0.310}$ & $2.556\pm{ 0.842}$ & $2.306\pm{ 1.469}$ & $44.051\pm{ 96.243}$ \\
8.5 & \cellcolor{green!25}\bm{$0.157\pm{ 0.023}$} & $0.290\pm{ 0.074}$ & $0.828\pm{ 0.266}$ & $2.011\pm{ 0.534}$ & $3.345\pm{ 2.295}$ & $137.973\pm{ 328.155}$ \\
\bottomrule
\end{tabular}}
\label{tab:repr-mlp-emd}

\end{table}

\clearpage

\subsection{mRNA-protein repressilator}
\label{app:repr-app-missing}
\subsubsection{Experiment setup} 
\label{app:repr-missing-setup}
In \cref{app:repr-setup-parametric} we introduced the repressilator system as a system of SDEs governing changes in mRNA concentration. A more complete model for this system takes also into account protein levels. Indeed, each gene produces a protein that represses the next gene's expression, with the last one repressing the first. So proteins play a big role in the repressilator feedback loop. And this is why scientists often consider a more complex version of this system, that evolves according to the following SDEs:
\begin{equation}
    \begin{aligned}
    dX_1&=\alpha + \frac{\beta}{1+(Y_3/k)^n}-\gamma X_1+\sigma X_1dW_1 \\
    dX_2&=\alpha +\frac{\beta}{1+(Y_1/k)^n}-\gamma X_2+\sigma X_2dW_2\\
    dX_3&=\alpha +\frac{\beta}{1+(Y_2/k)^n}-\gamma X_3+\sigma X_2dW_3\\
    dY_1&=\beta_p X_1-\gamma_p Y_1+\sigma Y_1dW_4\\
    dY_2&=\beta_p X_2-\gamma_p Y_2+\sigma Y_2dW_5\\
    dY_3&=\beta_p X_3-\gamma_p Y_3+\sigma Y_3dW_6
\end{aligned}
\label{eq:repr-family-protein}
\end{equation}

where $[dW_1,dW_2,dW_3, dW_4, dW_5, dW_6]$ is a 6D Brownian motion. $X_1,X_2,X_3$ represents the mRNA levels while $Y_1, Y_2, Y_3$ are the corresponding proteins. As explained above, the actual system regulation is now mediated by proteins rather than mRNA themselves. 

To obtain data, we fix the following parameters: $\alpha = 10^{-5}, \beta = 10, n=3, k = 1, \gamma = 1,\beta_p=1, \gamma_p = 1, \sigma = 0.02$. We start the dynamics with initial distribution $X_1, X_2\sim U(1,1.1)$ and $X_3\sim U(2,2.1)$, while $Y_i\sim U(0,0.1)$. We simulate the SDEs for 10 instants of time. 

To simulate the system, we numerically integrate the SDEs over 19 discrete time points, with sampling rate 0.5 with the Euler-Maruyama scheme(implemented via the \texttt{torchsde} Python package) with 200 samples at each step. Out of these 19, we used 10 odd numbered time steps as training and even numbered steps as validation for interpolation task. We further simulate one step further with time increment of 1 to hold out as test point for forecasting. In all these steps we only took $X_i$ as observations

\subsubsection{Model family choice}
\label{app:repr_implementation_protein}
\textbf{Our method.} For this experiment, we have access to the data-generating process, as described in \cref{eq:repr-family-protein}. Therefore, we select our model family to be the set of SDEs that satisfy this system of equations, \cref{eq:repr-family-protein}. We initialize the missing dimensions at all 0. The learning process involves optimizing the parameters using gradient descent, with a learning rate of 0.05 over 500 epochs. We choose this number of epochs such that in the last 20 epochs $R^2$ increases by less than 0.01.

\textbf{A Note on Baselines.} Since the two forecasting baseline methods that we consider cannot handle incomplete state observations we cannot use them to fit \cref{eq:repr-family-protein}. Instead, we fit a simpler mRNA-only model as described in \cref{eq:repr-family}. We do the same for \texttt{SBIRR} in the interpolation experiment.

\subsubsection{Forecasting results}
\label{app:repr-missing-forecast}
In this section we give more detail on the forecasting results for mRNA-protein repressilator. We provide numerical results in EMD and MMD for forecasting in \cref{tab:mlp_repres_missing}. Our method outperform baseline by a large margin, mostly because the correct account of the missing protein observation. Since the two baselines cannot make vector fields in correct dimension we did not compare vector field reconstruction. 

\begin{table}[!ht]
    \centering
    \caption{Evaluation metric for Repressilator forecasting with missing protein observations.}
    \begin{tabular}{lccc}
    \hline
    &     & Repressilator (with missing protein) & \\
    Metric     &  Ours & \texttt{SBIRR-ref} & \texttt{SB-forward} \\
    \hline 
    Forecast-MMD & \cellcolor{green!25}\textbf{0.0075 (0.0047)}& 1.26 (0.06) & 1.22  (0.09) \\
    Forecast-EMD & \cellcolor{green!25}\textbf{0.26 (0.042)}& 6.36 (0.51) & 7.24  (0.48) \\
    \hline
    \end{tabular}
    \label{tab:mlp_repres_missing}
\end{table}

\subsubsection{Interpolation resuts}
\label{app:repr-missing-interpol}

We finally consider the interpolation task for this more challenging incomplete state observation setting. As shown in \cref{fig:repres-missing-protein-interpol}, despite the mismatch between the observed variables and the true system state, our method is able to faithfully reconstruct the trajectories. It successfully captures the geometry of the limit cycle and aligns well with the validation snapshots. \texttt{SBIRR} also performs reasonably, although its trajectories are more dispersed. All remaining baselines fail to track the correct dynamics: \texttt{DMSB}  and \texttt{SF2M}  exhibit severe trajectory drift, while \texttt{OT-CFM}  and \texttt{SB-CFM}  overly simplify the structure, failing to represent the circular flow of the system.
The lines in \cref{fig:repres-missing-protein-interpol-metric} and the numbers in \cref{tab:repr-missingobs-mmd}–\cref{tab:repr-missingobs-emd} confirm these findings. Our method consistently achieves the lowest MMD and EMD across all validation times. In contrast, baseline methods show significantly higher EMD and MMD throughout, and their confidence intervals do not overlap with ours unless for \texttt{SBIRR} for EMD in one validation point. As in the previous Repressilator experiment, we cap the y-axis at 5 in \cref{fig:repres-missing-protein-interpol-metric} to enable meaningful visual comparisons across methods. This is necessary because both \texttt{SF2M}  and \texttt{SB-CFM}  exhibit rapidly increasing EMD values after time 3.5. In particular, we observe that for certain random seeds, the inferred trajectories diverge in the wrong direction early on and continue along that path, resulting in large distributional mismatch. Accordingly, we also refrain from highlighting the corresponding cells in green, even when the coloring criterion is formally met, as the overlap with the best method arises from the extremely large variance.

\begin{figure}[!ht]
    \centering
    \includegraphics[width=\linewidth]{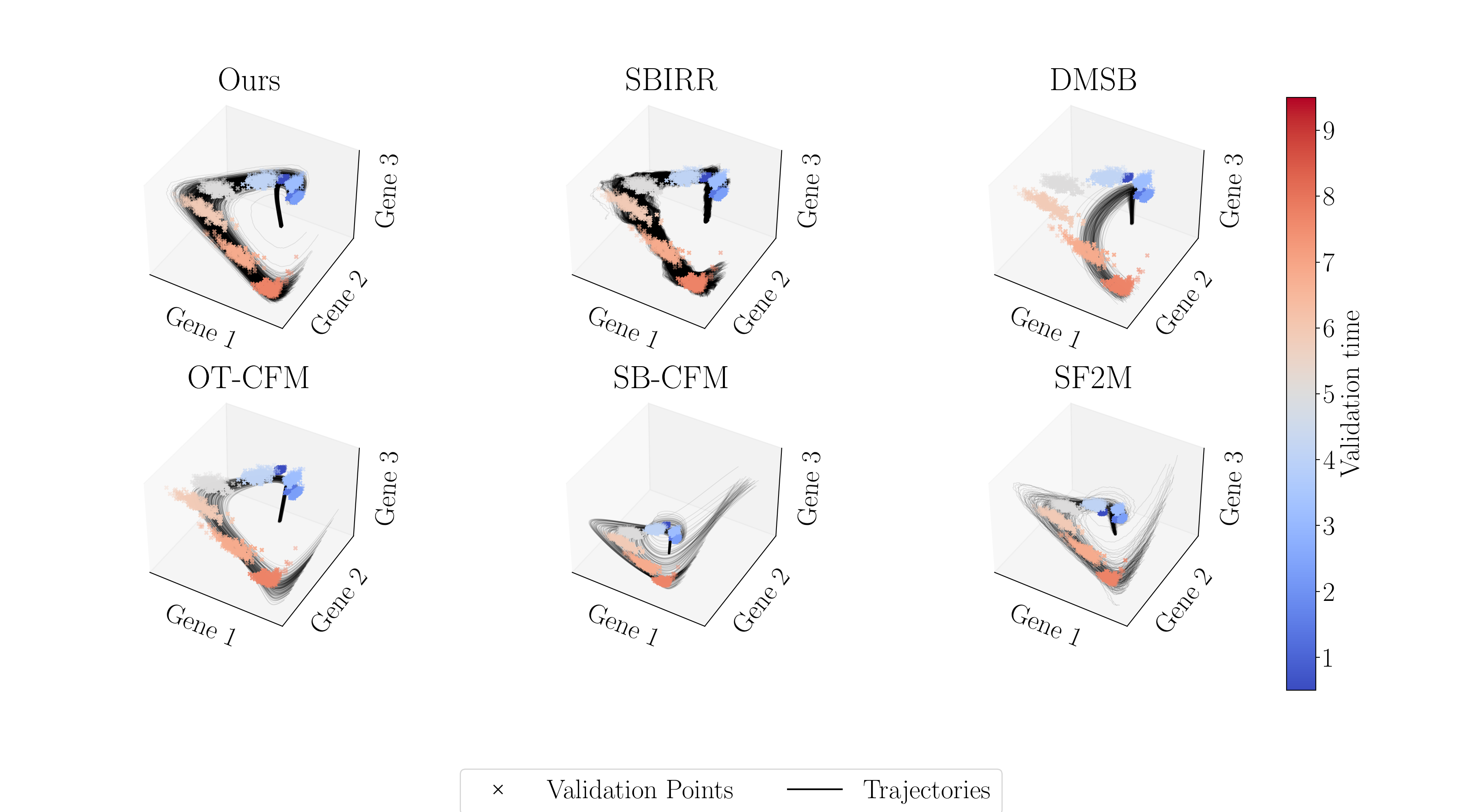}

    \caption{Interpolation of repressilator system with missing protein.}
    \label{fig:repres-missing-protein-interpol}
\end{figure}

\begin{figure}[!ht]
    \centering
    \includegraphics[width=\linewidth]{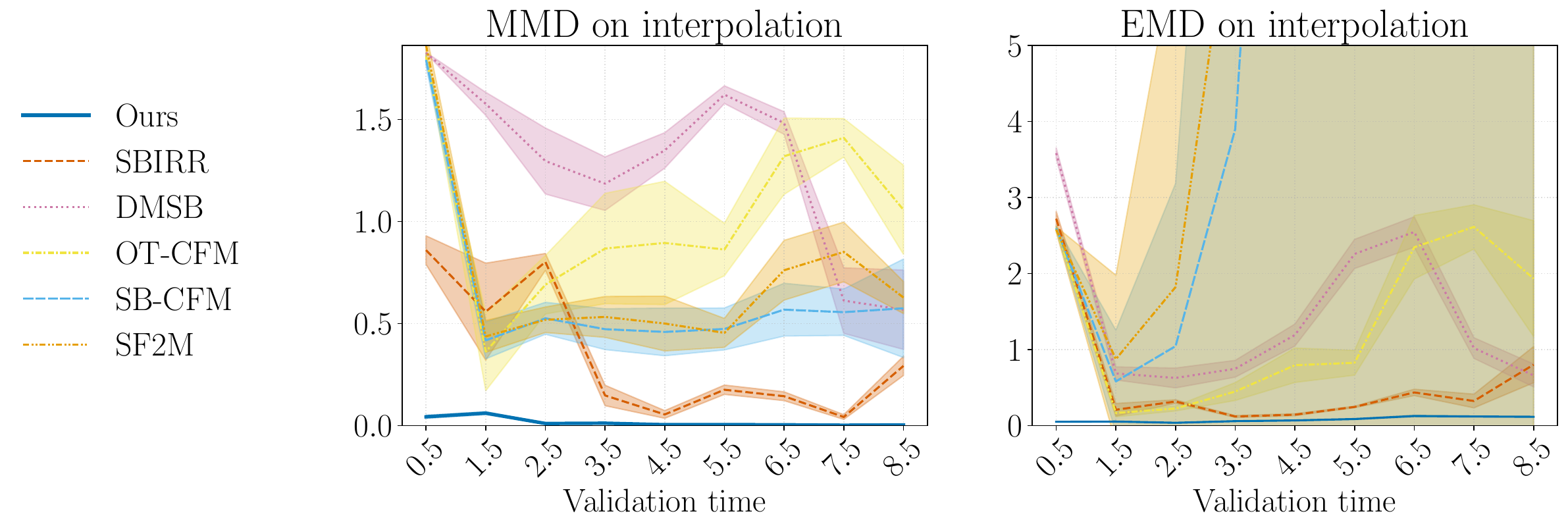}
    \caption{Metrics for Interpolation of repressilator system with missing protein.}
    \label{fig:repres-missing-protein-interpol-metric}
\end{figure}

\begin{table}[ht]
\caption{MMD at each validation point  for Repressilator with incomplete state observations.}
\centering
\makebox[0pt][c]{
\small
\begin{tabular}{lrrrrrr}
\toprule
Time & Ours & \texttt{SBIRR} & \texttt{DMSB} & \texttt{OT-CFM} & \texttt{SB-CFM} & \texttt{SF2M} \\\midrule
0.5 & \cellcolor{green!25}\bm{$0.043\pm{ 0.007}$} & $0.860\pm{ 0.070}$ & $1.826\pm{ 0.002}$ & $1.807\pm{ 0.010}$ & $1.791\pm{ 0.031}$ & $1.862\pm{ 0.066}$ \\
1.5 & \cellcolor{green!25}\bm{$0.061\pm{ 0.007}$} & $0.558\pm{ 0.238}$ & $1.576\pm{ 0.057}$ & $0.362\pm{ 0.192}$ & $0.418\pm{ 0.091}$ & $0.437\pm{ 0.078}$ \\
2.5 & \cellcolor{green!25}\bm{$0.011\pm{ 0.001}$} & $0.802\pm{ 0.042}$ & $1.296\pm{ 0.162}$ & $0.690\pm{ 0.143}$ & $0.526\pm{ 0.078}$ & $0.519\pm{ 0.063}$ \\
3.5 & \cellcolor{green!25}\bm{$0.012\pm{ 0.003}$} & $0.148\pm{ 0.050}$ & $1.185\pm{ 0.131}$ & $0.868\pm{ 0.270}$ & $0.472\pm{ 0.100}$ & $0.532\pm{ 0.100}$ \\
4.5 & \cellcolor{green!25}\bm{$0.005\pm{ 0.002}$} & $0.055\pm{ 0.019}$ & $1.348\pm{ 0.087}$ & $0.895\pm{ 0.302}$ & $0.459\pm{ 0.117}$ & $0.500\pm{ 0.134}$ \\
5.5 & \cellcolor{green!25}\bm{$0.005\pm{ 0.001}$} & $0.176\pm{ 0.023}$ & $1.621\pm{ 0.043}$ & $0.863\pm{ 0.128}$ & $0.474\pm{ 0.103}$ & $0.454\pm{ 0.071}$ \\
6.5 & \cellcolor{green!25}\bm{$0.004\pm{ 0.002}$} & $0.144\pm{ 0.022}$ & $1.482\pm{ 0.055}$ & $1.320\pm{ 0.187}$ & $0.568\pm{ 0.129}$ & $0.761\pm{ 0.147}$ \\
7.5 & \cellcolor{green!25}\bm{$0.002\pm{ 0.001}$} & $0.042\pm{ 0.012}$ & $0.612\pm{ 0.160}$ & $1.410\pm{ 0.094}$ & $0.556\pm{ 0.114}$ & $0.851\pm{ 0.146}$ \\
8.5 & \cellcolor{green!25}\bm{$0.004\pm{ 0.002}$} & $0.293\pm{ 0.047}$ & $0.568\pm{ 0.194}$ & $1.058\pm{ 0.218}$ & $0.575\pm{ 0.241}$ & $0.628\pm{ 0.079}$ \\
\bottomrule
\end{tabular}}
\label{tab:repr-missingobs-mmd}

\end{table}

\begin{table}[ht]
\caption{EMD at each validation point for Repressilator with incomplete state observations.}
\centering
\makebox[0pt][c]{
\small
\begin{tabular}{lrrrrrr}
\toprule
Time & Ours & \texttt{SBIRR} & \texttt{DMSB} & \texttt{OT-CFM} & \texttt{SB-CFM} & \texttt{SF2M} \\\midrule
0.5 & \cellcolor{green!25}\bm{$0.051\pm{ 0.003}$} & $2.721\pm{ 0.097}$ & $3.583\pm{ 0.071}$ & $2.596\pm{ 0.028}$ & $2.596\pm{ 0.029}$ & $2.578\pm{ 0.036}$ \\
1.5 & \cellcolor{green!25}\bm{$0.054\pm{ 0.002}$} & $0.209\pm{ 0.083}$ & $0.686\pm{ 0.086}$ & $0.166\pm{ 0.050}$ & $0.581\pm{ 0.675}$ & $0.874\pm{ 1.104}$ \\
2.5 & \cellcolor{green!25}\bm{$0.037\pm{ 0.001}$} & $0.317\pm{ 0.023}$ & $0.627\pm{ 0.131}$ & $0.225\pm{ 0.032}$ & $1.045\pm{ 2.139}$ & $1.823\pm{ 4.516}$ \\
3.5 & \cellcolor{green!25}\bm{$0.059\pm{ 0.002}$} & $0.120\pm{ 0.013}$ & $0.748\pm{ 0.111}$ & $0.449\pm{ 0.117}$ & $3.900\pm{ 9.557}$ & $7.221\pm{ 19.895}$ \\
4.5 & \cellcolor{green!25}\bm{$0.068\pm{ 0.004}$} & $0.144\pm{ 0.013}$ & $1.192\pm{ 0.145}$ & $0.794\pm{ 0.227}$ & $16.018\pm{ 44.580}$ & $30.671\pm{ 89.239}$ \\
5.5 & \cellcolor{green!25}\bm{$0.087\pm{ 0.005}$} & $0.246\pm{ 0.010}$ & $2.258\pm{ 0.192}$ & $0.826\pm{ 0.163}$ & $72.134\pm{ 212.668}$ & $136.803\pm{ 407.955}$ \\
6.5 & \cellcolor{green!25}\bm{$0.126\pm{ 0.011}$} & $0.437\pm{ 0.042}$ & $2.547\pm{ 0.200}$ & $2.346\pm{ 0.418}$ & $354.032\pm{ 1056.419}$ & $649.221\pm{ 1942.159}$ \\
7.5 & \cellcolor{green!25}\bm{$0.120\pm{ 0.006}$} & $0.324\pm{ 0.090}$ & $1.020\pm{ 0.136}$ & $2.614\pm{ 0.291}$ & $1691.087\pm{ 5066.068}$ & $2986.712\pm{ 8953.868}$ \\
8.5 & \cellcolor{green!25}\bm{$0.116\pm{ 0.005}$} & $0.802\pm{ 0.238}$ & $0.660\pm{ 0.150}$ & $1.934\pm{ 0.762}$ & $8102.608\pm{ 24292.997}$ & $13770.547\pm{ 41303.163}$ \\
\bottomrule
\end{tabular}}
\label{tab:repr-missingobs-emd}
\end{table}

\clearpage

\subsection{Current in the Gulf of Mexico}
\subsubsection{Experimental setup}
\label{app:gom-setup}
We test our method in fitting and forecasting real ocean-current data from the Gulf of Mexico. We use high-resolution (1 km) bathymetry data from a HYbrid Coordinate Ocean Model (HYCOM) reanalysis\footnote{Dataset available at \href{https://www.hycom.org/data/gomb0pt01/gom-reanalysis}{this link}.} \citep{panagiotis2014gulf}. This dataset has been in the public domain since it was released by the US Department of Defense. The dataset provides hourly ocean current velocity fields for the region extending from 98$^{\circ}$E to 77$^{\circ}$E in longitude and from 18$^{\circ}$N to 32$^{\circ}$N in latitude, covering every day since the first day of, 2001. 

We then generate particles following \citet{shen2024multi}. That is, we took the velocity field in a region where a vortex is observed in June 1st 2024 at 5pm. We then select an initial location near the vortex and uniformly sample 4,400 initial positions within a small radius ($0.05$) around this point and evolve these particles over 11 steps using the ocean current velocity field. The time step size is $1.0$ and left the last time step as validation. At each time point we retain 400 particles. We approximate the velocity at each particle's position using the velocity at the nearest grid point when the particle does not align exactly with a grid point. In addition, between each training time point, we simulate another 9 intermediate steps at middle point between each pair of consecutive training time points, with 400 particle each to test for interpolation. 

\subsubsection{Model family choice}
\label{app:gom_implementation}

We employ a physically motivated model to represent the vortex by combining a Lamb-Oseen vortex --- a solution of the two-dimensional viscous Navier-Stokes equations \citep{saffman1995vortex} --- with a constant divergence field. The Lamb-Oseen component captures the swirling, rotational dynamics typical of a vortex, while the divergence field is added to account for vertical motion or non-conservative forces that may cause a net expansion or contraction of the flow. In other words, this combined model enables us to represent both the core vortex behavior and the secondary effects influencing particle motion.

Formally, the particle trajectories are modeled by the following family of stochastic differential equations (SDEs):
\begin{equation}
    \begin{aligned}
        dX&=\left[-\gamma\frac{(Y-y_0)r_v}{(Y-y_0)^2+(Y-y_0)^2}\left(1-\exp\left(\frac{\sqrt{(Y-y_0)^2+(Y-y_0)^2}}{r_v}\right)\right)+d\frac{X-x_{0,d}}{r_d}  \right]dt+\sigma dW_x\\
        dY&=\left[\gamma\frac{(X-x_0)r_v}{(Y-y_0)^2+(Y-y_0)^2}\left(1-\exp\left(\frac{\sqrt{(Y-y_0)^2+(Y-y_0)^2}}{r_v}\right)\right)+d(Y-y_{0,d})  \right]dt+\sigma dW_y\\
    \end{aligned}
\end{equation}

In this formulation, the free parameters are:
\begin{itemize}
    \item Circulation ($\gamma$): Controls the strength of the vortex.
    \item Vortex length scale ($r_v$): Sets the radial decay of the vortex’s influence.
    \item Vortex center ($x_0, y_0$): Specifies the location of the vortex core.
    \item Divergence ($d$): Represents the magnitude of the constant divergence field.
    \item Divergence length scale ($r_d$): Governs the spatial extent of the divergence effect in the x-direction.
    \item Divergence center ($x_{0,d}, y_{0,d}$): Determines the reference location for the divergence field.
    \item Volatility ($\sigma$): Captures the stochastic fluctuations in particle motion.
\end{itemize}

\subsubsection{Forecasting results}
In this section, we further discuss results for the Gulf of Mexico vortex experiment. In particular, we analyze the EMD, MMD.

\begin{table}[!ht]
    \caption{Evaluation metric for Gulf of Mexico experiment (mean(sd)). Drift was evaluated using MSE on a grid.}
    \centering
    \begin{tabular}{lccc}
    \hline
    &     & Gulf of Mexico vortex & \\
    Metric     &  Ours & \texttt{SBIRR-ref} & \texttt{SB-forward} \\
    \hline
    Forecast-MMD & 0.66 (0.031) & \cellcolor{green!25}\textbf{0.35 (0.032)}& 0.62 (0.054) \\
    Forecast-EMD &  \cellcolor{green!25}\textbf{0.71(0.014)} & 0.89(0.034)& 0.94(0.081) \\
    Drift & 0.054 ($7.3\times 10^{-5}$) & \cellcolor{green!25}\textbf{0.031 (0.00032)} & 0.15 (0.023) \\
    \hline
    \end{tabular}

    \label{tab:GoM}
\end{table}

In the second row of \cref{tab:GoM}, we observe that our method achieves the lowest EMD for the forecasting task, indicating the closest match to the ground truth particle distribution. This aligns with the visual results in \cref{fig:GoMforecast} from the main text, where our forecasted particles accurately capture the spatial structure of the vortex, unlike the baselines which produce more scattered and less coherent predictions. While the MMD metric (first row) slightly favors the \texttt{SBIRR-ref} baseline, this discrepancy may be attributed to the sensitivity of MMD to particle density and kernel choice.

\subsubsection{Vector field reconstruction}
\label{app:gom-vector}
In the third row of \cref{tab:GoM}, we compare the drift reconstruction error using MSE on a grid. Here \texttt{SBIRR-ref} achieves the lowest error, and our method performs comparably and still significantly outperforms \texttt{SB-forward}. Visualizations in \cref{fig:gom-appendix} provide further insight: the reconstructed velocity fields from all the three methods exhibit a well-formed vortex structure closely resembling the ground truth (with our method and \texttt{SBIRR-ref} being slightly better, as also shown by the MSE results). 

\begin{figure}[!ht]
    \centering
    \includegraphics[width=\linewidth]{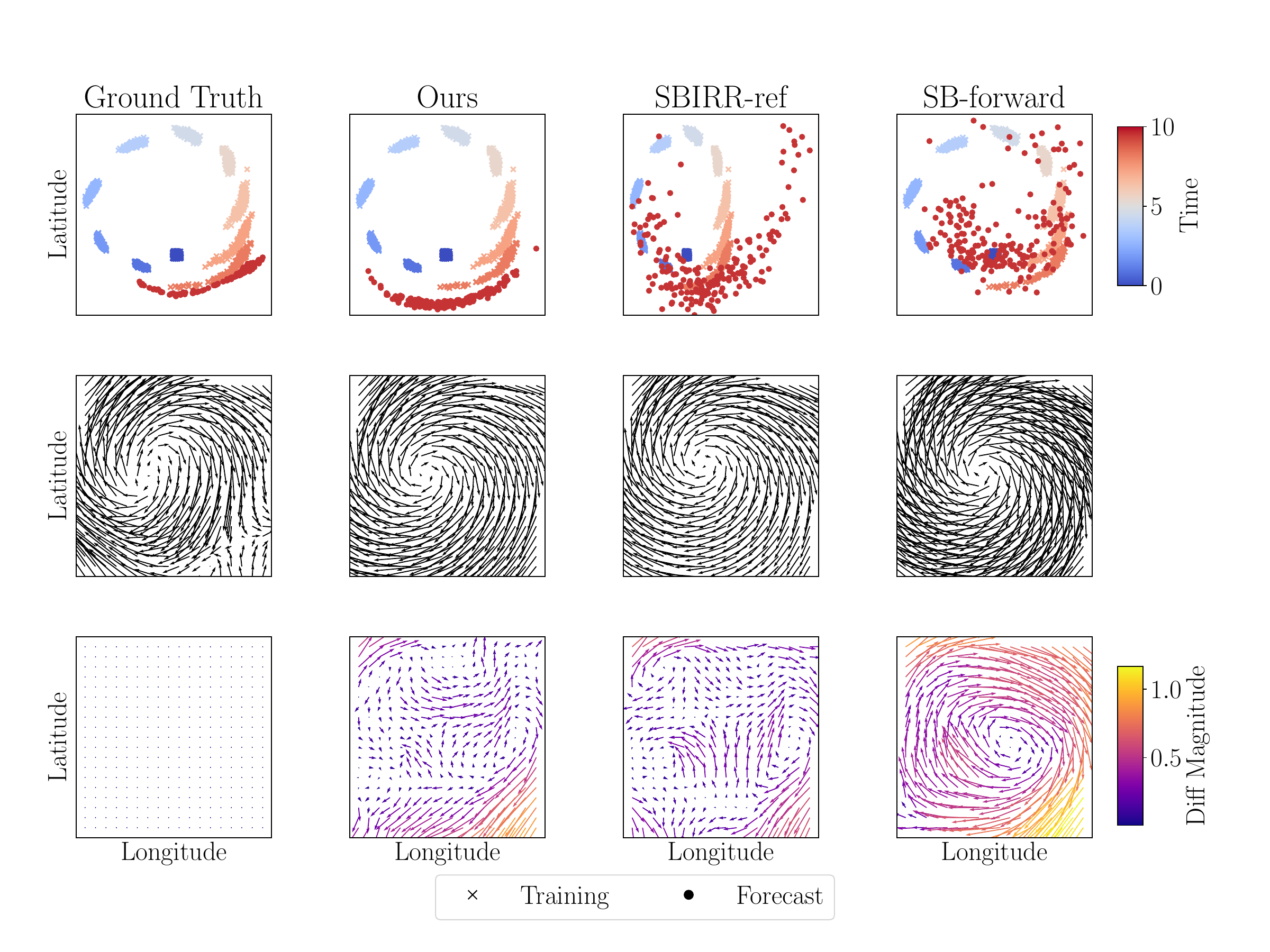}
    \caption{Experimental results for the Gulf of Mexico experiment.}
    \label{fig:gom-appendix}
\end{figure}

\subsubsection{Interpolation}
\label{app:gom-interpol}

We now evaluate interpolation performance on the real-world drifter trajectories in the Gulf of Mexico. As shown in \cref{fig:GoM-interpol}, our method captures the overall geometry of the flow and the looping structure of the trajectories. All baselines also succeed in reconstructing the large-scale circulation, with \texttt{SBIRR} achieving the closest match to the held-out validation points. These patterns are reflected in the quantitative metrics in \cref{fig:GoM-interpol-metric} and tables \cref{tab:gom-mmd}–\cref{tab:gom-emd}: across nearly all validation times, \texttt{SBIRR} achieves the lowest MMD and EMD values, consistently outperforming our method and often the second-best method by a considerable margin.

This performance difference highlights a key distinction between modeling objectives. \texttt{SBIRR} is designed to directly interpolate between training marginals, and in this setting—where particles are relatively dense and the underlying flow field is smooth—interpolating training points naturally leads to trajectories that also pass near the validation points, which lie in between. In contrast, our method is not explicitly optimized to pass through the training marginals, but rather to estimate a smooth underlying velocity field from the available data. This distinction becomes particularly relevant in tasks such as forecasting, where the goal is to recover and extrapolate the underlying dynamics rather than simply interpolate known states.

\begin{figure}
    \centering
    \includegraphics[width=\linewidth]{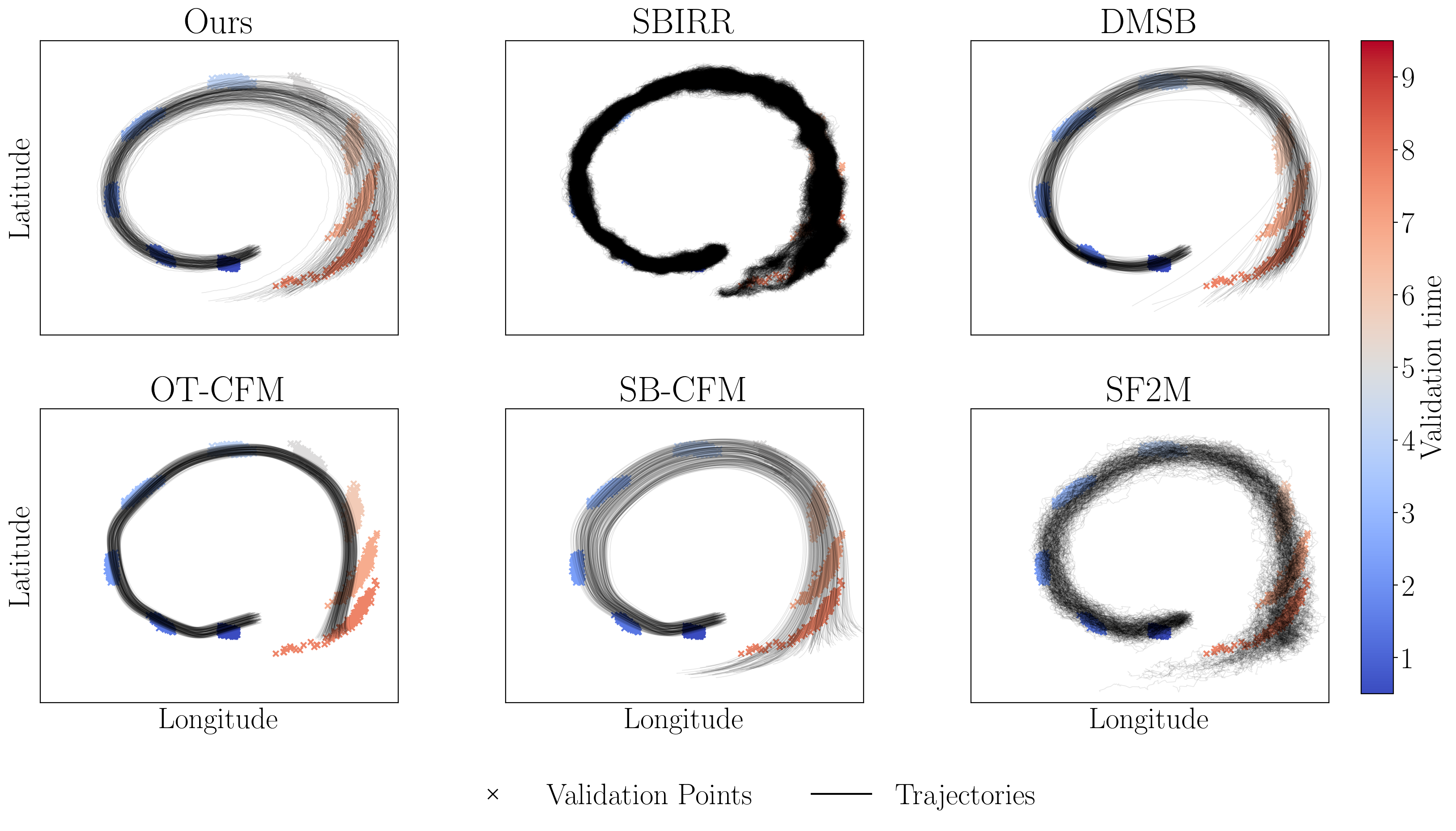}
    
    \caption{Interpolation of Gulf of Mexico current.}
    \label{fig:GoM-interpol}
\end{figure}

\begin{figure}
    \centering
    \includegraphics[width=\linewidth]{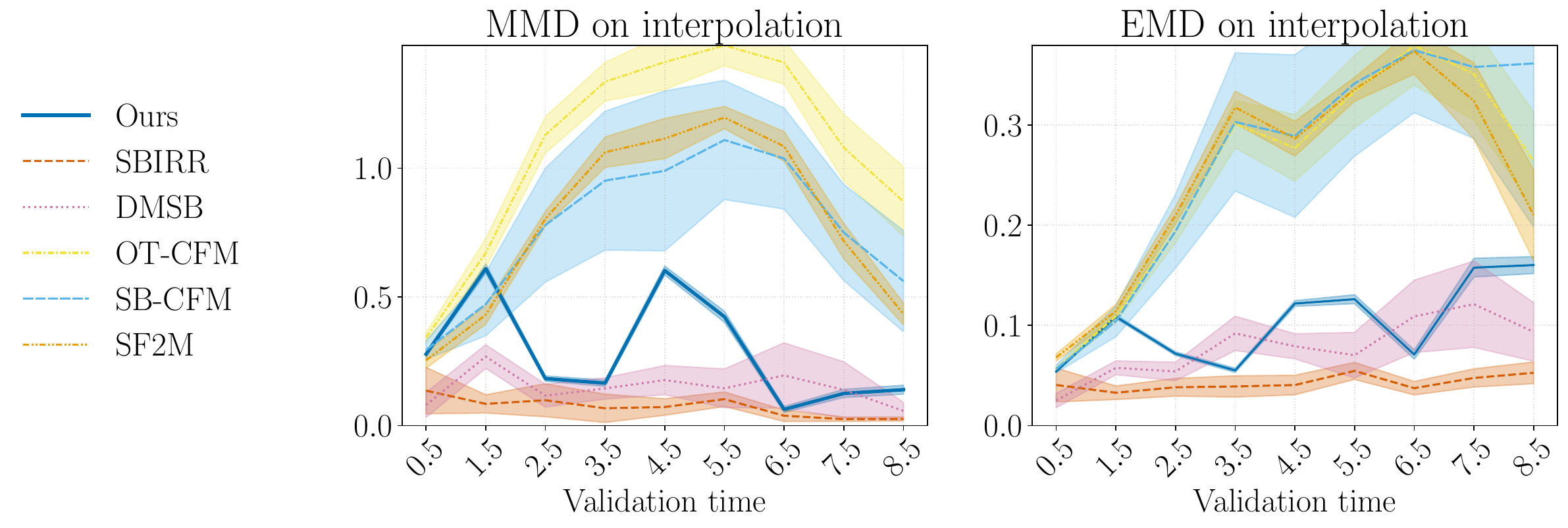}
    \caption{Metrics for Interpolation of Gulf of Mexico current.}
    \label{fig:GoM-interpol-metric}
\end{figure}

\begin{table}[ht]
\caption{MMD at each validation point for Gulf of Mexico.}
\centering
\makebox[0pt][c]{
\small
\begin{tabular}{lrrrrrr}
\toprule
Time & Ours & \texttt{SBIRR} & \texttt{DMSB} & \texttt{OT-CFM} & \texttt{SB-CFM} & \texttt{SF2M} \\\midrule
0.5 & $0.278\pm{ 0.009}$ & $0.136\pm{ 0.090}$ & \cellcolor{green!25}\bm{$0.080\pm{ 0.049}$} & $0.342\pm{ 0.019}$ & $0.294\pm{ 0.036}$ & $0.254\pm{ 0.031}$ \\
1.5 & $0.610\pm{ 0.018}$ & \cellcolor{green!25}\bm{$0.084\pm{ 0.035}$} & $0.268\pm{ 0.046}$ & $0.670\pm{ 0.056}$ & $0.472\pm{ 0.123}$ & $0.430\pm{ 0.037}$ \\
2.5 & $0.183\pm{ 0.010}$ & \cellcolor{green!25}\bm{$0.099\pm{ 0.064}$} & \cellcolor{green!25}$0.116\pm{ 0.044}$ & $1.129\pm{ 0.071}$ & $0.779\pm{ 0.221}$ & $0.803\pm{ 0.029}$ \\
3.5 & $0.165\pm{ 0.012}$ & \cellcolor{green!25}\bm{$0.067\pm{ 0.055}$} & $0.144\pm{ 0.041}$ & $1.335\pm{ 0.075}$ & $0.951\pm{ 0.270}$ & $1.062\pm{ 0.059}$ \\
4.5 & $0.602\pm{ 0.017}$ & \cellcolor{green!25}\bm{$0.072\pm{ 0.032}$} & $0.177\pm{ 0.057}$ & $1.412\pm{ 0.107}$ & $0.989\pm{ 0.311}$ & $1.115\pm{ 0.078}$ \\
5.5 & $0.423\pm{ 0.022}$ & \cellcolor{green!25}\bm{$0.103\pm{ 0.028}$} & $0.145\pm{ 0.075}$ & $1.476\pm{ 0.079}$ & $1.109\pm{ 0.231}$ & $1.196\pm{ 0.043}$ \\
6.5 & $0.063\pm{ 0.012}$ & \cellcolor{green!25}\bm{$0.038\pm{ 0.022}$} & $0.195\pm{ 0.126}$ & $1.410\pm{ 0.084}$ & $1.037\pm{ 0.196}$ & $1.084\pm{ 0.057}$ \\
7.5 & $0.125\pm{ 0.016}$ & \cellcolor{green!25}\bm{$0.025\pm{ 0.008}$} & $0.139\pm{ 0.109}$ & $1.079\pm{ 0.129}$ & $0.750\pm{ 0.187}$ & $0.717\pm{ 0.069}$ \\
8.5 & $0.140\pm{ 0.017}$ & \cellcolor{green!25}\bm{$0.026\pm{ 0.008}$} & $0.057\pm{ 0.034}$ & $0.870\pm{ 0.135}$ & $0.561\pm{ 0.196}$ & $0.434\pm{ 0.042}$ \\
\bottomrule
\end{tabular}}
\label{tab:gom-mmd}
\end{table}

\begin{table}[ht]
\caption{EMD at each validation point for Gulf of Mexico.}
\centering
\makebox[0pt][c]{
\small
\begin{tabular}{lrrrrrr}
\toprule
Time & Ours & \texttt{SBIRR} & \texttt{DMSB} & \texttt{OT-CFM} & \texttt{SB-CFM} & \texttt{SF2M} \\\midrule
0.5 & $0.054\pm{ 0.001}$ & $0.041\pm{ 0.017}$ & \cellcolor{green!25}\bm{$0.025\pm{ 0.007}$} & $0.057\pm{ 0.002}$ & $0.057\pm{ 0.004}$ & $0.068\pm{ 0.004}$ \\
1.5 & $0.109\pm{ 0.002}$ & \cellcolor{green!25}\bm{$0.033\pm{ 0.007}$} & $0.058\pm{ 0.007}$ & $0.109\pm{ 0.006}$ & $0.104\pm{ 0.015}$ & $0.114\pm{ 0.007}$ \\
2.5 & $0.072\pm{ 0.002}$ & \cellcolor{green!25}\bm{$0.038\pm{ 0.009}$} & $0.054\pm{ 0.009}$ & $0.194\pm{ 0.012}$ & $0.194\pm{ 0.037}$ & $0.210\pm{ 0.009}$ \\
3.5 & $0.055\pm{ 0.002}$ & \cellcolor{green!25}\bm{$0.039\pm{ 0.011}$} & $0.092\pm{ 0.017}$ & $0.301\pm{ 0.024}$ & $0.303\pm{ 0.069}$ & $0.318\pm{ 0.016}$ \\
4.5 & $0.122\pm{ 0.003}$ & \cellcolor{green!25}\bm{$0.041\pm{ 0.010}$} & $0.079\pm{ 0.013}$ & $0.277\pm{ 0.033}$ & $0.289\pm{ 0.081}$ & $0.286\pm{ 0.017}$ \\
5.5 & $0.126\pm{ 0.004}$ & \cellcolor{green!25}\bm{$0.055\pm{ 0.009}$} & $0.070\pm{ 0.023}$ & $0.333\pm{ 0.036}$ & $0.342\pm{ 0.073}$ & $0.336\pm{ 0.012}$ \\
6.5 & $0.071\pm{ 0.005}$ & \cellcolor{green!25}\bm{$0.037\pm{ 0.007}$} & $0.109\pm{ 0.036}$ & $0.379\pm{ 0.040}$ & $0.374\pm{ 0.062}$ & $0.374\pm{ 0.023}$ \\
7.5 & $0.158\pm{ 0.009}$ & \cellcolor{green!25}\bm{$0.048\pm{ 0.009}$} & $0.121\pm{ 0.043}$ & $0.351\pm{ 0.046}$ & $0.358\pm{ 0.071}$ & $0.324\pm{ 0.038}$ \\
8.5 & $0.160\pm{ 0.008}$ & \cellcolor{green!25}\bm{$0.053\pm{ 0.011}$} & $0.094\pm{ 0.029}$ & $0.264\pm{ 0.049}$ & $0.361\pm{ 0.163}$ & $0.210\pm{ 0.045}$ \\
\bottomrule
\end{tabular}}
\label{tab:gom-emd}
\end{table}

\clearpage

\subsection{T cell-mediated immune activation}
In this section, we describe details of the T cell-mediated immune activation experiment contained in the main text. The dataset was released with the paper by \citet{jiang2024d} under CC-BY-NC 4.0 international license. 
\subsubsection{Background for the biological experiment.}
\label{app:pbmc-biology}
\paragraph{Peripheral blood mononuclear cells}
Peripheral blood mononuclear cells (PBMCs) comprise a heterogeneous mixture of immune cells, including T cells, B cells, natural killer (NK) cells, and myeloid lineages.  
To trigger a coordinated immune response, the PBMC pool from a single healthy donor was stimulated \emph{in vitro} with anti-CD3/CD28 antibodies at $t=0$ to selectively induce T cell activation.   
Cytokines released by activated T cells subsequently induce transcriptional changes and subsequent cytokine communications in the immune cell populations, yielding a rich, system-wide dynamical process.

\paragraph{Data acquisition.}
Cells were sampled every 30 min for 30 h (61 time points) and profiled with multiplexed single-cell mRNA sequencing \citep{jiang2024d}.  Raw counts were library-size normalised and log-transformed following standard scRNA-seq workflows.  
Each snapshot is very high-dimensional (there are hundreds of cells and thousands of genes for each cell). 
Because modelling all genes directly is infeasible and biologically redundant, we adopt the widely used \emph{gene-program} formulation: groups of co-expressed genes are collapsed into latent variables that capture coordinated transcriptional activity.  

\subsubsection{Experiment setup}
\label{app:pbmc-setup}
We use the 30 biologically annotated programs in \citet{jiang2024d} together with the dataset, which was computed from orthogonal non-negative matrix factorisation (oNMF) \citep{ding2006orthogonal}. Projecting each cell onto this 30-dimensional program space yields a low-noise, interpretable representation that is well suited for dynamical modelling. We took data from 0-20 hours (41 snapshots in total), before the cells reached steady state. We train our model using data at $0, 1, \dots, 19$th hours for training, left measurement at $20$th hour to test for forecast, and left $0.5, 1.5, \dots, 18.5$th hour to test for interpolation. We show in \cref{fig:pbmc-training-progression} the evolution of the 20 training points and the forecast validation time point (bottom-left).

\begin{figure}
    \centering
    \includegraphics[width=1\linewidth]{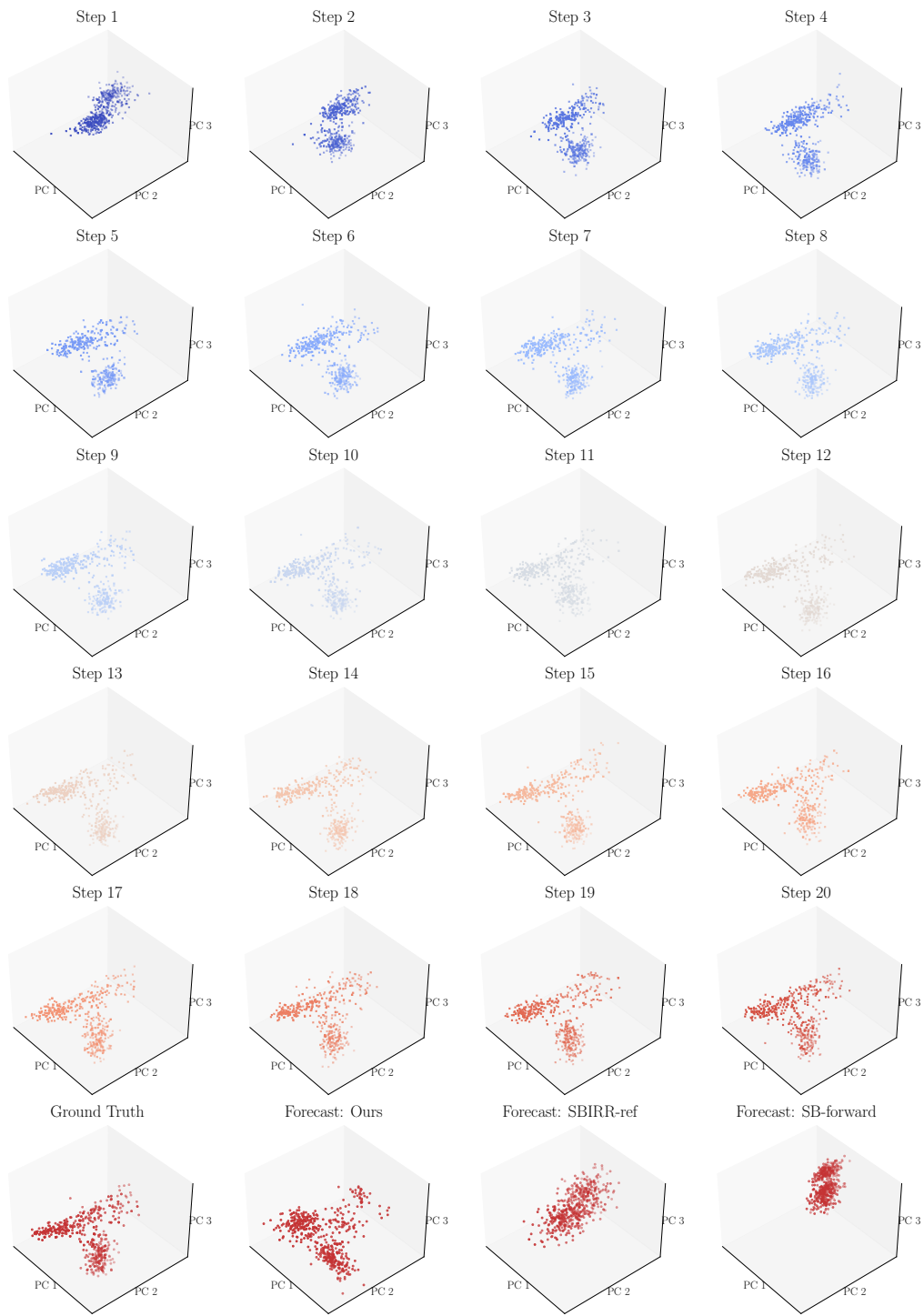}
    \caption{Training data for the PBMC experiment and forecasts with the three methods.}
    \label{fig:pbmc-training-progression}
\end{figure}

\subsubsection{Model family choice}
\label{app:pbmc-implementation}
We employ the architecture of \cref{eq:mlpactive}, instantiated with a 30-dimensional state space.  
The drift is parameterised by a three-layer MLP with hidden widths $[128,128,128]$, ReLU activations and a sigmoid output that keeps gene-program values within biologically plausible bounds.  
Hyper-parameters were chosen via a small grid search (two vs. three layers, and 32 vs. 64 vs. 128 hidden nodes per layer) on the $R^{2}$ score.

\subsubsection{Forecasting results}
\label{app:pbmc-forecast}
Table~\ref{tab:repr-pbmc} reports quantitative forecasting performance.  
Because the ground-truth vector field is unknown, we restrict evaluation to distributional metrics.  
EMD solver failed to converge in 30 dimensions, so we use MMD with an RBF kernel of bandwidth~1.  
Our method attains an MMD that is an order of magnitude lower than either baseline, confirming the qualitative superiority observed in \cref{fig:pbmc-forecasting}.  

\begin{table}[!ht]
    \centering
    \caption{One-step-ahead forecasting error on the T-cell activation dataset (mean $\pm$ s.d.\ over 10 seeds). MMD is computed with an RBF kernel of bandwidth 1 after scaling each gene program to unit variance. EMD is not reported because the solver failed to converge in 30 dimensions.}
    \begin{tabular}{lccc}
    \hline
    &     & pbmc & \\
    Metric     &  Ours & \texttt{SBIRR-ref} & \texttt{SB-forward} \\
    \hline
    Forecast-MMD & \cellcolor{green!25}\textbf{0.0042 (0.0017)} & 0.097 (0.060)& 0.51 (0.16) \\
    \hline
    \end{tabular}

    \label{tab:repr-pbmc}
\end{table}

\subsubsection{Interpolation}
\label{app:pbmc-interpol}

We now assess interpolation performance on this real-world single-cell dataset. As shown in \cref{fig:pbmc-interpol}, our method produces biologically plausible trajectories that span the principal components of the data and remain well-aligned with the progression of held-out validation points. \texttt{SBIRR}, \texttt{OT-CFM} , and \texttt{SB-CFM}  perform similarly well in this task, whereas \texttt{DMSB}  produces notably erratic paths with high variance, and \texttt{SF2M}  displays more diffuse interpolations. The same pattern appears if we look at interpolation predictions across the validation time steps, as shown in \cref{fig:pbmc-interpolation-1-5}, \cref{fig:pbmc-interpolation-6-10}, \cref{fig:pbmc-interpolation-11-15}, and \cref{fig:pbmc-interpolation-16-20}. In particular, we see that our method and \texttt{SBIRR} produce quite accurate interpolations for most of the time steps. Quantitative results in \cref{fig:pbmc-interpol-metric} and \cref{tab:pbmc-mmd} support these impressions. Across the full time course, our method consistently achieves the lowest MMD at most validation points, often by a statistically significant margin. \texttt{SBIRR} performs competitively, particularly at later times, while the performance of \texttt{DMSB}  degrades substantially, as reflected by persistently high MMD values throughout the trajectory. We note that, in contrast to previous experiments, we do not report EMD in this setting. Due to the high dimensionality of the latent space (30 dimensions), EMD becomes increasingly unreliable as a metric, suffering from the curse of dimensionality and producing unstable estimates. For this reason, we focus our evaluation on MMD, which remains well-behaved in high-dimensional settings and provides a more robust comparison of distributional fidelity across methods.

\begin{figure}[!ht]
    \centering
    \includegraphics[width=\linewidth]{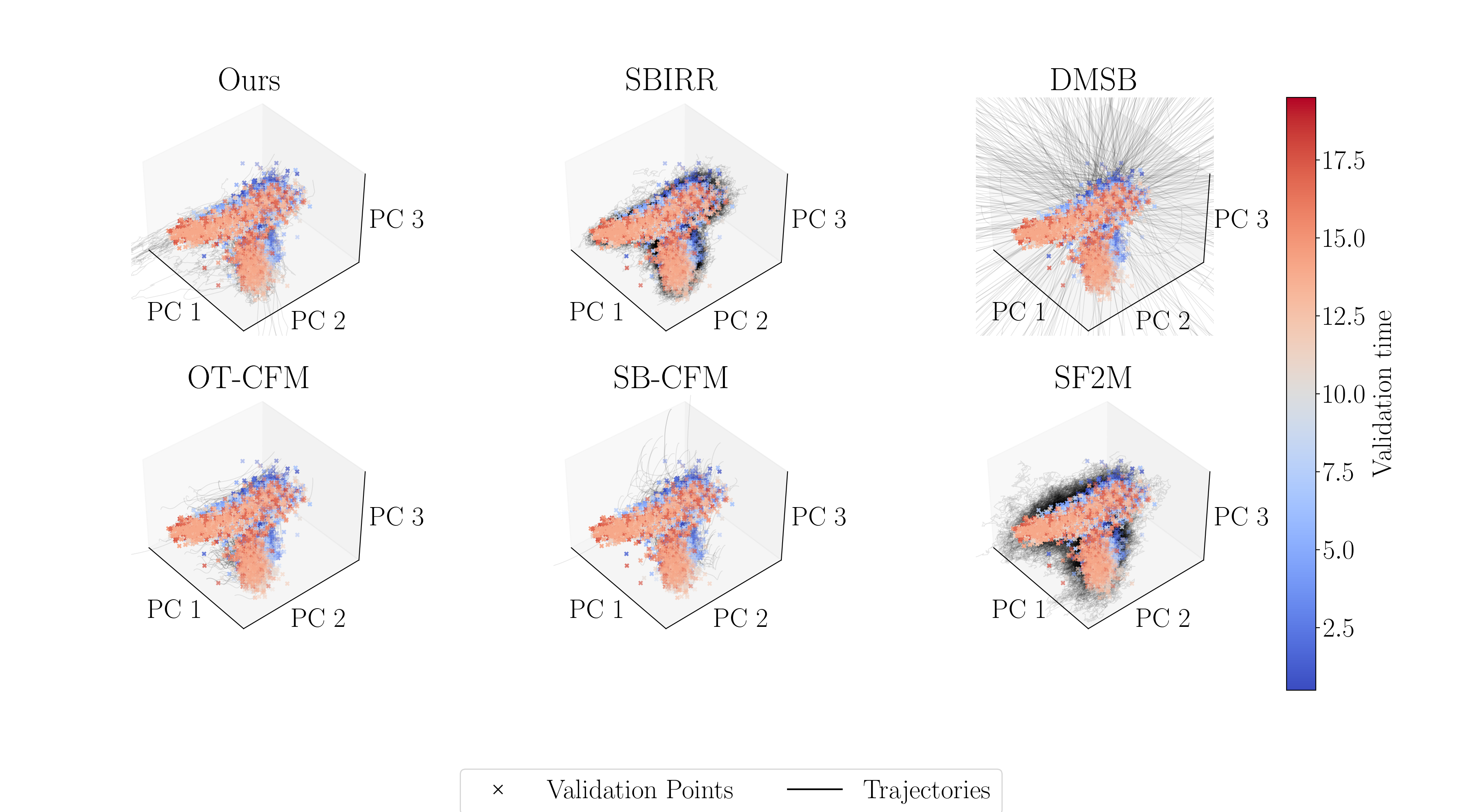}
    
    \caption{Interpolation of pbmc dataset. }
    \label{fig:pbmc-interpol}
\end{figure}

\begin{figure}[!ht]
    \centering
    \includegraphics[width=\linewidth]{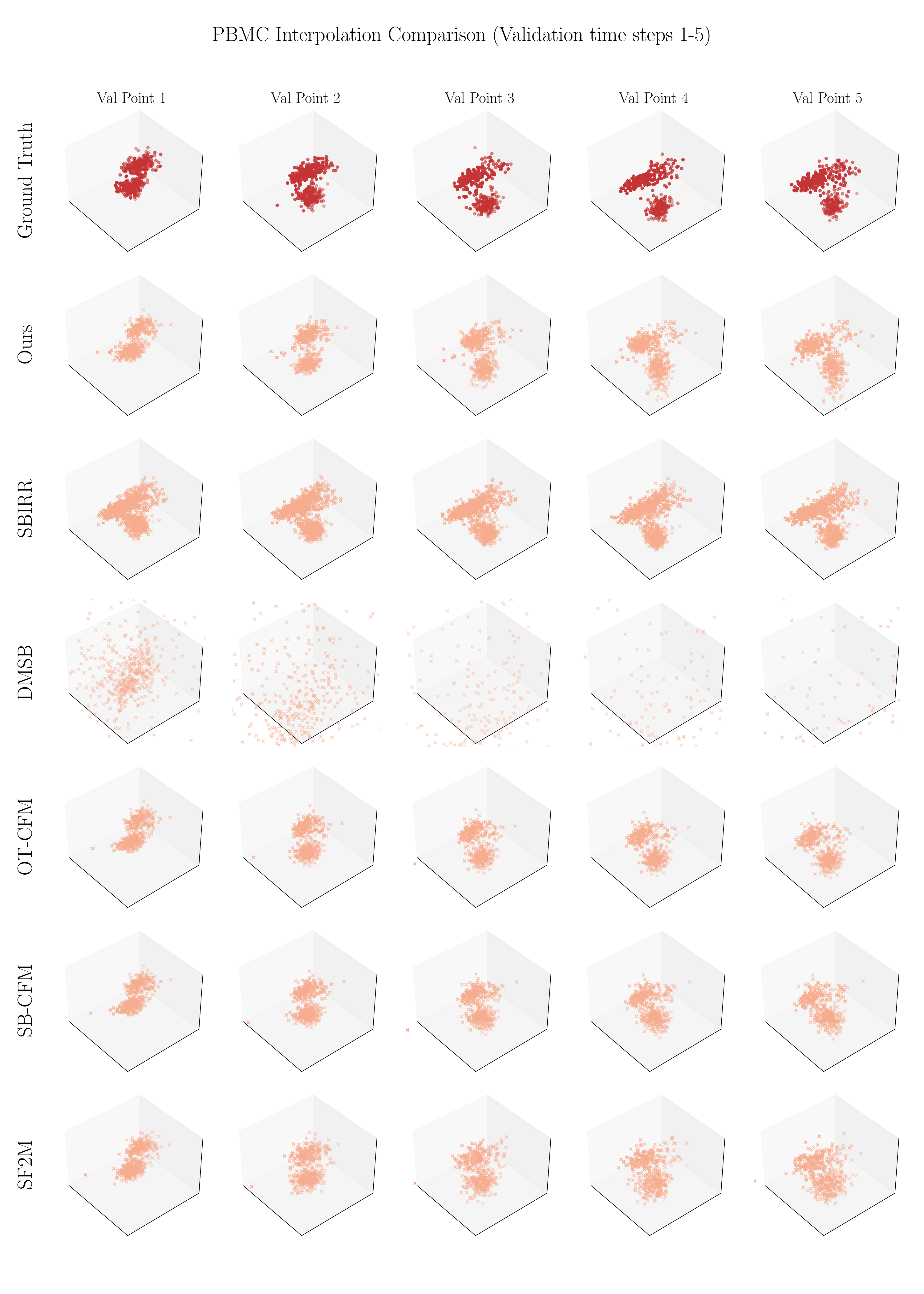}
    \caption{PBMC interpolation results for the first 5 validation time points. The axis are the first three principal components as in the forecasting experiment. The first row shows the evolution of cells for ground truth. The other six rows show the predicted cells at the validation time points for the our method and the five baselines.}
    \label{fig:pbmc-interpolation-1-5}
\end{figure}

\begin{figure}[!ht]
    \centering
    \includegraphics[width=\linewidth]{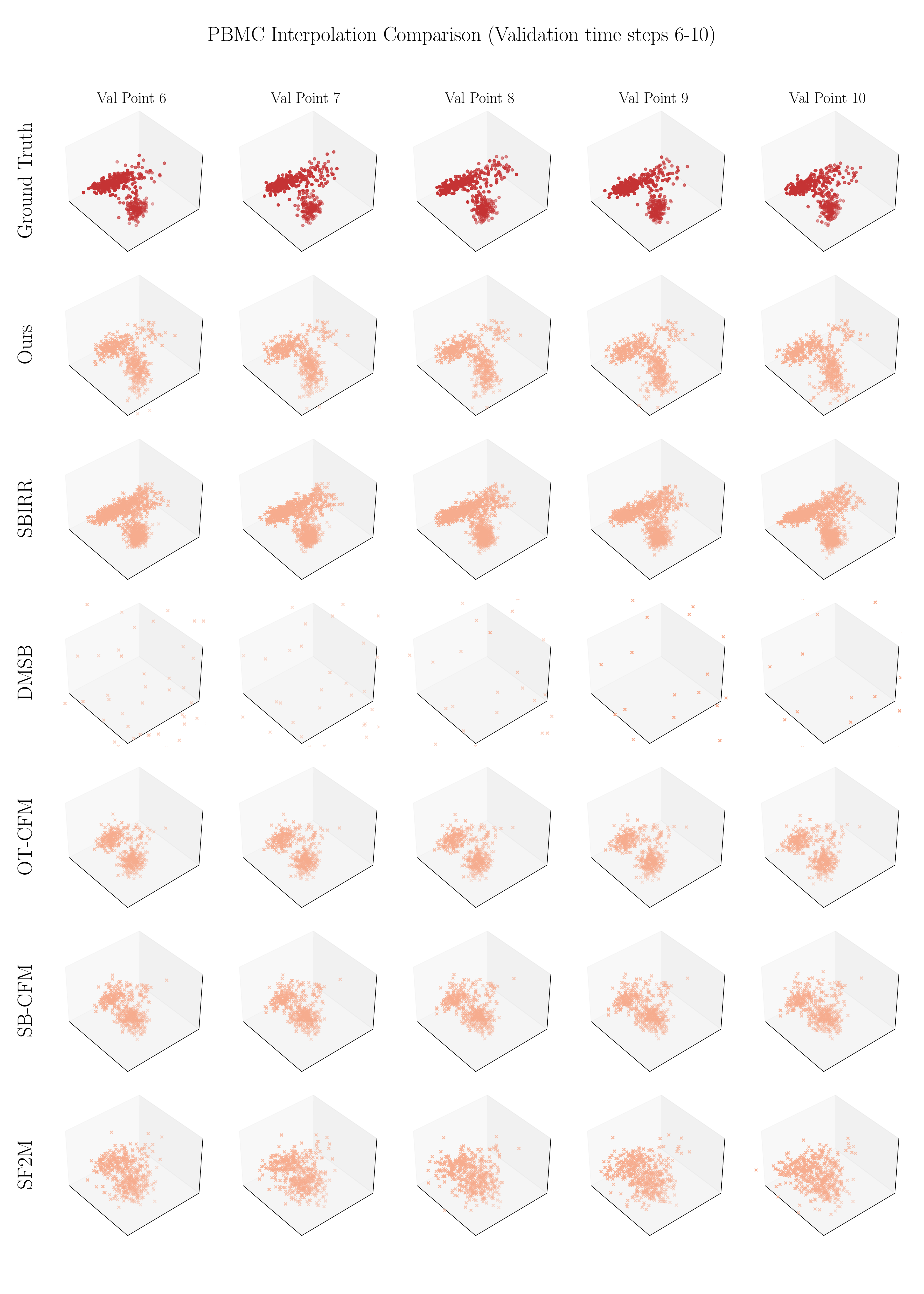}
    \caption{PBMC interpolation results for validation time points 6 to 10. The axis are the first three principal components as in the forecasting experiment. The first row shows the evolution of cells for ground truth. The other six rows show the predicted cells at the validation time points for the our method and the five baselines.}
    \label{fig:pbmc-interpolation-6-10}
\end{figure}

\begin{figure}[!ht]
    \centering
    \includegraphics[width=\linewidth]{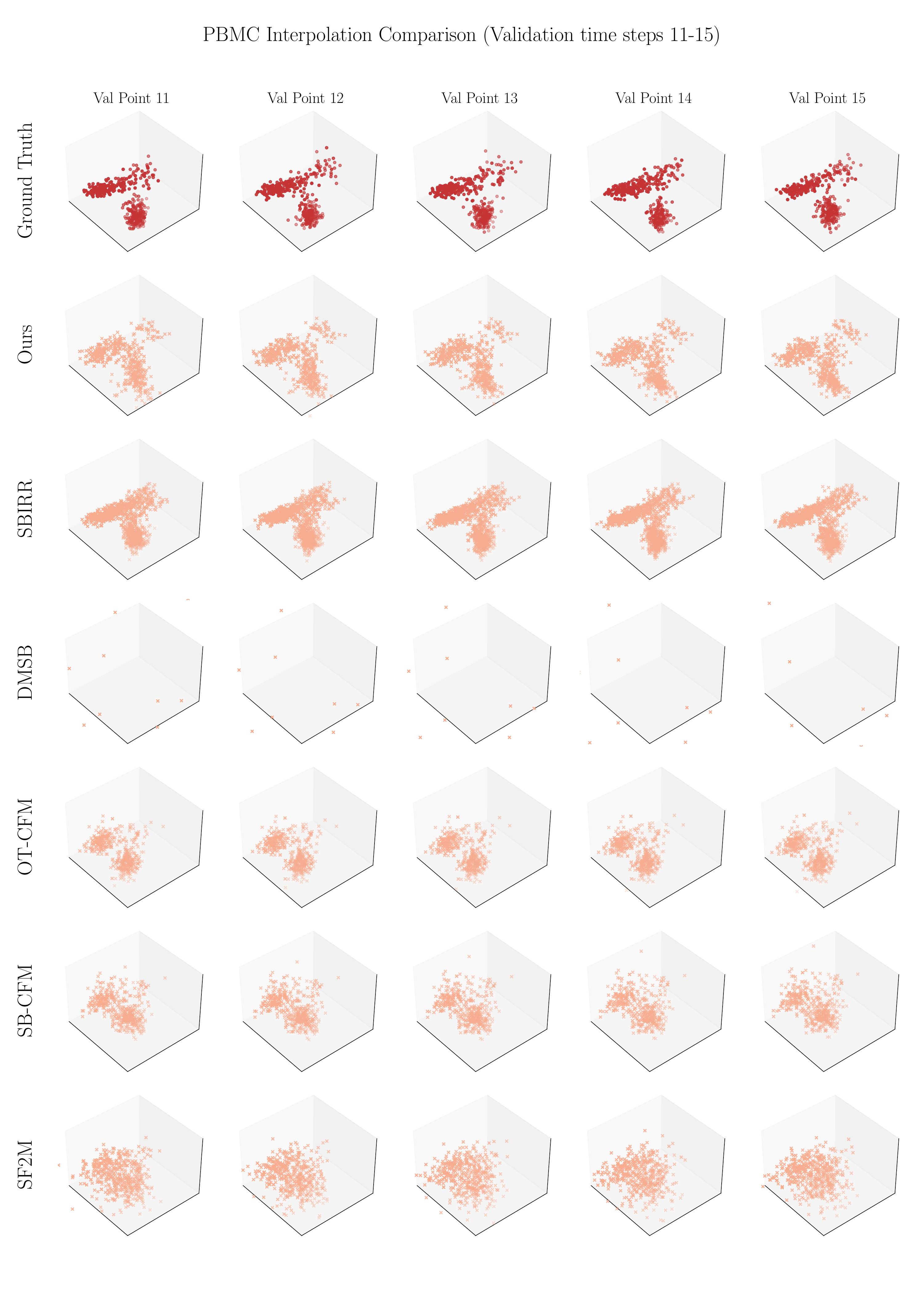}
    \caption{PBMC interpolation results for validation time points 11 to 15. The axis are the first three principal components as in the forecasting experiment. The first row shows the evolution of cells for ground truth. The other six rows show the predicted cells at the validation time points for the our method and the five baselines.}
    \label{fig:pbmc-interpolation-11-15}
\end{figure}

\begin{figure}[!ht]
    \centering
    \includegraphics[width=\linewidth]{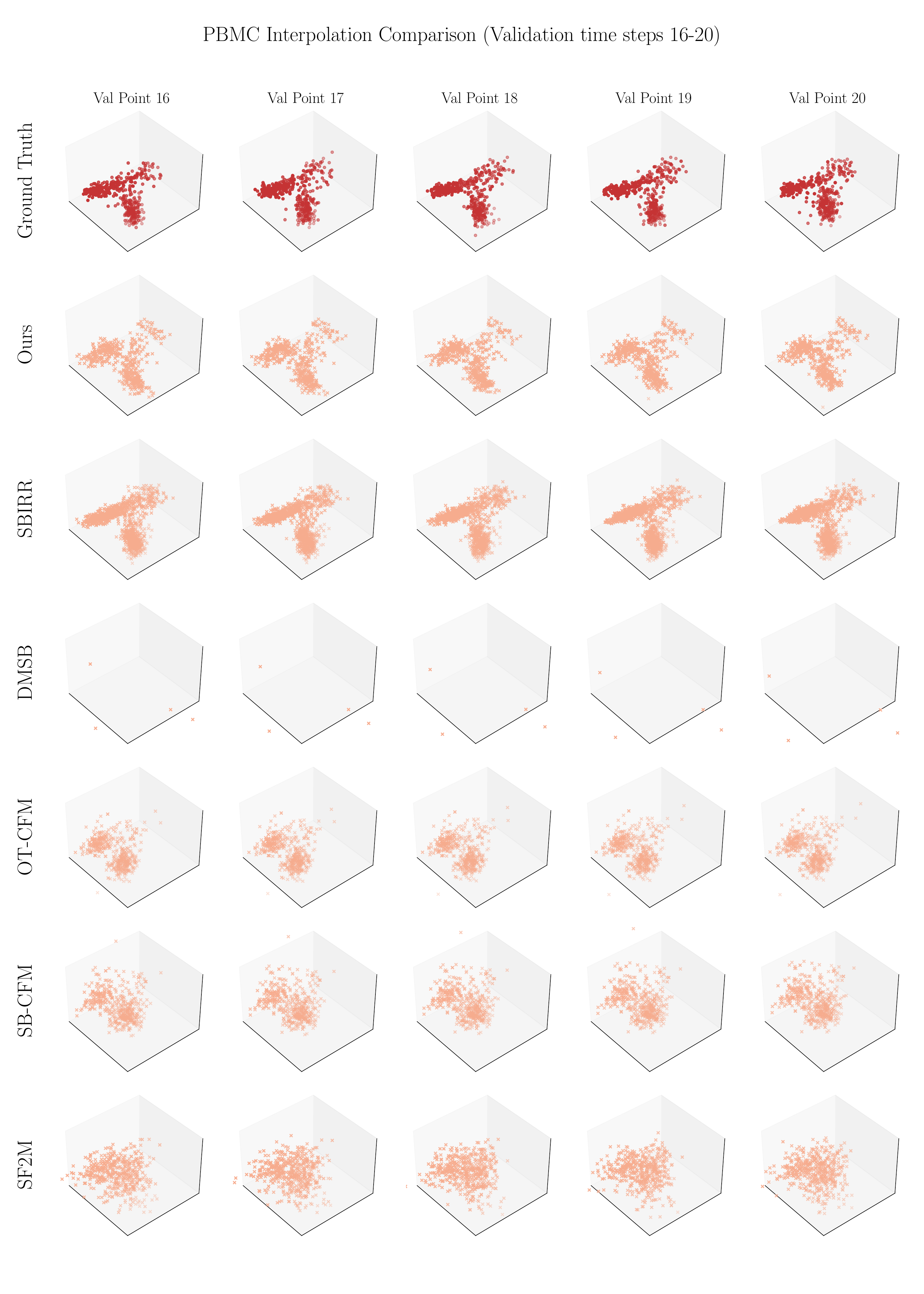}
    \caption{PBMC interpolation results for validation time points 16 to 20. The axis are the first three principal components as in the forecasting experiment. The first row shows the evolution of cells for ground truth. The other six rows show the predicted cells at the validation time points for the our method and the five baselines.}
    \label{fig:pbmc-interpolation-16-20}
\end{figure}

\begin{figure}[!ht]
    \centering
    \includegraphics[width=\linewidth]{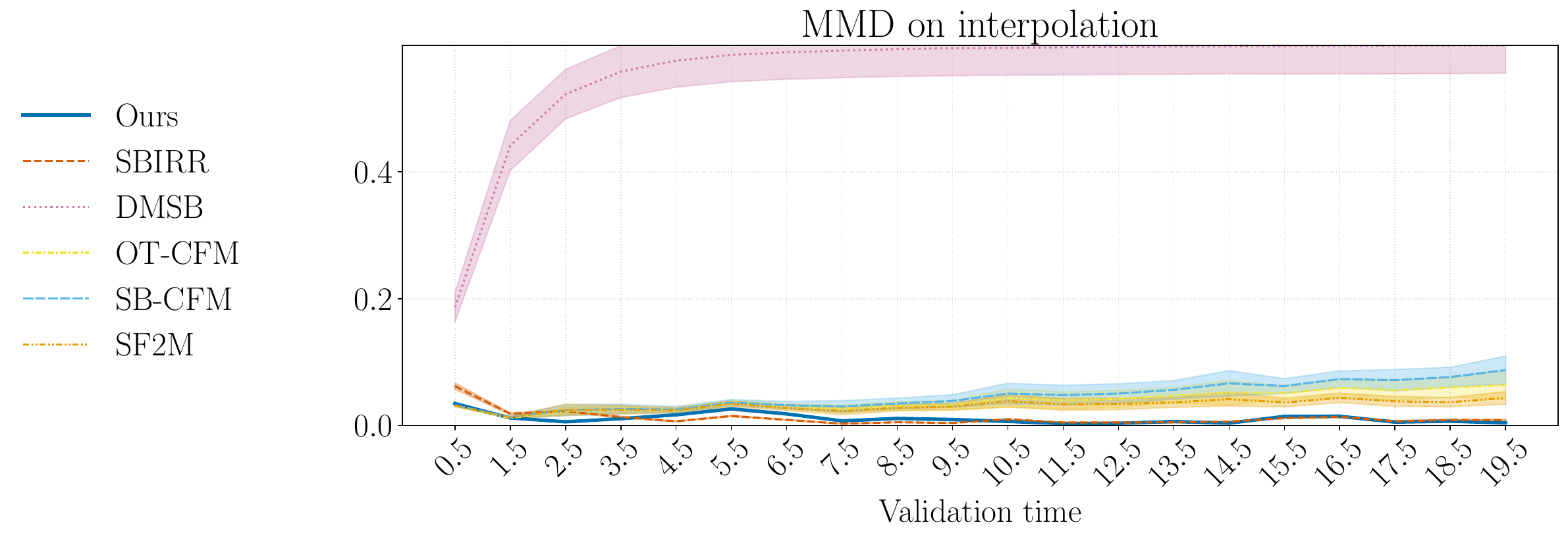}
    \caption{Metrics for Interpolation of pbmc dataset. }
    \label{fig:pbmc-interpol-metric}
\end{figure}

\begin{table}[ht]
\caption{MMD at each validation point for PBMC.}
\centering
\makebox[0pt][c]{
\small
\begin{tabular}{lrrrrrr}
\toprule
Time & Ours & \texttt{SBIRR} & \texttt{DMSB} & \texttt{OT-CFM} & \texttt{SB-CFM} & \texttt{SF2M} \\\midrule
0.5 & $0.035\pm{ 0.002}$ & $0.062\pm{ 0.005}$ & $0.187\pm{ 0.023}$ & \cellcolor{green!25}\bm{$0.031\pm{ 0.001}$} & $0.032\pm{ 0.001}$ & \cellcolor{green!25}$0.031\pm{ 0.001}$ \\
1.5 & \cellcolor{green!25}$0.012\pm{ 0.000}$ & $0.019\pm{ 0.001}$ & $0.442\pm{ 0.039}$ & \cellcolor{green!25}\bm{$0.012\pm{ 0.001}$} & $0.013\pm{ 0.002}$ & $0.013\pm{ 0.002}$ \\
2.5 & \cellcolor{green!25}\bm{$0.006\pm{ 0.001}$} & $0.022\pm{ 0.001}$ & $0.523\pm{ 0.039}$ & $0.023\pm{ 0.006}$ & $0.024\pm{ 0.009}$ & $0.025\pm{ 0.009}$ \\
3.5 & \cellcolor{green!25}\bm{$0.011\pm{ 0.001}$} & $0.013\pm{ 0.002}$ & $0.558\pm{ 0.041}$ & $0.026\pm{ 0.005}$ & $0.026\pm{ 0.008}$ & $0.025\pm{ 0.007}$ \\
4.5 & $0.017\pm{ 0.001}$ & \cellcolor{green!25}\bm{$0.007\pm{ 0.001}$} & $0.576\pm{ 0.042}$ & $0.024\pm{ 0.003}$ & $0.025\pm{ 0.005}$ & $0.024\pm{ 0.004}$ \\
5.5 & $0.026\pm{ 0.002}$ & \cellcolor{green!25}\bm{$0.015\pm{ 0.001}$} & $0.585\pm{ 0.042}$ & $0.036\pm{ 0.003}$ & $0.036\pm{ 0.005}$ & $0.034\pm{ 0.004}$ \\
6.5 & $0.018\pm{ 0.002}$ & \cellcolor{green!25}\bm{$0.009\pm{ 0.001}$} & $0.589\pm{ 0.042}$ & $0.031\pm{ 0.004}$ & $0.032\pm{ 0.006}$ & $0.027\pm{ 0.004}$ \\
7.5 & $0.007\pm{ 0.001}$ & \cellcolor{green!25}\bm{$0.003\pm{ 0.000}$} & $0.591\pm{ 0.043}$ & $0.026\pm{ 0.006}$ & $0.030\pm{ 0.009}$ & $0.023\pm{ 0.006}$ \\
8.5 & $0.011\pm{ 0.002}$ & \cellcolor{green!25}\bm{$0.005\pm{ 0.000}$} & $0.594\pm{ 0.043}$ & $0.031\pm{ 0.007}$ & $0.035\pm{ 0.008}$ & $0.028\pm{ 0.005}$ \\
9.5 & $0.009\pm{ 0.002}$ & \cellcolor{green!25}\bm{$0.004\pm{ 0.000}$} & $0.595\pm{ 0.043}$ & $0.032\pm{ 0.008}$ & $0.039\pm{ 0.010}$ & $0.030\pm{ 0.005}$ \\
10.5 & \cellcolor{green!25}\bm{$0.006\pm{ 0.001}$} & $0.010\pm{ 0.001}$ & $0.596\pm{ 0.043}$ & $0.043\pm{ 0.014}$ & $0.050\pm{ 0.016}$ & $0.038\pm{ 0.010}$ \\
11.5 & \cellcolor{green!25}\bm{$0.003\pm{ 0.000}$} & $0.005\pm{ 0.001}$ & $0.597\pm{ 0.043}$ & $0.040\pm{ 0.014}$ & $0.048\pm{ 0.015}$ & $0.033\pm{ 0.009}$ \\
12.5 & \cellcolor{green!25}\bm{$0.003\pm{ 0.001}$} & $0.005\pm{ 0.001}$ & $0.597\pm{ 0.044}$ & $0.042\pm{ 0.014}$ & $0.050\pm{ 0.015}$ & $0.034\pm{ 0.009}$ \\
13.5 & $0.006\pm{ 0.001}$ & \cellcolor{green!25}\bm{$0.005\pm{ 0.000}$} & $0.598\pm{ 0.044}$ & $0.047\pm{ 0.012}$ & $0.056\pm{ 0.015}$ & $0.036\pm{ 0.008}$ \\
14.5 & \cellcolor{green!25}\bm{$0.004\pm{ 0.001}$} & $0.006\pm{ 0.001}$ & $0.598\pm{ 0.044}$ & $0.052\pm{ 0.019}$ & $0.066\pm{ 0.020}$ & $0.042\pm{ 0.011}$ \\
15.5 & $0.014\pm{ 0.002}$ & \cellcolor{green!25}\bm{$0.012\pm{ 0.001}$} & $0.599\pm{ 0.044}$ & $0.051\pm{ 0.011}$ & $0.062\pm{ 0.012}$ & $0.036\pm{ 0.007}$ \\
16.5 & \cellcolor{green!25}$0.014\pm{ 0.002}$ & \cellcolor{green!25}\bm{$0.014\pm{ 0.002}$} & $0.599\pm{ 0.044}$ & $0.060\pm{ 0.014}$ & $0.073\pm{ 0.013}$ & $0.044\pm{ 0.007}$ \\
17.5 & \cellcolor{green!25}\bm{$0.005\pm{ 0.001}$} & $0.007\pm{ 0.001}$ & $0.599\pm{ 0.044}$ & $0.056\pm{ 0.016}$ & $0.071\pm{ 0.017}$ & $0.038\pm{ 0.008}$ \\
18.5 & \cellcolor{green!25}\bm{$0.007\pm{ 0.001}$} & $0.009\pm{ 0.001}$ & $0.599\pm{ 0.044}$ & $0.060\pm{ 0.015}$ & $0.076\pm{ 0.016}$ & $0.037\pm{ 0.007}$ \\
19.5 & \cellcolor{green!25}\bm{$0.004\pm{ 0.002}$} & $0.008\pm{ 0.002}$ & $0.600\pm{ 0.044}$ & $0.064\pm{ 0.023}$ & $0.087\pm{ 0.022}$ & $0.043\pm{ 0.010}$ \\
\bottomrule
\end{tabular}}
\label{tab:pbmc-mmd}
\end{table}

\clearpage

\section{Identifiability Analysis}
\label{app:identifiability}

In this appendix, we provide further details on the identifiability problem from the the main text discussion.

\textbf{Why drift and volatility are not identified in general.}
Even with complete access to the marginal distributions $\marginal{t}$ over time, the pair $(\truedrift,\truevolatility)$ is not uniquely determined by the Fokker–Planck equation
\begin{align*}
   \frac{\partial \marginal{t}}{\partial t} = \nabla \cdot \left[-\truedrift\,\marginal{t} + \frac{1}{2}\,\truevolatility\,\truevolatility^\top \nabla \marginal{t}\right] 
\end{align*}

For example, suppose $(\truedrift, \truevolatility)$ satisfies the equation for a given $\marginal{t}$. Then, for any vector field $\bm{h}$ that satisfies the continuity condition $\nabla \cdot \left(\bm{h}\,\marginal{t}\right) = 0$, the modified drift $\truedrift' = \truedrift + \bm{h}$ with the same volatility $\truevolatility$ also satisfies the Fokker–Planck equation. This observation indicates that an infinite family of drift functions can generate the same evolution of the marginal distribution if no further constraints are imposed. Furthermore, let $\bm{A}$ be any orthogonal matrix (i.e., $\bm{A}\bm{A}^\top = \bm{I}$). Then, the pair $(\truedrift, \truevolatility\,\bm{A})$ also satisfies the Fokker–Planck equation. These examples illustrate the inherent non-uniqueness (or non-identifiability) of the drift and volatility functions based solely on the evolution of the marginal distributions.

In practice, to achieve identifiability, one must restrict the candidate function classes for $\truedrift$ and $\truevolatility$. For instance, assuming that $\truedrift$ is a gradient field (i.e., $\truedrift = \nabla \Phi$ for some potential $\Phi$ and that $\truevolatility$ is constant is known to yield identifiability under suitable conditions \citep{Lavenant2021, guan2024identifying}. A complete characterization of identifiability in more general settings is beyond the scope of this work and constitutes an important direction for future research.

\clearpage

\end{document}